\documentclass{article}

\usepackage{url}
\usepackage{amsfonts}
\usepackage{nicefrac}
\usepackage{microtype}      
\usepackage{xcolor}         
\usepackage{graphicx}
\usepackage{booktabs} 
\usepackage{natbib}
\usepackage{booktabs}   \usepackage{mathrsfs}
\usepackage{amsmath,amssymb}
\usepackage{algorithm}
\usepackage[noend]{algorithmic}
\usepackage{changepage}
\usepackage{makecell}
\usepackage{microtype}  \usepackage{amsthm}
\usepackage{changepage}
\usepackage{color}
\usepackage{bbm}
\usepackage{hyperref,url}
\hypersetup{
  colorlinks,
  citecolor=blue!70!black,
  linkcolor=blue!70!black,
  urlcolor=blue!70!black}

\newcommand{\cG}{\mathcal{G}}
\newcommand{\cS}{\mathcal{S}}
\newcommand{\cA}{\mathcal{A}}
\newcommand{\cR}{\mathcal{R}}
\newcommand{\cM}{\mathcal{M}}
\newcommand{\cC}{\mathcal{C}}
\newcommand{\cF}{\mathcal{F}}
\newcommand{\cT}{\mathcal{T}}
\newcommand{\cB}{\mathcal{B}}
\newcommand{\cD}{\mathcal{D}}
\newcommand{\cZ}{\mathcal{Z}}
\newcommand{\cV}{\mathcal{V}}

\newcommand{\cE}{\mathcal{E}}
\newcommand{\cN}{\mathcal{N}}

\newcommand{\EE}{\mathbb{E}}

\newcommand{\ds}{\mathrm{dim}}

\newcommand{\wh}[1]{\widehat{#1}}
\newcommand{\poly}{\mathrm{poly}}

\newcommand{\pairs}{\cZ}
\newcommand{\funclass}{\cF}
\newcommand{\cover}{\cC}
\newcommand{\sen}{\mathsf{sensitivity}}
\newcommand{\coversize}{\cN}
\newcommand{\states}{\cS}
\newcommand{\actions}{\cA}

\newtheorem{assum}{Assumption}

\newtheorem{thm}{Theorem}
\newtheorem{lem}{Lemma}
\newtheorem{prop}{Proposition}
\newtheorem{defn}{Definition}
\newtheorem{remark}{Remark}
\renewcommand{\epsilon}{\varepsilon}
\renewcommand{\tilde}{\widetilde}
\DeclareMathOperator*{\argmin}{arg\,min}
\DeclareMathOperator*{\argmax}{arg\,max}

\usepackage{fullpage}
\title{Online Sub-Sampling for Reinforcement Learning with General Function Approximation\thanks{Authors ordered alphabetically.}}
\author{
Dingwen Kong\\
Peking University\\
\texttt{dingwenk@pku.edu.cn}\\
\and Ruslan Salakhutdinov\\
Carnegie Mellon University\\
\texttt{rsalakhu@cs.cmu.edu}
\and Ruosong Wang\\
Carnegie Mellon University\\
\texttt{ruosongw@andrew.cmu.edu}\\
\and Lin F. Yang\\
University of California, Los Angles\\
\texttt{linyang@ee.ucla.edu}
}

\begin{document}

\maketitle

\begin{abstract}
Most of the existing works for reinforcement learning (RL) with general function approximation (FA) focus on understanding the statistical complexity or regret bounds. However, the computation complexity of such approaches is far from being understood -- indeed, a simple optimization problem over the function class might be as well intractable.
In this paper, we tackle this problem by establishing an efficient online sub-sampling framework that measures the information gain of data points collected by an RL algorithm and uses the measurement to guide exploration.
For a value-based method with complexity-bounded function class, we show that the policy only needs to be updated for $\propto\operatorname{poly}\log(K)$ times for running the RL algorithm for $K$ episodes while still achieving a small near-optimal regret bound. In contrast to existing approaches that update the policy for at least $\Omega(K)$ times, our approach drastically reduces the number of optimization calls in solving for a policy. When applied to settings in \cite{wang2020reinforcement} or \cite{jin2021bellman}, we improve the overall time complexity by at least a factor of $K$.
Finally, we show the generality of our online sub-sampling technique by applying it to the reward-free RL setting and multi-agent RL setting.

\end{abstract}
\section{Introduction}
Function approximation (FA) is one of the key techniques in scaling up reinforcement learning (RL) in real-world applications. FA methods have achieved phenomenal empirical success in games or control tasks~ \citep{mnih2013playing,mnih2015human,silver2017mastering,vinyals2019grandmaster,akkaya2019solving}, where in these applications, RL agents learn to control complex systems by approximating value functions using deep neural networks.
In contrast, the theoretical understanding of RL with general FA is still rudimentary, e.g., reasonable theoretical understandings have only been established in the tabular setting or the linear setting~(c.f., \cite{azar2017minimax,jin2018q,yang2020reinforcement,jin2020provably, zanette2020learning}, and references therein). Designing RL algorithms with general FA with provable efficiency becomes increasingly important as it helps understanding the limits of existing algorithms while inspiring designing better practical algorithms.

Recently, there is a line of works targeting the sample efficiency of FA methods, including \citet{wang2020reinforcement, jin2021bellman, ayoub2020model, feng2021provably}. 
These papers assume the value, policy, or transition models of an RL environment can be approximated by a function from a function class. 
For intance, \citet{wang2020reinforcement} assumes that the function class captures the \emph{value closedness}, which requires that the \emph{Bellman projection} (defined formally in Section~\ref{sec:basics}) of any value function lies in the target function class. 
\citet{wang2020reinforcement}~establish a provably correct algorithm that achieves a regret bound of $\widetilde{O}(\poly(dH)\sqrt{K})$\footnote{Throughout the paper, we use $\widetilde{O}(\cdot)$ to suppress logarithm factors. }, where  $H$ is the planning horizon,  $K$ is the total number of episodes, and $d$ depends on the eluder dimension~\citep{russo2013eluder} and log-covering numbers of the function class.
Here the regret measures the difference between the expected rewards collected by the optimal policy and that of the RL algorithm.
Whereas \citet{jin2021bellman} study a more general setting that the function class captures only \emph{value completeness}, which requires the Bellman projection of a value function \emph{in the function class} is still in the class. The algorithm in \cite{jin2021bellman} achieves a similar regret bound but requires solving a number of possibly computationally intractable optimization problems.  

Compared to the regret bound or sample complexity of FA methods, the computation complexity of these methods is much less understood. 
For instance, for the simple regression problem, which is a core step in many RL algorithms, finding a solution in a general function class might be intractable. 
Indeed, the computation complexity is heavily determined by the structure of the function class. To still measure a meaningful complexity, many papers consider a type of oracle complexity -- the number of calls of certain optimization oracle on the function class. 
For instance, the time complexity in \cite{foster2020beyond,wang2020reinforcement} measures the number of calls to a regression oracle, which solves a regression problem with the function class. The time complexity in \cite{jin2021bellman} measures the number of oracle calls to a nested optimization problem. Such an \emph{oracle complexity} allows us to focus on the complexity introduced by the RL algorithm rather than the optimization structure of the function class and is the focus of this paper. Nevertheless, existing results with FA take at least $\poly(K)$ number of optimization oracle calls regardless of the assumptions on the function class.  We therefore ask:
\emph{
is there a generic approach to design a time-efficient RL algorithm with the least  number of optimization oracle calls?
}

In this paper, we develop a novel online sub-sampling technique that drastically reduces the amount of computation of RL with FA, while achieving small regret bounds, hence being both \emph{sample and time efficient}.
The core idea is to establish a novel notion called \emph{online sensitivity score} applied to each data point (i.e., a state-action pair) collected by an RL algorithm. This score measures the \emph{information gain} of a new data point with respect to the function class. 
Using this score as a sampling probability,  we maintain a small subsampled dataset. For  each new data point arriving, the algorithm can quickly compute the score by solving an optimization problem with respect to the small subsampled dataset.
We show that for many function classes, the number of data points in the subsampled dataset is only proprotional to $\poly\log(K)$, for $K$ episodes of learning.
This sub-sampling technique enables a \emph{low-switching-cost} RL framework that
updates  the  learned policy only when the sub-sampling buffer has changed, 
reducing the expensive computation in obtaining a new policy.
We show that under both the closedness assumption in \cite{wang2020reinforcement} or the completeness assumption in \cite{jin2021bellman}, our framework reduces the computation complexity to $\poly(dH\log(K))$ optimization oracle calls, while preserving the regret bounds. The sub-sampling procedure, on the other hand, only requires to solve a optimization problem over a $\tilde{O}(d^2)$-sized dataset per timestep.

{We remark that besides saving computation, the low-switching-cost property itself is also desired in many real-world RL applications. For example, in recommendation systems~\citep{afsar2021reinforcement} and healthcare~\citep{yu2020mopo}, each new recommendation/new medical treatment must pass several internal tests to ensure safety before being deployed~\citep{huang2022towards}, which can be time-costly and requires additional human efforts. Furthermore, the low-switching-cost property enables concurrent RL~\citep{guo2015concurrent}, where several agents interact with the same MDP in parallel. See Section~\ref{sec:closerelated} for a brief discussion on related works on RL with low switching costs.}

{Finally, we apply our online sub-sampling technique to the reward-free setting, where the exploration does not require the guidance of a reward, and the multi-agent RL setting, where the agent plays against an adversarial opponent.}
We summarize our contributions below.
\vspace{-2mm}
\begin{itemize}
    \itemsep0em 
    \item We propose a generic subsampling framework to measure the information gain of the data points collected by an RL algorithm wrt. the target function class. This low-switching-cost framework allows an RL algorithm to update policy for number of times logarithmically depending on the number of episodes played by the algorithm, while still having a small regret. Thus the framework serves as a general method for accelerating existing algorithms for RL with general FA. 
    \item To give rigorous and meaningful results on computation complexity, we consider the oracle complexity: the number of calls of certain optimization oracle on the function class. As two representative examples, we apply the framework in the closedness setting (\cite{wang2020reinforcement}) and completeness setting (\cite{jin2021bellman}). We establish near-optimal-regret value-based algorithms with exponentially smaller oracle complexity.
    \item {We show the generality of our online sub-sampling technique by applying it to reward-free RL and multi-agent RL, which may be of independent interests.}
\end{itemize}

\subsection{Closely Related Work}
\label{sec:closerelated}
{Here we briefly review some closely related works. More related works are deferred to the appendix.

\paragraph{RL with Low Switching Cost.}  \citet{bai2019provably}~is the first work that studies switching cost in tabular RL. Former works ~\citep{abbasi2011improved,gao2021provably,wang2021provably} also design algorithms with $\poly\log(K)$-type switching cost in linear stochastic bandit or linear MDP settings. However, all their works rely on tracking the determinant of the feature covariance matrix and switching the policy (or updating the model parameter) when the determinant doubles. However, their method can not work in the general FA setting since a \emph{feature extractor} may not be available. Our work is the \emph{first} to design a provably low-switching-cost RL algorithm in the general FA setting.}

\paragraph{RL with General Function Approximation.} Besides \citet{wang2020reinforcement} and \citet{jin2021bellman}, 
\citet{feng2021provably} provides a policy-based algorithm whose sample complexity is bounded by the eluder dimension of the function approximation class. 
\citet{jiang2017contextual}~design a provably efficient algorithm whose sample complexity can be upper bounded in terms of the Bellman rank (which measures the rank of the Bellman error matrix) of the function class. \citet{ayoub2020model}~propose an algorithm for model-based RL based on value-targeted regression.
They bound the complexity using eluder dimension as well.
\citet{du2021bilinear}~propose an algorithm for Bilinear function classes. 
\citet{foster2020instance}~propose a LSVI-based algorithm whose regret bound depends on the notion of disagreement coefficient. We remark that all these algorithms take at least $\poly(K)$ number of optimization oracle calls. 
In this paper we demonstrate the efficacy of our approach for the algorithms in \citet{wang2020reinforcement} and \citet{jin2021bellman}. We will discuss the possibility of extending our framework to a broader scope in Section~\ref{limit}.

\section{Preliminaries}
\label{sec:pre}

Before we formally introduce our algorithmic framework, we first present the background of reinforcement learning and function classes.
\subsection{Episodic Markov Decision Process}
In this paper, we consider the finite-horizon Markov decision process (MDP) $M=(\mathcal{S},\mathcal{A},P,r,H,s_{1})$, where $\mathcal{S}$ is the state space, $\mathcal{A}$ is the action space, $P=\{P_h\}_{h=1}^H$ where $P_h:\mathcal{S} \times \mathcal{A} \rightarrow \mathcal{\triangle}(\mathcal{S})$ are the transition operators, $r=\{r_h\}_{h=1}^H$ where $r_h:\mathcal{S} \times \mathcal{A} \rightarrow [0,1]$ are the deterministic reward functions, and $H$ is the planning horizon. Without loss of generality, we assume that the initial state $s_{1}$ is fixed.\footnote{For a general initial distribution $\rho$, we can treat it as the first stage transition probability, $P_1$.}
In RL, an agent interacts with the environment episodically. Each episode consists of $H$ time steps. A deterministic policy $\pi$ chooses an action $a \in \mathcal{A}$ based on the current state $s\in \mathcal{S}$ at each time step $h \in [H]$. Formally, $\pi=\{\pi_h\}_{h=1}^H$ where for each $h \in [H]$, $\pi_h:\mathcal{S}\rightarrow \mathcal{A}$ maps a given state to an action. In each episode, the policy $\pi$ induces a trajectory
$$
s_1,a_1,r_1,s_2,a_2,r_2,...,s_H,a_H,r_H,s_{H+1}
$$
where $s_1$ is fixed, $a_1=\pi_1(s_1)$, $r_1=r_1(s_1,a_1)$, $s_2 \sim P_1(\cdot|s_1,a_1)$, $a_2=\pi_2(s_2)$, etc.

We use Q-function and V-function to evaluate the long-term expected cumulative  reward in terms of the current state (state-action pair) and the policy deployed. Concretely, the Q-function and V-function are defined as:
$$Q_h^\pi(s,a):=\mathbb{E}\big[\sum_{h'=h}^Hr_{h'}|s_h=s,a_h=a,\pi\big]$$
and $$V_h^\pi(s):=\mathbb{E}\big[\sum_{h'=h}^Hr_{h'}|s_h=s,\pi\big].$$
The agent aims to gradually improve its performance during the interaction with the environment. Thus an important measurement of the effectiveness of the learning algorithm is the \emph{regret}. More precisely, we assume that the agent interacts with the environment for $K>0$ episodes. For $k\in[K]$, at the beginning of the $k$-th episode the agent chooses a policy $\pi^k$ to collect a trajectory. Then the regret is defined as $$\text{Regret}(K):=\sum_{k=1}^K\left(V_1^*(s_1)-V_1^{\pi^k}(s_1)\right)$$
where $s_1$ is the fixed initial state. We assume $K$ is fixed and known to the agent\footnote{There are standard techniques to extend the results to unknown $K$, e.g.,~\citep{besson2018doubling}.
}. Throughout the paper, we define $T:=KH$ to be the total number of steps. 

In this paper we consider the \emph{global switching cost}, which counts the number of policy changes in the running of the algorithm in $K$ episodes, namely: 
$$
N_{\text{switch}}^{\text{gl}}:=\sum_{k=1}^{K-1}\mathbb{I}\{\pi_k\neq \pi_{k+1}\}.
$$

\paragraph{Additional Notations.} We define the infinity-norm of function $f:\mathcal{S}\times\mathcal{A}\rightarrow \mathbb{R}$ and $v:\mathcal{S}\rightarrow \mathbb{R}$ as: 
$\|f\|_\infty:=\sup_{(s,a)\in\mathcal{S}\times\mathcal{A}}|f(s,a)|$ and $\|v\|_\infty:=\sup_{s\in\mathcal{S}}|v(s)|$.
Given a dataset $\mathcal{D}:=\{(s_i,a_i,q_i)\}_{i=1}^n\subseteq \mathcal{S}\times\mathcal{A}\times\mathbb{R}$, for a function $f:\mathcal{S}\times\mathcal{A}\rightarrow \mathbb{R}$, we define $\|f\|_{\mathcal{D}}:=\big(\sum_{i=1}^n(f(s_i,a_i)-q_i)^2\big)^{1/2}.$
Furthermore, for a set of state-action pairs $\mathcal{Z}\subseteq \mathcal{S}\times\mathcal{A}$ and a function $f:\mathcal{S}\times\mathcal{A}\rightarrow \mathbb{R}$, we define $\|f\|_{\mathcal{Z}}:=\left(\sum_{(s,a)\in{\mathcal{Z}}}f(s,a)^2\right)^{1/2}.$

\subsection{Function Class}
\label{sec:basics}
In our setting, we assume that a function class $\mathcal{F}\subset \{f:\mathcal{S}\times\mathcal{A}\rightarrow[0,H+1]\}$ is given as a priori. In the algorithm we will use functions from $\mathcal{F}$ to approximate the optimal Q-function. For any $f:\cS\times\cA\rightarrow\mathbb{R}$ and $h\in[H]$, we define the \emph{Bellman projection} of $f$ as {$\cT_h f:\cS\times\cA\rightarrow\mathbb{R}$, such that for all $(s,a)\in\cS\times\cA$,} 
\[
\cT_h f(s,a):=r_h(s,a)+\mathbb{E}_{s'\sim\mathbb{P}_h(\cdot|s,a)}\max_{a'\in\cA}f(s',a').
\] 
For $v:\cS\rightarrow\mathbb{R}$ and $h\in[H]$, we also define the Bellman projection of $v$ as $\cT_h v:\cS\times\cA\rightarrow\mathbb{R}$, such that for all $(s,a)\in\cS\times\cA$,
$$\cT_h v(s,a):=r_h(s,a)+\mathbb{E}_{s'\sim\mathbb{P}_h(\cdot|s,a)}v(s').$$


We  need to assume bounded covering numbers for both the function class and the state-action space to measure the complexity of the RL problem.
This assumption is standard in the literature~\citep{russo2013eluder, wang2020reinforcement, jin2021bellman}.
\begin{assum}[Covering Number]
\label{assum:cover}
The function class 
$\mathcal{F}$, and state-action space $\mathcal{S}\times\mathcal{A}$ both have bounded covering numbers. Concretely, for any $\varepsilon >0$, there exists an $\varepsilon$-cover $\cover(\funclass, \varepsilon) \subseteq \funclass$ with size $|\cover(\funclass, \varepsilon)| \le \coversize(\funclass, \varepsilon)$, such that for any $f\in \funclass$, there exists $f'\in \cover(\funclass,\varepsilon)$ with $\|f-f'\|_{\infty}\le \varepsilon$. Also, there exists an $\varepsilon$-cover $\cover(\cS\times \cA, \varepsilon)$ with size $|\cover(\cS\times \cA, \varepsilon)| \le \coversize(\cS\times\cA, \varepsilon)$, such that for any $(s,a)\in \cS\times \cA$, there exists $(s',a')\in \cover(\cS\times\cA,\varepsilon)$ with $\sup_{f\in \funclass} |f(s,a)-f(s', a')|\le \varepsilon$.
\end{assum}

\section{Algorithmic Framework}
\label{sec:onlinesample}
\subsection{Online Sub-Sampling Procedure}
In standard online RL algorithms, the agent computes a new policy using newly collected data after each episode. Computing a new policy is usually time-costly.  Our idea is that we only update the policy when we have collected enough new information. In order to achieve this goal, we need to develop an criterion to decide whether we have collected enough information or not. In addition, this criterion itself can not be time-costly.

A na\"ive idea is that when a new data point arrives, we compute the \emph{information gain} of the new data point with respect to the previously collected dataset. And then we only update the policy when the information gain is high. However, when the size of the dataset grows, the computation of the information gain itself becomes expensive. 

We modify the above idea via online sub-sampling. We maintain a small subsampled dataset $\{\widehat{\mathcal{Z}}_h^k\}_{h=1}^H$, which is initialized to be an empty set for all $h\in[H]$. We hope this subsampled dataset always provides a good approximation to the original dataset, $\mathcal{Z}_h^k:=\{(s_h^{\tau},a_h^{\tau})\}_{\tau\in[k-1]}$. By good approximation we mean $\|f_1 - f_2\|_{\widehat{\mathcal{Z}}_h^k}$ is always close to $\|f_1 - f_2\|_{\mathcal{Z}_h^k}$ for all $f_1,f_2\in\cF$. At the beginning of episode $k$, the algorithm receives  $\{\widehat{\mathcal{Z}}_h^{k-1}\}_{h=1}^H$ and $\{(s_h^{k-1},a_h^{k-1})\}_{h=1}^{H}$, i.e., the current sampling buffer and the trajectory obtained in the previous episode.
For each $h\in[H]$, we compute the \emph{online sensitivity score} to measure the information gain of $(s_h^{k-1},a_h^{k-1})$ with respect to $\widehat{\mathcal{Z}}_h^{k-1}$ by setting $\mathcal{Z} = \widehat{\mathcal{Z}}_h^{k-1}$ and $z = (s_h^{k-1},a_h^{k-1})$ in~\eqref{eqn:sensitivity}.
\begin{equation}
\label{eqn:sensitivity}
\sen_{\mathcal{Z},\mathcal{F}}(z):= 
\min\left\lbrace \sup_{f_1,f_2\in\mathcal{F}}\frac{(f_1(z)-f_2(z))^2}{\min\{\|f_1-f_2\|^2_{\mathcal{Z}},T(H+1)^2\}+\beta},1\right\rbrace.
\end{equation}
Here $\beta$ is a regularization parameter whose value will be specified in the appendix.
For each $h \in [H]$, starting with $\widehat{\mathcal{Z}}_h^k \gets \widehat{\mathcal{Z}}_h^{k-1}$, our algorithm then adds $(s_h^{k-1},a_h^{k-1})$ into $\widehat{\mathcal{Z}}_h^k$ with probability proportional to its online sensitivity score. 
We also set the weight (or equivalently, the number of copies added to the sub-sampled dataset) of $(s_h^{k-1},a_h^{k-1})$ to be the reciprocal of the sampling probability, if added. {We need to round $(s_h^{k-1},a_h^{k-1})$ to a finite cover before adding it to $\widehat{\mathcal{Z}}_h^k$, which gives us the convenience of applying union bound in our analysis. }

Then we can recompute and switch the policy only when the sub-sampling buffer has changed. Note that such change indicates large information gain, thus serves as a natural policy-switching criterion.

Algorithm~\ref{alg:sample} provides one-step sampling procedure. We run Algorithm~\ref{alg:sample} in the following manner to obtain $\widehat{\mathcal{Z}}_h^k$ ($k=1,2,...,K$):
\begin{align*}
&\widehat{\mathcal{Z}}_h^1\leftarrow\{\}, \text{ and for }k=2,3,...,K,\text{ run:}\\
&\widehat{\mathcal{Z}}_h^k\leftarrow \textbf{Online-Sample}(\mathcal{F},\widehat{\mathcal{Z}}_h^{k-1}, (s_{h}^{k-1},a_{h}^{k-1}),\delta)
\end{align*}
As will be shown shortly, the size of the sub-sampling buffer is bounded if the function class has bounded complexity. 
As a result, we rarely update the policy and avoid most of the computation-cost in computing the new policy. In addition, computing the online sensitivity score is also cheap due to the small size of the buffer, $\widehat{\mathcal{Z}}_h^k$. 

We remark that the algorithm in \citet{wang2020reinforcement} also contains a sampling framework to obtain small-sized approximation to  the original dataset. The small-sized subsampled dataset is then used to compute the bonus function. However, their sampling framework deals with static data and more importantly, is very inefficient. Therefore their sampling framework is not suitable for the purpose of this paper.
\begin{algorithm}[tb]
	\caption{Online-Sample($\mathcal{F},\widehat{\mathcal{Z}},z,\delta$)\label{alg:sample}}
	\begin{algorithmic}
		\STATE \textbf{Input:} Function class $\mathcal{F}$, current sub-sampled dataset $\widehat{\mathcal{Z}}\subseteq\mathcal{S}\times\mathcal{A}$, new state-action pair $z$, failure probability $\delta\in(0,1)$ 
		\STATE Let $p_z$ to be the smallest real number such that $1/p_z$ is an integer and $$p_z\geq\min\{1,C\cdot\sen_{\widehat{\mathcal{Z}},\mathcal{F}}(z)
		\cdot \log(T\mathcal{N}(\mathcal{F},\sqrt{\delta/64T^3})/\delta)\}$$
		
	    \STATE Let $\widehat{z}\in\cC(\mathcal{S}\times\mathcal{A},1/{(}16\sqrt{64T^3/\delta}))$ such that 
	    $ \sup_{f\in\mathcal{F}}|f(z)-f(\widehat{z})|\leq 1/16\sqrt{64T^3/\delta}
	    $
		\STATE  Add $1/p_z$ copies of $\widehat{z}$ into $\widehat{\mathcal{Z}}$ (or equivalently, set the weight of  $\widehat{z}$ to be $1/p_z$) with probability $p_z$
		\STATE \textbf{return} $\widehat{\mathcal{Z}}$
	\end{algorithmic}
\end{algorithm}

\subsection{Algorithmic Framework with Online Sub-Sampling}
In this section we build a \emph{low-switching-cost} RL framework based on the online sub-sampling procedure. Our framework splits into three steps: \textbf{Sampling}, \textbf{Planning} and \textbf{Executing}. In Step \textbf{Sampling}, we apply online sub-sampling to the data collected so far to  reduce the size of the dataset. 
In Step \textbf{Planning}, if we have already collected enough new information (the sub-sampling buffer changes), we call a planner to compute a new policy and then update the current policy to be the new policy. Otherwise, we keep using the old policy. 
In Step \textbf{Executing}, we use the current policy to interact with the environment to collect new data. The full algorithm is presented in Algorithm~\ref{alg:main}.


\begin{algorithm}[tb]
	\caption{RL with Online Sub-Sampling (RLOSS) \label{alg:main}}
	\begin{algorithmic}
		\STATE \textbf{Input:} Failure probability $\delta \in (0,1)$ and number of episodes $K$.
		\STATE $\tilde{k}\leftarrow 1$,
		$\widehat{\mathcal{Z}}_h^1\leftarrow\{\}\quad \forall h\in[H].$
		\FOR {episode $k=1,2,...,K$}
		\FOR {$h=H,H-1,...,1$}
	    \STATE 
	    $\widehat{\mathcal{Z}}_h^k\leftarrow \textbf{Online-Sample}(\mathcal{F},\widehat{\mathcal{Z}}_h^{k-1}, (s_{h}^{k-1},a_{h}^{k-1}),\delta)$\\
	    $(\text{if } k\geq 2 )$
		\ENDFOR
		\IF {$k=1$ \OR $\exists  h\in[H] ~~ \widehat{\mathcal{Z}}^k_h\neq\widehat{\mathcal{Z}}^{{k-1}}_h$}
		
		\STATE $\tilde{k}\leftarrow k$,
		$\pi^k\leftarrow\textbf{Planner}(\widehat{\mathcal{Z}}_h^k,\mathcal{Z}_h^k)$ {\color{blue}// $\tilde{k}$: index of the latest policy}
		\ENDIF
		\STATE Execute policy $\pi^{\tilde{k}}$ to induce a trajectory $s_1^k,a_1^k,r_1^k,...,s^k_H,a^k_H,r^k_H,s^k_{H+1}$
		\ENDFOR
	\end{algorithmic}
\end{algorithm}

\subsection{Theoretical Guarantee}
In this section we provide the theoretical guarantee of Algorithm~\ref{alg:sample}.
We use the eluder dimension~\citep{russo2013eluder} to measure the complexity of the function class. 
We remark that our subsampling algorithm does not require a bounded eluder dimension. However, this condition serves as a sufficient condition for us to obtain an efficient sample bound. We  believe there are other type of complexity measures, which may also provide a bound for the number of samples in the framework.
We also remark that a wide range of function classes, including linear functions, generalized linear functions and bounded degree polynomials, have bounded eluder dimension. 
\begin{defn}[Eluder Dimension]
Let $\varepsilon\ge 0$ and $\pairs =\{(s_i, a_i)\}_{i=1}^n\subseteq \cS\times\cA$ be a sequence of state-action pairs. \\
(1) A state-action pair $(s,a)\in \cS\times \cA$ is \emph{$\varepsilon$-dependent} on $\pairs$ with respect to $\funclass$ if any $f, f'\in \funclass$ satisfying $\|f - f'\|_{\pairs} \le \varepsilon$ also satisfies $|f(s,a) - f'(s,a)|\le \varepsilon$.  \\
(2) An $(s,a)$ is \emph{$\varepsilon$-independent} of $\pairs$ with respect to $\funclass$ if $(s,a)$ is not $\varepsilon$-dependent on $\pairs$. \\
(3) The \emph{$\varepsilon$-eluder dimension} $\ds_E(\funclass,\varepsilon)$ of a function class $\funclass$ is the length of the longest sequence of elements in $\states \times \actions$ such that, for some $\varepsilon' \ge \varepsilon$, every element is $\varepsilon'$-independent of its predecessors.
\end{defn}

In the sequel, we define $d=\max(\log(\mathcal{N}(\mathcal{F},\delta/T^2)),\dim_E(\mathcal{F},1/T),\log(\mathcal{N}(\mathcal{S}\times\mathcal{A},\delta/T^2)))$ to be the complexity of the function class. We now present the guarantee of Algorithm~\ref{alg:sample}.

\begin{thm}[informal]
\label{thm:sample}
For any fixed $\beta\in[1,TH^2]$, with high probability, for any $h\in [H]$, the sub-sampled dataset $\widehat{\mathcal{Z}}_h^k$ ($k=1,2,...,K$) changes for at most $\widetilde{O}(d^2)$ times. Therefore the number of distinct elements in $\widehat{\mathcal{Z}}_h^k$ is also bounded by $\widetilde{O}(d^2)$. And for all $(h,k)\in[H]\times[K]$ and $f_1,f_2\in\cF$, 
\[
\frac1C\|f_1 - f_2\|_{\mathcal{Z}_h^k}
\leq \|f_1 - f_2\|_{\widehat{\mathcal{Z}}_h^k} \leq C\|f_1 - f_2\|_{\mathcal{Z}_h^k}
\]
holds for some constant $C>0$.
\end{thm}
In words, the theorem guarantees that the distance of any two functions measured by the historical dataset is approximated by the subsamping dataset.
Note that Theorem~\ref{thm:sample} directly implies the switching-cost of the framework is at most $\widetilde{O}(d^2H)$ since the size of the subsampled dataset is $\widetilde{O}(d^2H)$. Here we provide a proof sketch of Theorem~\ref{thm:sample} and defer the whole proof to the appendix.

In our proof, we first show that the summation of the online sensitivity scores is upper bounded by $\widetilde{O}(d^2)$ if we use $\mathcal{Z}_h^k$ (the original dataset) instead of $\widehat{\mathcal{Z}}_h^k$ (the sub-sampled dataset) to calculate the online sensitivity scores. 
This is established by a combinatorial argument which draws a connection between the eluder dimension and the summation of the online sensitivity scores. 
However, in our algorithm, for efficiency considerations, we use $\widehat{\mathcal{Z}}_h^k$ instead of $\mathcal{Z}_h^k$ to calculate the sensitivity scores.
Fortunately, as we will show, $\widehat{\mathcal{Z}}_h^k$ provides an accurate estimation to $\mathcal{Z}_h^k$.
Moreover, thanks to the design of the online sensitivity scores, their summation is robust to perturbations on the datasets. 
Hence, the summation of the sensitivity scores can be bounded even if $\widehat{\mathcal{Z}}_h^k$ is used in replace of $\mathcal{Z}_h^k$. 
Note that the summation of the sensitivity scores provides an upper bound on the expected size of the sub-sampled dataset, and a high probability bound can be easily obtained by using martingale concentration bounds. 

In order to show that $\|f_1 - f_2\|_{\widehat{\mathcal{Z}}_h^k}$ is close to $\|f_1 - f_2\|_{\mathcal{Z}_h^k}$, we note that sub-sampling proportional to online sensitivity scores implies that the estimator is unbiased and has low variance, and thus the desired result follows by Bernstein-type martingale concentration bounds.

\subsection{Computational Efficiency} 
\label{sec:compute}
{In this section we discuss the computational efficiency of the online sub-sampling framework. We introduce the following Regression Oracle to rigorously measure the computation complexity. As will be shown in Appendix~\ref{sec:bisearch}, the online sensitivity score can be estimated with $\tilde{O}(1)$ calls to the regression oracle. Since the size of the subsampled dataset is at most $\tilde{O}(d^2)$, the sub-sampling procedure is oracle-efficient -- in each episode, it only calls a $\tilde{O}(d^2)$-sized regression oracle for $\tilde{O}(H)$ times.}

\paragraph{Regression Oracle.} We assume access to a weighted least-squares regression oracle over the function class $\mathcal{G}$, which takes a set of weighted examples $\mathcal{D}=\{(v_i,s_i,a_i,y_i)\}_{i=1}^n\subseteq \mathbb{R}_{+}\times\mathcal{S}\times\mathcal{A}\times\mathbb{R}$ as input, and outputs the function with the smallest weighted squared loss:
$$
\text{ORACLE}(\mathcal{D},\mathcal{G})=\operatorname{argmin}_{g\in\mathcal{G}}\sum_{i=1}^n v_i\left(g(s_i,a_i)-y_i\right)^2.
$$
\begin{remark}
This oracle commonly appears in the literature~\citep{foster2018practical,foster2020beyond,foster2020instance}. 
In the linear setting where $g(s,a)=w_g^\mathrm{T}\phi(s,a)$, the solution to the regression oracle is $w_g^*=\left(\sum_{i=1}^n v_i\phi(s_i,a_i)\phi(s_i,a_i)^\mathrm{T}\right)^{-1}\left(\sum_{i=1}^n v_iy_i\phi(s_i,a_i)\right)$. 
In the general setting where $\cF$ is a general differentiable/sub-differentiable function class like neural networks, this oracle can be solved efficiently using gradient-based algorithms. 
\end{remark}
\section{Efficient RL Algorithm with Value Closedness}
\label{sec:close}
In this section we demonstrate that our subsampling framework can be applied to achieve an efficient algorithm in the following  closedness assumption.\footnote{Our framework is robust to \emph{model mis-specification}, i.e., similar results hold when the function class only satisfies Assumption~\ref{assum:express_1} \emph{approximately}. See Appendix~\ref{sec:mis} for details.}
\begin{assum}[Value Closedness]
\label{assum:express_1}
For all $h\in[H]$ and $V:\mathcal{S}\rightarrow [0,H]$,  $\cT_h V\in\cF$
\end{assum}
In words, the assumption requires that any function will be Bellman-projected to the function class. 
Under Assumption~\ref{assum:express_1}, \citet{wang2020reinforcement}~establish a provably correct algorithm that achieves a regret bound of $\widetilde{O}(\poly(dH)\sqrt{K})$. 
However, their algorithm calls an optimization oracle for at least $\Omega(K)$ times and thus can be time inefficient.
Yet, by leveraging the ideas from \citet{wang2020reinforcement}, we design our planner that can take advantage our subsampling framework and reduce the oracle complexity.
\subsection{Algorithm Design}
\begin{algorithm}[tb]
	\caption{Planner-A($\widehat{\mathcal{Z}}_h^k,\mathcal{Z}_h^k$)\label{alg:plancom}}
	\begin{algorithmic}
		\STATE \textbf{Input:} Datasets  $\widehat{\mathcal{Z}}_h^k,\mathcal{Z}_h^k$
		\STATE $Q_{H+1}^k(\cdot,\cdot)\leftarrow 0$,$V_{H+1}^k(\cdot)\leftarrow 0$
		\FOR {$h=H,H-1,...,1$}
		\STATE $\mathcal{D}_h^k\leftarrow \{(s_{h}^{\tau},a_{h}^{\tau},r_h^\tau+V_{h+1}^k(s_{h+1}^{\tau}))\}_{\tau \in [k-1]}$
		\STATE $f_h^k\leftarrow \text{argmin}_{f\in \mathcal{F}}\|f\|^2_{\mathcal{D}_h^k}$
		\STATE $b_h^k(\cdot,\cdot)\leftarrow \sup_{f_1, f_2 \in \mathcal{F}, \|f_1-f_2\|^2_{\widehat{\mathcal{Z}}^k_h}\leq\beta}|(f_1(\cdot,\cdot)-f_2(\cdot,\cdot)|
		$
		\STATE $Q_h^k(\cdot,\cdot)\leftarrow \min\{f_h^k(\cdot,\cdot)+b_h^k(\cdot,\cdot),H\}$  
		\STATE $V_h^k(\cdot)\leftarrow\max_{a\in \mathcal{A}}Q_h^k(\cdot,a)$
		\STATE $\pi_h^k(\cdot)\leftarrow \text{argmax}_{a\in \mathcal{A}}Q_h^k(\cdot,a)$
		\ENDFOR
		\STATE\textbf{Output: policy $\pi^k$}
	\end{algorithmic}
\end{algorithm}
The core idea in this planner is to calculate the optimistic Q-functions using least-squares value iteration.
For each $h=H,H-1,...,1$, we solve the following optimization problem:
\begin{align*}
f_h^k\gets\argmin_{f\in \cF}
\sum_{\tau=1}^{k-1}\left( f(s_h^\tau,a_h^\tau)-\left(r_h^\tau+\max_{a\in\mathcal{A}}Q_{h+1}^k(s_{h+1}^\tau,a)\right) \right)^2
\end{align*}
and set the estimated Q-function to be
$
Q_h^k(\cdot,\cdot)\leftarrow \min\left\lbrace f_h^k(\cdot,\cdot)+b_h^k(\cdot,\cdot),H\right\rbrace,
$
where $b_h^k(\cdot,\cdot)$ is an exploration bonus defined on the sub-sampled dataset to measure the length of the confidence interval,
$$ b_h^k(\cdot,\cdot) \gets \sup_{f_1, f_2\in \mathcal{F},\|f_1-f_2\|^2_{\widehat{\mathcal{Z}}^k_h}\leq\beta} |(f_1(\cdot,\cdot)-f_2(\cdot,\cdot)|.
$$
The policy is then defined as the greedy policy with respect to $Q_h^k$.

We study the computation complexity of this planner in terms of the regression oracle introduced in Section~\ref{sec:compute}. As will be shown in Appendix~\ref{sec:bisearch}, the computation of $b_h^k$ can be reduced to $\widetilde{O}(1)$ calls to a regression oracle of size $\widetilde{O}(d^2)$ , whereas $f_h^k$ requires to call a regression oracle of size $\Omega(K)$. Fortunately, we only need to call this policy planner for $\widetilde{O}(d^2H)$ times, thus the number of ``big'' oracle calls is bounded by $\widetilde{O}(d^2H^2)$. To execute the policy $\pi^{\tilde{k}}$ in episode $k$, we also need to compute $b_h^{\tilde{k}}(s,\cdot)$ for all encountered state $s$, which requires $\widetilde{O}(A)$ calls to a regression oracle of size $\widetilde{O}(d^2)$ per time step.
\subsection{Theoretical Guarantee} 
Here we present the theoretical guarantee of Algorithm~\ref{alg:main} equipped with Planner-A. The computation complexity is measure by the number of calls to regression oracle with size $\Omega(K)$.
\begin{thm} 
\label{thm:main_regret}
Assume Assumption~\ref{assum:cover} and Assumption~\ref{assum:express_1} hold and $T$ is sufficiently large. 
Equipped with Planner-A, with probability $1-\delta$, Algorithm~\ref{alg:main} achieves a regret bound 
\[\textstyle{
\operatorname{Regret}(K) = O(\sqrt{ \iota_1\cdot H^3\cdot T })}
\]
where 
$$
\iota_1=\log(T\mathcal{N}(\mathcal{F},\delta/T^2)/\delta)\cdot \dim^2_E(\mathcal{F},1/T)
\cdot\log^2 T \cdot\log\left(\mathcal{N}(\mathcal{S}\times\mathcal{A},\delta/T^2) \cdot T/\delta\right).
$$
Furthermore, with probability $1-\delta$ the algorithm calls a $\Omega(K)$-sized regression oracle for at most $\widetilde{O}(d^2H^2)$ times.
\end{thm}
Note that  our regret bound is the same with that in~\citet{wang2020reinforcement} when applied to the same setting, whereas our number of oracle calls is much smaller. 
\begin{remark}
\label{rmk:2}
When specialized to the linear MDP setting, i.e., when $\cF$ is the class of $d_{\text{\emph{lin}}}$-dimensional linear functions, the global switching cost bound of our algorithm is $\widetilde{O}(d_{\text{\emph{lin}}}^2H)$, which is worse than the $\widetilde{O}(d_{\text{\emph{lin}}}H)$ bound given in~\citet{gao2021provably}. 
However, for linear functions, our sampling procedure is equivalent to the online leverage score sampling~\citep{cohen2016online}, and therefore, by using the analysis in ~\citep{cohen2016online} which is specific to the linear setting, the switching cost bound can be improved to $\widetilde{O}(d_{\text{\emph{lin}}}H)$, matching the bound given in~\citet{gao2021provably}. Using the same technique, our regret bound can be improved to $\widetilde{O}(\sqrt{d_{\text{\emph{lin}}}^3H^3T})$ in the linear setting, matching the bound given in~\citet{jin2020provably,gao2021provably}.
\end{remark}
\begin{remark}
\label{rmk3}
If we assume the time cost of the regression oracle scales linearly with its data size, we can show that the whole time-complexity of our algorithm is $\widetilde{O}(\poly(dH)K)$. 
Under the same assumption the whole time-complexity of the algorithm proposed in \citet{wang2020reinforcement} is at least $\Omega(K^2)$.
\end{remark}
\subsubsection{Proof Sketch of Theorem 2}

The first step of the proof is to show optimism, which requires that our exploration bonus upper bounds the estimation error. As our exploration bonus also has bounded complexity, optimism can be established using similar uniform convergence argument as in ~\citet{jin2020provably, wang2020reinforcement}.

We switch the policy only when the sub-sampled dataset is changed; this can also be understood from a regret decomposition perspective. We use $\tilde{k}$ to denote the index of the policy used in the $k$-th episode. The algorithm design naturally guarantees that $\widehat{\cZ}_h^k= \widehat{\cZ}_h^{\tilde{k}}$. With optimism, the regret can be decomposed as:
$$
\text{Regret}(K)\leq \sum_{k=1}^K \left(V_1^{\tilde{k}}(s_1)-V_1^{\pi_{\tilde{k}}}(s_1)\right) \leq \sum_{k=1}^K\sum_{h=1}^{H-1}\xi_h^k+2\sum_{k=1}^K\sum_{h=1}^H b_h^{\tilde{k}}(s_h^k,a_h^k)
$$
where $\{\xi_h^k\}$ is a sequence of martingale differences, and $b_h^{\tilde{k}}$ is the bonus function in the $\tilde{k}$-th episode.
Note that our exploration bonus is defined on the sub-sampled dataset $\widehat{\cZ}_h^k$, and therefore $\widehat{\cZ}_h^{\tilde{k}}=\widehat{\cZ}_h^k$ indicates $b_h^{\tilde{k}}=b_h^k$. Then using ideas in~\citet{russo2013eluder}, $\sum_{k=1}^K\sum_{h=1}^H b_h^{k}(s_h^k,a_h^k)$ can be upper bounded in terms of the eluder dimension.

\section{Efficient RL Algorithm with Value Completeness}
\label{sec:complete}
In this section we demonstrate the efficacy of our framework under the following value completeness assumptions. Note that this assumption is more general than the value closedness assumption and thus captures a larger set of function classes. 
\begin{assum}[Value Completeness]
\label{assum:express_3}
For all $h\in[H]$, $\mathcal{T}_h\cF_{[0,H-h]}\subseteq\cF$, where $\cF_{[0,H-h]}:=\{f\in\cF:\|f\|_{\infty}\leq H-h\}$ is the truncated function class in step $h$.
\end{assum}
In words, this assumption requires that the function class is closed under the Bellman projection. However, unlike value closedness, the projection of an arbitrary function might not lie in the function class. Moreover, this assumption does not guarantee that the optimal Q-function is captured by the function class. Hence, we also require the following standard realizability assumption. 
\begin{assum}[Realizability]
\label{assum:real}
For all $h\in[H]$, $Q_h^*\in\cF$.
\end{assum}
Under these two assumptions, \citet{jin2021bellman} develop an algorithm, GOLF, whose regret scales with the \emph{Bellman eluder dimension} of the function class $\cF$, denoted as $\dim_{BE}(\cF,\epsilon)$. The definition and discussion of the Bellman eluder dimension is deferred to Appendix~\ref{sec:bellman}.
\subsection{Algorithm Design}
In this section we use the planner in \cite{jin2021bellman}. The main component of this planner is the construction of the confidence set $\cB^k$. For each $h\in[H]$, the planner maintains a local regression constraint using the collected transition data. The constraint essentially requires low empirical bellman error.
Then we solve a \emph{global} optimization  problem to ensure the optimism of the learned Q-function:
\begin{equation}
\label{eqn:oracle}
\text{max}_{Q^k\in\mathcal{B}^k}(\max_{a\in\mathcal{A}}Q_1^k(s_1,a)).
\end{equation}
The policy is defined as the greedy policy with respect to $Q^k$. The full algorithm is presented in Algorithm~\ref{alg:plansample}.
\begin{algorithm}[tb]
	\caption{Planner-B($\widehat{\mathcal{Z}}_h^k,\mathcal{Z}_h^k$)\label{alg:plansample}}
	\begin{algorithmic}
		\STATE \textbf{Input:} Datasets  $\widehat{\mathcal{Z}}_h^k,\mathcal{Z}_h^k$
		
		\STATE $\mathcal{D}_h^k(f)\leftarrow \{(s_{h}^{\tau},a_{h}^{\tau},r_h^\tau+\max_{a\in\cA}f(s_{h+1}^{\tau},a))\}_{\tau \in [k-1]}$ for all $h\in[H]$
		\STATE $\mathcal{B}^k\leftarrow\{\{f_h\}_{h=1}^{H+1}:\forall h\in[H],\|f_{h}\|_{\infty}\leq H+1-h,\|f_h\|^2_{\mathcal{D}_h^k(f_{h+1})}\leq\text{inf}_{g\in \mathcal{F}}\|g\|^2_{\mathcal{D}_h^k(f_{h+1})}+\beta,f_{H+1}\equiv 0\}$
		\STATE 
		$
		(Q_1^k,Q_2^k,...,Q_{H+1}^k)\leftarrow \operatorname{argmax}_{Q^k\in\mathcal{B}^k}(\max_{a\in\mathcal{A}}Q_1^k(s_1,a))
		$
		\STATE $\pi_h^k(\cdot)\leftarrow \text{argmax}_{a\in \mathcal{A}}Q_h^k(\cdot,a)$
		\STATE\textbf{Output: policy $\pi^k$}
	\end{algorithmic}
\end{algorithm}
\subsection{Theoretical Guarantee}
Note that \eqref{eqn:oracle} is a multi-level optimization problem, where the constraint region of the upper-level problem is determined implicitly by the solution set of the lower level problem. Unfortunately for now we do not know how to solve this problem efficiently. We therefore assume a \emph{nested optimization oracle} to solve this problem. Compared to the Algorithm GOLF proposed in \citet{jin2021bellman}, the number of oracle calls in our algorithm decreased from $\Omega(K)$ to $\widetilde{O}(d^2H)$, and the regret bound is the same when applied to the same setting. The theoretical guarantee is formally stated as:
\begin{thm} 
\label{thm:main_regret2}
Assume Assumption~\ref{assum:cover}, Assumption~\ref{assum:express_3} and Assumption~\ref{assum:real} holds and $T$ is sufficiently large. 
Equipped with Planner-B, with probability $1-\delta$ Algorithm~\ref{alg:main}
achieves a regret bound 
$$
\operatorname{Regret}(K) = O(\sqrt{ \iota_1\cdot H^3\cdot T })
$$
where 
$$
\iota_1=\log(T\mathcal{N}(\mathcal{F},1/K)/\delta)\cdot \dim_{BE}(\mathcal{F},1/\sqrt{K}).
$$
Here $\dim_{BE}(\cdot,\cdot)$ denotes the Bellman eluder dimension. Furthermore, with probability $1-\delta$ the algorithm calls the nested optimization oracle for at most $\widetilde{O}(d^2 H)$ times.
\end{thm}
We postpone the proof sketch to Appendix~\ref{proof_thm3}. The key step is similar to the one of Theorem~\ref{thm:main_regret}, i.e., decomposing the regret properly in the face of delayed policy updates.


\section{Application to Other Settings}
In this section, we show the generality of our online sub-sampling technique by applying it to different RL tasks.
\subsection{Reward-Free RL}
The reward-free RL contains two phases, the \emph{exploration phase} and the \emph{planning phase}. In the exploration phase, the agent interacts with the MDP in episodes as usual, but receives no reward signal. After the exploration phase, the agent is given a reward function in the planning phase. The reward-free RL aims to output a near-optimal policy with respect to the given reward function with no additional access to the environment. As mentioned in~\citet{jin2020reward} and~\citet{wang2020reward}, this paradigm is particularly suitable for the batch RL setting and the setting where there are multiple reward functions of interest.

Our algorithm is based on the ``exploration-driven reward'' technique proposed in \citet{wang2020reward}, which studies reward-free RL with linear function approximation. We further run the online sub-sampling algorithm to control policy updates in the exploration phase. We present an informal theoretical guarantee here, and the details are deferred to the appendix.

\textbf{Theorem 4} (informal)\textbf{.} \emph{Under standard assumptions, with high probability, our algorithm guarantees to output an $\epsilon$-optimal policy for any given reward function after exploring the environment for $\widetilde{O}(d^4 H^6\epsilon^{-2})$ episodes, while only switches the policy for $\widetilde{O}(d^2H)$ times during the exploration phase. Furthermore, with high probability, our algorithm calls a $\Omega(K)$-sized regression oracle for at most $\widetilde{O}(d^2H^2)$ times. Here $d$ depends on the eluder dimension of the target function class.}
\subsection{Multi-Agent RL}
We consider the task of learning two-player zero-sum Markov games with general function approximation. This setting is formerly studied in \citet{jin2022power,huang2021towards}. We focus on the online setting, where the learner can only control the max-player, and an adversarial opponent controls the min-player.\footnote{Similar results hold for the self-play setting using the technique of the ``exploiter'' introduced in \citet{jin2022power}.} The learning goal is to minimize the regret compared to the value of the Nash equilibrium. 

Based on the algorithms proposed in \citet{jin2022power,huang2021towards}, we run the online subsampling algorithm to control policy updates for the max-player. Again, we present an informal theoretical guarantee here and defer the details to the appendix. 

\textbf{Theorem 5} (informal)\textbf{.}
\emph{Under standard assumptions, with high probability, the regret of the max-player is bounded by $\tilde{O}(\sqrt{d^2\cdot H^4\cdot K})$. At the same time, the max-player only switches her policy for $\tilde{O}(d^2 H)$ times. Furthermore, with high probability, our algorithm calls the nested optimization oracle for at most $\widetilde{O}(d^2 H)$ times. Here $K$ is the number of episodes, and $d$ depends on the online Bellman eluder dimension~\citep{jin2022power} of the target function class.}
\section{Discussions}
\subsection{Practical Impacts} Though our work is of theoretical nature, here we discuss the potential practical impacts of our work. The most direct practical impact is the improvement over \citet{wang2020reinforcement}. A notable drawback of applying the algorithm in \citet{wang2020reinforcement} in practice is its bad computational efficiency. Equipped with our framework, the whole time complexity of the algorithm grows linear in $K$ (under the time cost assumption of the regression oracle, see Remark~\ref{rmk3}), making it more suitable for practical use.

Another direction is to combine our framework with a heuristic optimizer to obtain more practical algorithms. For example, in DQN, the agent doesn’t need to train after each step. Instead, it uses experience replay to train on small batches once every $n$ steps, where $n$ is a tuning parameter, but is fixed during the entire course of training. Our subsampling approach would introduce an efficient and adaptive approach to adjust $n$ – the network is retrained whenever the policy needs a switch. Note that the bonus function defined in our paper can be implemented for DNN and has been implemented in \citet{feng2021provably}. While this idea needs further experimental verification, we believe our methods can inspire more efficient practical algorithms. 

\subsection{Limitations and Future Works}
\label{limit}
Admittedly, our current theoretical guarantees rely heavily on the eluder dimension structures of the function class, i.e., the sub-sampled dataset size can be bounded by a function of the eluder dimension. Here we discuss future directions to extend our framework to a broader scope. 
\paragraph{Information-Theoretic View Point.} One can understand the structure of our framework more generally from an information-theoretic point of view: the algorithm switches the policy when there is sufficient information gained, and the data points causing the information change are important points to be stored. Hence, we believe such an algorithmic structure can be further summarized and abstracted to fit a broader range of settings. We remark that the idea of measuring information gain has been further studied in \citet{lu2019information}, which could be integrated with our techniques for sampling. 

Along the direction of accelerating existing algorithms, it would also be interesting to see whether our techniques can be used to accelerate the algorithms in \citet{foster2021statistical}, which is the most general framework for RL with general FA up to now, subsuming the frameworks of \citet{jiang2017contextual}, \citet{jin2021bellman}, and \citet{du2021bilinear}.

\bibliography{ref}
\bibliographystyle{plainnat}
\newpage
\appendix
\tableofcontents
\section{Additional Notations, Parameter Setting, Additional Related Works}
In this section we provide additional notations used in the analysis, choice of the parameter $\beta$, and additional related works.
\subsection{Additional Notations}
For $n$ events $\mathcal{E}_1, \mathcal{E}_2, \ldots, \mathcal{E}_n$, we write 
\[
\Pr(\mathcal{E}_1\mathcal{E}_2\ldots\mathcal{E}_n)  = \Pr(\mathcal{E}_1 \cap \mathcal{E}_2 \cap \ldots \cap \mathcal{E}_n). 
\]
For a event $\mathcal{E}$, we use $\mathbb{I}\{\mathcal{E}\}$ to denote the indicator function, i.e., 
\[
\mathbb{I}\{\mathcal{E}\}=
\begin{cases}
1&\mathcal{E}\text{ holds}\\
0&\text{otherwise}\\
\end{cases}.
\]
For a event $\mathcal{E}$, we use $\mathcal{E}^c$ to denote its complement. 
For a multiset $\cZ$, we use $|\cZ|$ to denote the cardinality of $\cZ$, and $n_d(\cZ)$ the number of distinct elements in $\cZ$.

\subsubsection{Bellman Eluder Dimension}
In this section, we define the Bellman eluder dimension. Please refer to \citet{jin2021bellman} for more discussions on this concept.

Firstly, we introduce the distributional eluder dimension.
\label{sec:bellman}

\begin{defn}[Distributional Eluder Dimension]
Let $\varepsilon\ge 0$ and $\cZ =\{\mu_1,\mu_2,...,\mu_n\}_{i=1}^n\subseteq \Delta(\cS\times\cA)$ be a sequence of probability measures over state-action space. \\
(1) A probability measure $\mu\in \Delta(\cS\times \cA)$ is \emph{$\varepsilon$-dependent} on $\cZ$ with respect to $\cG$ if any $g\in \cG$ satisfying $\sqrt{\sum_{\mu\in\cZ}\mathbb{E}_{\mu}g(s,a)^2} \le \varepsilon$ also satisfies $\mathbb{E}_{\mu}|g(s,a)|\le \varepsilon$.  \\
(2) An $\mu$ is \emph{$\varepsilon$-independent} of $\cZ$ with respect to $\cG$ if $\mu$ is not $\varepsilon$-dependent on $\cZ$. \\
(3) The \emph{$\varepsilon$-distributional eluder dimension} $\ds_{DE}(\cG,\Pi,\varepsilon)$ of a function class $\cG$ is the length of the longest sequence of elements which take values in $\Pi$ such that, for some $\varepsilon' \ge \varepsilon$, every element is $\varepsilon'$-independent of its predecessors.
\end{defn}
Typically there are two choices for the distribution family $\Pi$:\\
(1) We use $\cD_{\cF,h}$ to denote the collection of all probability measures over $\cS\times\cA$ at the $h^{\text{th}}$ step, which can be generated by excuting the greedy policy $\pi_{(f_1,f_2,...,f_H)}$ induced by any $f_1,f_2,...,f_h\in\cF$, i.e., $\pi_{f,h}(\cdot)=\text{argmax}_{a\in\cA}f_h(\cdot,a)$.\\
(2) $\cD_{\Delta}:=\{\delta_{(s,a)}(\cdot)|s\in\cS,a\in\cA\}$, i.e., all single point distribution.

Then the Bellman eluder dimension is stated as:
\begin{defn}[Bellman Eluder Dimension] Let $(I-\cT_h)\cF:=\{f-\cT_h f':f,f'\in\cF\}$ be the set of bellman residuals induced by $\cF$ at step $h$. The $\varepsilon$-Bellman eluder dimension of $\cF$ is defined as $\dim_{BE}(\cF,\varepsilon):=\min_{\Pi\in\{\cD_{\cF},\cD_{\Delta}\}}\max_{h\in[H]}\dim_{DE}((I-\cT_h)\cF,\Pi_h,\varepsilon)$.
\end{defn}
\subsection{Choice of the parameter $\beta$}
\label{sec:beta}
We elaborate the choice of $\beta$. When using Planner-A, we set $\beta$ to be
\[
\beta=CH^2\cdot\log(T\mathcal{N}(\mathcal{F},\delta/T^2)/\delta) \cdot\dim_E(\mathcal{F},1/T)\cdot\log^2 T\cdot
\log\left(\mathcal{C}(\mathcal{S}\times\mathcal{A},\delta/(T^2)) \cdot T/\delta\right).
\]
When using planner-B, we set $\beta$ to be
\[
\beta=CH^2\cdot \log(T\cN(\cF,1/K)/\delta).
\]
In the reward-free RL setting (Algorithm~\ref{alg:exploration} and Algorithm~\ref{alg:planning}, will be introduced later), we set $\beta$ to be
\begin{align*}
\beta=CH^2\cdot(\log(\mathcal{N}(\mathcal{R},1/T))&\cdot\dim_E(\mathcal{F},1/T)\\
+\log(T\mathcal{N}(\mathcal{F},\delta/T^2)/\delta) \cdot\dim_E(\mathcal{F},1/T)&\cdot\log^2 T\cdot
\log\left(\mathcal{C}(\mathcal{S}\times\mathcal{A},\delta/(T^2)) \cdot T/\delta\right)).
\end{align*}
\subsection{Additional Related Works}
\paragraph{Online Sub-Sampling.} Our core technique, sub-sampling by online sensitivity score, is inspired by the sensitivity sampling technique introduced in~\cite{langberg2010universal,feldman2011unified,feldman2013turning} and 
the online leverage score sampling technique introduced in~\cite{cohen2016online}.
However, the algorithm and analysis in~\citep{cohen2016online} works only for linear functions, and the framework in~\cite{langberg2010universal,feldman2011unified,feldman2013turning} can only deal with static datasets.
On the other hand, our techniques can deal with general function classes while operate in an online manner. 

\paragraph{Tabular RL.}
There is a long line of theoretical work on the sample complexity and regret bound for RL in the tabular setting. See, e.g.,~\citep{kearns2002near,kakade2003sample,szita2010model,jaksch2010near,azar2013minimax,osband2016lower,azar2017minimax,jin2018q,sidford2018near,zanette2019tighter,agarwal2020model,wang2020long,zhang2020reinforcement,sidford2020solving,cui2020minimax,li2020breaking,li2021breaking,menard2021ucb,xiong2021randomized} and references therein. However, as these results all depend polynomially on the size of the state space, they can not be directly applied to real-world problems with large state spaces. 

\paragraph{RL with Linear Function Approximation.}
The former most basic and frequently studied setting is RL with linear function approximation. See, e.g.,~\citep{yang2019sample,yang2020reinforcement,jin2020provably,du2019good,wang2019optimism,zanette2020learning,zanette2020provably,du2020agnostic,cui2020plug,agarwal2020flambe,zhou2021nearly,hu2022nearly} for recent theoretical advances. 
\paragraph{RL with Low Switching Cost.}
\citet{bai2019provably}~is the first work that studies switching cost in RL. This problem was later studied in \cite{zhang2020almost,gao2021provably, wang2021provably,huang2022towards,qiao2022sample}. Our work focus on the \emph{global} switching cost studied in \cite{gao2021provably}.
\paragraph{Reward-Free RL.}
The reward-free RL setting is proposed in~\cite{jin2020reward}. \citet{kaufmann2020adaptive} refine the algorithm proposed in \cite{jin2020reward} with improved sample complexity. 
\citet{wang2020reward,zanette2020provably,qiu2021reward,wagenmaker2022reward,chen2022statistical,chen2022unified}~further design provably efficient algorithms with function approximation in the reward-free setting. 
\paragraph{Learning Two-Player Zero-Sum Games.} Recently there is a line of theoretical works on efficiently learning two-player zero-sum games. See, e.g., \citet{zhang2020model,sidford2020solving,bai2020provable,liu2021sharp,jin2021v} and references therein. \citet{xie2020learning,chen2021understanding,jin2022power,huang2021towards} further consider learning two-player zero-sum games with function approximation. Our work focuses on the general function approximation setting studied in \citet{huang2021towards,jin2022power}.
\section{Computing Bonus and Sensitivity via Regression Oracle}
\label{sec:bisearch}
The first step is to show that the following constrained optimization problem can be solved within $\tilde{O}(1)$ calls to the regression oracle:
\begin{align}\label{opt:constrained}
\text{max } f_1(s,a)-f_2(s,a)\text{ s.t. }\|f_1-f_2\|_{\mathcal{Z}}^2\leq \epsilon 
, \ f_1,f_2\in\mathcal{F}.
\end{align}
This can be accomplished using ideas from~\citet{foster2018practical}. The idea is solving the following regression problem to estimate the solution to the constrained optimization problem:
\[
\min  \|f_1-f_2\|_{\mathcal{Z}}^2+\frac w2(f_1(s,a)-f_2(s,a)-2H)^2
\]
we do binary search over $w$ to find the proper value of $w$.
The full algorithm is presented in Algorithm~\ref{bisearch}.
\begin{algorithm}
	\caption{Binary Search\label{bisearch}}
	\begin{algorithmic}[1]
		\STATE \textbf{Input:} Dataset $\mathcal{Z}$, objective $(s,a)$, tolerance $\beta$, precision $\alpha$
		\STATE $\mathcal{G}\leftarrow \mathcal{F}-\mathcal{F}$
	    \STATE $R(g,w):=\|g\|_{\mathcal{Z}}^2+\frac w2(g(s,a)-2(H+1))^2$, $\forall g\in \mathcal{G}$
	    \STATE $w_L\leftarrow 0$, $w_H\leftarrow \beta/(\alpha(H+1))$
	    \STATE $g_L\leftarrow 0$, $z_L\leftarrow 0$
	    \STATE $g_H\leftarrow \text{argmin}_{g\in\mathcal{G}}R(g,w_H)$, $z_H\leftarrow g_H(s,a)$
	    \STATE $\Delta\leftarrow \alpha\beta/(8(H+1)^3)$
	    \WHILE {$|z_H-z_L|>\alpha$ and $|w_H-w_L|>\Delta$}
		\STATE $\widetilde{w}\leftarrow (w_H+w_L)/2$
		\STATE $\widetilde{g}\leftarrow \text{argmin}_{g\in\mathcal{G}}R(g,\widetilde{w})$, $\widetilde{z}\leftarrow \widetilde{g}(s,a)$
		\IF {$\|\widetilde{g}\|^2_{\mathcal{Z}}>\beta$}
		\STATE $w_H\leftarrow \widetilde{w}$, $z_H\leftarrow \widetilde{z}$
		\ELSE 
		\STATE $w_L\leftarrow \widetilde{w}$, $z_L\leftarrow \widetilde{z}$
		\ENDIF
		\ENDWHILE
		\STATE \textbf{Output: $z_H$} 
	\end{algorithmic}
\end{algorithm}
The following theorem shows that when $\cF$ is convex, algorithm \ref{bisearch} solves the constrained optimization problem up to a precision of $\alpha$ in $O(\log(1/\alpha))$ iterations, i.e., $O(\log(1/\alpha))$ oracle invocations. When $\cF$ is not convex, the constrained optimization problem can be solved with $O(1/\alpha)$ oracle invocations using the techniques in \cite{krishnamurthy2017active}.
\begin{prop}
\label{thm:oracle}
Assume that the optimal solution to the following constrained optimization problem is $g^*=f_1^*-f_2^*$.
$$
\text{maximize } f_1(s,a)-f_2(s,a)
$$
$$
\text{subject to }\|f_1-f_2\|_{\mathcal{Z}}^2\leq\beta 
, \quad f_1,f_2\in\mathcal{F}
$$
We run algorithm \ref{bisearch} to solve the above problem. If the function class $\mathcal{F}$ is convex and closed under pointwise convergence, then algorithm \ref{bisearch} terminate after $O(\log(1/\alpha))$ oracle invocations and the returned values satisfy
$$
|z_H-g^*(s,a)|\leq \alpha.
$$
\end{prop}
\begin{proof}
Note that if $\mathcal{F}$ is convex, then $\mathcal{F}-\mathcal{F}$ is also convex due to the following equation:
$$
\lambda(f_1-f_2)+(1-\lambda)(f_3-f_4)=(\lambda f_1 +(1-\lambda)f_3)-(\lambda f_2+(1-\lambda)f_4).
$$
The rest of the proof is identical to Theorem 1 of \cite{foster2018practical}. We omit it here for brevity.
\end{proof}
The second step is to show how to reduce the computation of the online sensitivity scores to the constrained optimization problem in~\eqref{opt:constrained}. Indeed, estimation of the online sensitivity score can be reduced to the following
optimization problems:
$
\text{max } f_1(s,a)-f_2(s,a)
\text{ s.t. }\|f_1-f_2\|_{\mathcal{Z}}^2\leq 2^\alpha
, \ f_1,f_2\in\mathcal{F}
$
for $\alpha \in \{0,1,...,\log(T(H+1)^2),+\infty\}$.
Indeed, assuming the solution of the above problem is $f_1^\alpha,f_2^\alpha$ for some $\alpha$, and let
\begin{align*}
&\sen^{\text{est}}_{\mathcal{Z},\mathcal{F}}(z) = \max_{\alpha}\left\lbrace\min\left\lbrace \frac{(f_1^\alpha(z)-f_2^\alpha(z))^2}{\min\{\|f_1^\alpha-f_2^\alpha\|^2_{\mathcal{Z}},T(H+1)^2\}+\beta},1\right\rbrace\right\rbrace.
\end{align*}
We then have $1\leq\sen_{\mathcal{Z},\mathcal{F}}(z)/\sen^{\text{est}}_{\mathcal{Z},\mathcal{F}}(z)\leq 2$.
Note that a 2-approximation of the sensitivity score is sufficient for our analysis.
Note that the exploration bonus naturally has the form of \eqref{opt:constrained}. Hence, both the exploration bonus and the sensitivity scores can be computed using $\widetilde{O}(1)$ regression oracle calls.

\section{Application in other settings}
\subsection{Reward-Free RL Setting}
In this section we show that our results can be extended to the reward-free exploration setting, in which the agent explores the environment without the guidance of a reward, while achieving both sample efficiency and computation efficiency.
We begin with some basics and notations of reward-free RL.
\subsubsection{Reward-Free RL}
The reward-free RL contains two phases, the \emph{exploration phase} and the \emph{planning phase}. In the exploration phase, the agent interacts with the MDP in episodes as usual, but receives no reward signal. After the  exploration phase, the agent is given a reward function in the planning phase. The goal of the reward-free RL is to output a near-optimal policy with respect to the given reward function with no additional access to the environment. As mentioned in~\citet{jin2020reward} and~\citet{wang2020reward}, this paradigm is particular suitable for the batch RL setting and the setting where there are multiple reward functions of interest.

\paragraph{Notations.} Slightly changing the notation, we define the value (action-value) functions with respect to a given reward function $r=\{r_h\}_{h=1}^H$ as 
$$
V_h^{\pi}(s,r)=\mathbb{E}\left[\sum_{h'=h}^Hr_{h'}(s_{h'},a_{h'})\left|\right.s_h=s,\pi\right]
$$
and 
$$
Q_h^{\pi}(s,a,r)=\mathbb{E}\left[\sum_{h'=h}^Hr_{h'}(s_{h'},a_{h'})\left|\right.s_h=s,a_h=a,\pi\right].
$$
The optimal value (action-value) functions $V_h^*(s,r)$ and $Q_h^*(s,a,r)$ are defined similarly. We say a policy $\pi$ is a $\varepsilon$-optimal policy with respect to $r$ if $V_1^*(s_1,r)-V_1^{\pi}(s_1,r)\leq \varepsilon$.
\subsubsection{Function Class}
The reward-free RL setting also requires a function class $\mathcal{F}\subset \{f:\mathcal{S}\times\mathcal{A}\rightarrow[0,H+1]\}$ given as a priori. The following version of closedness assumption is needed 
\begin{assum}
\label{assum:express_2}
For any $h\in[H]$ and any $V:\mathcal{S}\rightarrow [0,H]$,  
\[
\cT_h V\in\cF
\]
\end{assum}
Compared to Assumption~\ref{assum:express_1}, with no reward function, Assumption~\ref{assum:express_2} can be regarded as a constrain on the transition core. Intuitively, this assumption guarantees that we can use function class $\mathcal{F}$ to effectively explore the transition operator.

In Linear MDPs, it is assumed that the reward function is linear in the feature extractor. Instead of  making explicit assumption on the structure of the reward function, we assume the reward function given in the planning phase belongs to a function class with bounded covering number.
\begin{assum} 
\label{assum:reward}
The reward function $r=\{r_h\}_{h=1}^H$ belongs to a function class $\mathcal{R}\subseteq\{\mathcal{S}\times\mathcal{A}\rightarrow[0,1]\}$, i.e., $r_h\in\mathcal{R}$ for all $h\in[H]$.
And for any $\varepsilon>0$, there exists an $\varepsilon$-cover $\mathcal{C}(\mathcal{R},\varepsilon)$ with size $|\mathcal{C}(\mathcal{R},\varepsilon)|\leq \mathcal{N}(\mathcal{R},\varepsilon).$
\end{assum}
\subsubsection{Algorithm}

The algorithm consists of two phases: an exploration phase and a planning phase. Here we use the planner in \citet{wang2020reinforcement}. 
Below, we introduce our algorithm.
\begin{algorithm}[h]
	\caption{Exploration Phase\label{alg:exploration}}
	\begin{algorithmic}
		\STATE \textbf{Input:} Failure probability $\delta \in (0,1)$, number of episodes $K$, and function class $\cF$.
		\STATE $\tilde{k}\leftarrow 1$
		\STATE $\widehat{\mathcal{Z}}_h^1\leftarrow\{\}\quad \forall h\in[H]$
		\FOR {episode $k=1,2,...,K$}
		\FOR {$h=H,H-1,...,1$}
		\STATE $\widehat{\mathcal{Z}}_h^k\leftarrow \textbf{Online-Sample}(\mathcal{F},\widehat{\mathcal{Z}}_h^{k-1}, (s_{h}^{k-1},a_{h}^{k-1}),\delta)$  (if $k\geq 2$)	
	    \ENDFOR
		\IF 
		{
		$k=1$ 
		\OR
		$\exists h\in[H] ~~ \widehat{\mathcal{Z}}^k_h\neq\widehat{\mathcal{Z}}^{\tilde{k}}_h$}
		\STATE $\tilde{k}\leftarrow k$
		\STATE $Q_{H+1}^k(\cdot,\cdot)\leftarrow 0$,$V_{H+1}^k(\cdot)\leftarrow 0$
		\FOR {$h=H,H-1,...,1$}
		\STATE $\mathcal{D}_h^k\leftarrow \{(s_{h}^{\tau},a_{h}^{\tau},V_{h+1}^k(s_{h+1}^{\tau}))\}_{\tau \in [k-1]}$
		\STATE $f_h^k\leftarrow \text{argmin}_{f\in \mathcal{F}}\|f\|^2_{\mathcal{D}_h^k}$
		\STATE $b_h^k(\cdot,\cdot)\leftarrow \sup_{\|f_1-f_2\|^2_{\widehat{\mathcal{Z}}_h^k}\leq\beta}|(f_1(\cdot,\cdot)-f_2(\cdot,\cdot)|$
		\STATE $r_h^k(\cdot,\cdot)\leftarrow \min\{b_h^k(\cdot,\cdot)/H,1\}$
		\STATE $Q_h^k(\cdot,\cdot)\leftarrow \min\{f_h^k(\cdot,\cdot)+b_h^k(\cdot,\cdot)+r_h^k(\cdot,\cdot),H\}$ 
		\STATE $V_h^k(\cdot)\leftarrow \max_{a\in \mathcal{A}}Q_h^k(\cdot,a)$
		\STATE $\pi_h^k(\cdot)\leftarrow \text{argmax}_{a\in \mathcal{A}}Q_h^k(\cdot,a)$
		\ENDFOR
		\ENDIF
		\STATE Execute policy $\pi^{\tilde{k}}$ to induce a trajectory $s_1^k,a_1^k,r_1^k,...,s^k_H,a^k_H,r^k_H,s^k_{H+1}$
		\ENDFOR
	\end{algorithmic}
\end{algorithm}

\begin{algorithm}[tb]
	\caption{Planning Phase\label{alg:planning}}
	\begin{algorithmic}
		\STATE \textbf{Input:} Dataset $\mathcal{Z}_h^K,h\in[H]$, subsampled dataset $\widehat{\mathcal{Z}}_h^K,h\in[H]$, reward function $r=\{r_h\}_{h=1}^H$, and function class $\cF$.
		\STATE $Q_{H+1}(\cdot,\cdot)\leftarrow 0$,$V_{H+1}(\cdot)\leftarrow 0$
	    \STATE $\mathcal{Z}_h\leftarrow \mathcal{Z}_h^K, \widehat{\mathcal{Z}}_h\leftarrow \widehat{\mathcal{Z}}_h^K$
	    \FOR {$h=H,H-1,...,1$}
		\STATE $\mathcal{D}_h\leftarrow \{(s_{h}^{\tau},a_{h}^{\tau},V_{h+1}(s_{h+1}^{\tau}))\}_{\tau \in [K-1]}$
		\STATE $f_h\leftarrow \text{argmin}_{f\in \mathcal{F}}\|f\|^2_{\mathcal{D}_h}$
		\STATE 
		$b_h(\cdot,\cdot)\leftarrow \sup_{\|f_1-f_2\|^2_{\widehat{\mathcal{Z}}_h}\leq\beta}|(f_1(\cdot,\cdot)-f_2(\cdot,\cdot)|$
		\STATE $Q_h(\cdot,\cdot)\leftarrow \min\{f_h(\cdot,\cdot)+b_h(\cdot,\cdot)+r_h(\cdot,\cdot),H\}$ 
		\STATE $V_h(\cdot)\leftarrow \max_{a\in \mathcal{A}}Q_h(\cdot,a)$
		\STATE $\pi_h(\cdot)\leftarrow \text{argmax}_{a\in \mathcal{A}}Q_h(\cdot,a)$
		\ENDFOR
		\STATE \textbf{return} $\pi=\{\pi_h\}_{h=1}^H$
	\end{algorithmic}
\end{algorithm}
\paragraph{Exploration Phase.} Our algorithm for the exploration phase is quite similar to Algorithm~\ref{alg:main}\footnote{We need to slightly change the choice of $\beta$ in the reward-free setting. See Section~\ref{sec:beta} for details.}. The main difference is that without the guidance from the reward signal, we use the following \emph{exploration-driven} reward function to encourage exploration: 
$$
r_h^k(\cdot,\cdot)\leftarrow \min\{b_h^k(\cdot,\cdot)/H,1\}.
$$
The full algorithm used in the exploration phase is presented in Algorithm~\ref{alg:exploration}.

\paragraph{Planning Phase.} In the planning phase we do optimistic planning similar to Algorithm~\ref{alg:main}, but with real reward instead of exploration-driven reward. We still add the bonus function to guarantee optimism.
\subsubsection{Theoretical Guarantee} Now we state our theoretical guarantee in the reward-free RL setting.
\begin{thm}[formal]
\label{thm:main_free}
Suppose Assumption~\ref{assum:cover}, Assumption~\ref{assum:express_2}, and Assumption~\ref{assum:reward} holds and $T$ is sufficiently large.
For any given $\delta\in(0,1)$, after collecting $K$ trajectories during the exploration phase (by Algorithm~\ref{alg:exploration}), with probability at least $1-\delta$, for any reward function $r=\{r_h\}_{h=1}^H$ satisfying Assumption~\ref{assum:reward}, Algorithm~\ref{alg:planning} outputs an $O(H^3\cdot\sqrt{\iota_1/K})$-optimal policy for the MDP $(\cS,\cA, P, r, H, s_1)$. Here,
\begin{align*}
\iota_1=\log(\mathcal{N}(\mathcal{R},1/T))&\cdot\dim_E(\mathcal{F},1/T)\\
+\log(T\mathcal{N}(\mathcal{F},\delta/T^2)/\delta)\cdot\log^2 T
&\cdot \dim^2_E(\mathcal{F},1/T)
\cdot\log\left(\mathcal{N}(\mathcal{S}\times\mathcal{A},\delta/T^2) \cdot T/\delta\right).
\end{align*}
Furthermore, with probability $1-\delta$ the algorithm calls a $\Omega(K)$-sized regression oracle for at most $\widetilde{O}(d^2H^2)$ times.
\end{thm}
Using ideas in the proof of Theorem~\ref{thm:main_regret}, the proof of Theorem~\ref{thm:main_free} follows rather straightforwardly from~\citet{wang2020reward}. The high-level idea is to show that, after the exploration phase, for any reward function, the error of the planning policy is upper bounded by the expectation of the bonus functions, which is shown to be small enough using results proved in Theorem~\ref{thm:main_regret}.
The formal proof of Theorem~\ref{thm:main_free} is presented in Section~\ref{pf:thm3}.
\begin{remark}
Let $d(T,\delta):=\max(\log(\mathcal{N}(\mathcal{R},1/T))$, $\log(\mathcal{N}(\mathcal{F}$, $\delta/T^2))$, $\dim_E(\mathcal{F},1/T)$, $\log(\mathcal{N}(\mathcal{S}\times\mathcal{A}, \delta/T^2)))$. Then the output policy is guaranteed to be $\widetilde{O}(H^3\cdot d(T,\delta)^2/\sqrt{K})$-optimal with high probability. In the tabular case we have $d(T,\delta)=O(|\cS||\cA|\cdot\poly\log(|\cS||\cA|T\delta^{-1}))=\widetilde{O}(|\cS||\cA|)$. When $\cF$ and $\mathcal{R}$ are both the class of d-dimensional linear functions we have $d(T,\delta)=O(d\cdot\poly\log(dT\delta^{-1}))=\widetilde{O}(d)$. 
However, it is hard to rigorously show this kind of property when $\cF$ and $\mathcal{R}$ are both general function classes. Generally speaking, if $d(T,\delta)=O(d^*\cdot\poly\log(d^*T\delta^{-1}))=\widetilde{O}(d^*)$ where $d^*$ depends only on the complexity of $\cF$ and $\mathcal{R}$, then the output policy is guaranteed to be $\widetilde{O}(H^3\cdot (d^*)^2/\sqrt{K})$-optimal with high probability. Thus for any $\epsilon>0$, by taking $K=C\cdot(d^*)^4H^6\cdot\epsilon^{-2}\cdot \mathrm{polylog}(d^*T\delta^{-1}\epsilon^{-1})$ where $C>0$ is a sufficiently large constant, our algorithm guarantees to output an $\epsilon$-optimal policy after exploring the environment for $\widetilde{O}((d^*)^4H^6\epsilon^{-2})$ episodes. In this case, our sample complexity bound and switching cost bound become $\widetilde{O}((d^*)^4H^6\epsilon^{-2})$ and $\widetilde{O}((d^*)^2H)$. In particular, when $\cF$ and $\mathcal{R}$ are the class of d-dimensional linear functions, our sample complexity bound can be improved to $\widetilde{O}(d^3H^6\epsilon^{-2})$ with refined analysis using the technique  mentioned in Remark~\ref{rmk:2}, matching the bound given in~\citet{wang2020reward}.
\end{remark}
\subsection{Multi-Agent RL Setting}
\subsubsection{Two-Player Zero-Sum Game}
We consider a two-player zero-sum episodic Markov game MG($\cS,\cA,\cB,P,r,H,s_1$), where $\cS$ is the state space, $\cA$ and $\cB$ are the action space for the max-player and the min-player respectively, $P=\{P_h\}_{h=1}^H$ where $P_h:\cS\times\cA\times\cB\rightarrow\Delta(\cS)$ are the transition operators, $r=\{r_h\}_{h=1}^H$ where $r_h:\cS\times\cA\times\cB\rightarrow[0,1]$ are the deterministic reward functions, and $H$ is the planning horizon.

In this game the two player interacts with the environment episodically. Each episode starts from the fixed initial state $s_1$ and consists of $H$ time steps. A Markov policy of the max-player $\mu=\{\mu_h\}_{h=1}^H$, where for each $h\in[H]$, $\mu_h:\cS\rightarrow\Delta(\cA)$ maps a state to a distribution in the probability simplex over $\cA$. Similarly, we can define a Markov policy $\nu$ over $\cB$ for the min-player. In the rest of the section we only consider Markov policy. In each episode, the two players choose their policies $\mu$ and $\nu$, and induce a trajectory
$$
s_1,a_1,b_1,r_1,s_2,\ldots,s_H,a_H,b_H,r_H,s_{H+1}
$$
where $s_1$ is fixed, $a_1\sim\mu_1(s_1)$, $b_1\sim\nu_1(s_1)$ $r_1=r_1(s_1,a_1,b_1)$, $s_2 \sim P_1(\cdot|s_1,a_1,b_1)$, etc. Note that the whole trajectory (including the actions of the opponent) is visible to the two players. 

Similar to the single agent setting, we use the Q and V-function to evaluate the long-term expected cumulative reward. For two policies $\mu$, $\nu$, and time step $h\in[H]$, we define
\[
Q_h^{\mu,\nu}(s,a,b)=\EE\left[\sum_{h'=h}^H r_{h'}\mid s_h=s,a_h=a,b_h=b\right],
\]
and
\[
V_h^{\mu,\nu}(s)=\EE\left[\sum_{h'=h}^H r_{h'}\mid s_h=s\right].
\]
For a policy of the max-player $\mu$, there exists a best response policy of the min-player $\nu^{\dag}(\mu)$ satisfying $V_h^{\mu,\nu^{\dag}(\mu)}(s)=\inf_{\nu}V_h^{\mu,\nu}(s)$ for all $(s,h)$. We further denote $V_h^{\mu,\dag}(s):=V_h^{\mu,\nu^{\dag}(\mu)}(s)$. Similarly we can define $\mu^{\dag}(\nu)$ and $V_h^{\dag,\nu}$. 

We denote $(\mu^*,\nu^*)$ as the Nash equilibrium of the Markov game, which satisfies the following minimax equation
\[
V_h^{\mu^*,\nu^{*}}(s)=V_h^{\dag,\nu^{*}}(s)=V_h^{\mu^*,\dag}(s)=:V_h^*(s).
\]
We can see from the minimax equation that in a Nash equilibrium, no player can benefit by \emph{unilaterally} deviating from her own policy.
\paragraph{Learning Objective.} We consider the online setting where the learner can only control the max-player, and the min-player is controlled by a potentially adversarial opponent. In the beginning of the $k$-th episode, the learner chooses her policy $\mu^k$, and the opponent chooses a policy $\nu^k$ based on $\mu^k$. The learning goal is to minimize the following regret
\[
\text{NashRegret}(K):=\sum_{k=1}^K\left[V_1^{*}(s_1)-V_1^{\mu^k,\nu^k}(s_1)\right].
\]
\subsubsection{Function Class}
Similar to the single-agent setting, we assume that a function class $\mathcal{F}\subset \{f:\mathcal{S}\times\mathcal{A}\times\cB\rightarrow[0,H+1]\}$ is given as a priori. 
For any $f:\cS\times\cA\times\cB\rightarrow\mathbb{R}$, we denote the Nash policy of the max-player induced by $f$ as
\[
\mu_f(s):=\argmax_{\mu\in\Delta(\cA)}\min_{\nu\in\Delta(\cB)}\mu^{\top}f(s,\cdot,\cdot)\nu,
\]
and the Nash value function induced by $f$ as
\[
V_f(s):=\max_{\mu\in\Delta(\cA)}\min_{\nu\in\Delta(\cB)}\mu^{\top}f(s,\cdot,\cdot)\nu.
\]
Furthermore, for any $f:\cS\times\cA\times\cB\rightarrow\mathbb{R}$, we define the \emph{Bellman projection} of $f$ as {$\cT_h f:\cS\times\cA\times\cB\rightarrow\mathbb{R}$, such that for all $(s,a,b)\in\cS\times\cA\times\cB$,} \[\textstyle{
\cT_h f(s,a,b):=r_h(s,a,b)+\mathbb{E}_{s'\sim\mathbb{P}_h(\cdot|s,a,b)}V_f(s').}
\] 


We still need to assume bounded covering numbers of the function class. Similar to Assumption~\ref{assum:cover}, we can define $\cC(\cF,\epsilon),\cN(\cF,\epsilon),\cC(\cS\times\cA\times\cB,\epsilon),\cN(\cS\times\cA\times\cB,\epsilon)$.

We assume the following multi-agent version of the realizability and completeness assumptions used in \citet{jin2022power,huang2021towards}.
\begin{assum}[Completeness]
\label{assum:express_3_mg}
For all $h\in[H]$, $\mathcal{T}_h\cF_{[0,H-h]}\subseteq\cF$, where $\cF_{[0,H-h]}:=\{f\in\cF:\|f\|_{\infty}\leq H-h\}$ is the truncated function class in step $h$.
\end{assum}
\begin{assum}[Realizability]
\label{assum:real_mg}
For all $h\in[H]$, $Q_h^*\in\cF$.
\end{assum}
To measure the complexity of $\cF$, we introduce the \emph{Online Bellman Eluder Dimension}~\citep{jin2022power}.
\begin{defn}[Distributional Eluder Dimension]
Let $\varepsilon\ge 0$ and $\cZ =\{\mu_1,\mu_2,...,\mu_n\}_{i=1}^n\subseteq \Delta(\cS\times\cA\times\cB)$ be a sequence of probability measures over state-action space. \\
(1) A probability measure $\mu\in \Delta(\cS\times \cA\times\cB)$ is \emph{$\varepsilon$-dependent} on $\cZ$ with respect to $\cG$ if any $g\in \cG$ satisfying $\sqrt{\sum_{\mu\in\cZ}\mathbb{E}_{\mu}g(s,a,b)^2} \le \varepsilon$ also satisfies $\mathbb{E}_{\mu}|g(s,a,b)|\le \varepsilon$.  \\
(2) An $\mu$ is \emph{$\varepsilon$-independent} of $\cZ$ with respect to $\cG$ if $\mu$ is not $\varepsilon$-dependent on $\cZ$. \\
(3) The \emph{$\varepsilon$-distributional eluder dimension} $\ds_{DE}(\cG,\Pi,\varepsilon)$ of a function class $\cG$ is the length of the longest sequence of elements which take values in $\Pi$ such that, for some $\varepsilon' \ge \varepsilon$, every element is $\varepsilon'$-independent of its predecessors.
\end{defn}
Here we only consider the single point distribution: $\cD_{\Delta}:=\{\delta_{(s,a,b)}(\cdot)|s\in\cS,a\in\cA,b\in\cB\}$.
Then the online Bellman eluder dimension is stated as:
\begin{defn}[Online Bellman Eluder Dimension] Let $(I-\cT_h)\cF:=\{f-\cT_h f':f,f'\in\cF\}$ be the set of bellman residuals induced by $\cF$ at step $h$. The $\varepsilon$-online Bellman eluder dimension of $\cF$ is defined as \[\dim_{OBE}(\cF,\varepsilon):=\max_{h\in[H]}\dim_{DE}((I-\cT_h)\cF,\cD_{\Delta},\varepsilon).\]
\end{defn}
\subsubsection{Algorithm}
Here we introduce the algorithm. The algorithm is based on Algorithm 1 in jin, which can be seen as a multi-agent version of the GOLF algorithm~\citep{jin2021bellman}. We use the online sub-sampling procedure to carefully control the policy update of the max-player. Here the online sub-sampling algorithm is exactly the one in Algorithm~\ref{alg:sample}, but with triple data points $(s,a,b)\in\cS\times\cA\times\cB$. The full algorithm is presented in Algorithm~\ref{alg:MG}. 
Here we choose $\beta$ to be
\[
\beta=C\cdot H^2\cdot \log(T\cN(\cF,1/K)/\delta).
\]
\begin{algorithm}[h]
	\caption{GOLF with Online Sub-Sampling\label{alg:MG}}
	\begin{algorithmic}
	\STATE \textbf{Input:} Failure probability $\delta \in (0,1)$, number of episodes $K$, and function class $\cF$.
	\STATE $\tilde{k}\leftarrow 1$
	\STATE $\widehat{\mathcal{Z}}_h^1\leftarrow\{\}\quad \forall h\in[H]$
	\FOR{$k=1,2,\ldots,K$}
	    \FOR {$h=H,H-1,...,1$}
	    \STATE 
	    $\widehat{\mathcal{Z}}_h^k\leftarrow \textbf{Online-Sample}(\mathcal{F},\widehat{\mathcal{Z}}_h^{k-1}, (s_{h}^{k-1},a_{h}^{k-1},b_{h}^{k-1}),\delta)(\text{if } k\geq 2 )$
		\ENDFOR
		\IF {$k=1$ \OR $\exists  h\in[H] ~~ \widehat{\mathcal{Z}}^k_h\neq\widehat{\mathcal{Z}}^{{k-1}}_h$}
	
		\STATE $\tilde{k}\leftarrow k$ {\color{blue} // $\tilde{k}$ : index of the latest policy}
		\STATE Define $\mathcal{D}_h^k(f)\leftarrow \{(s_{h}^{\tau},a_{h}^{\tau},b_h^\tau,r_h^\tau+V_f(s_{h+1}^{\tau}))\}_{\tau \in [k-1]}$ for all $h\in[H]$
		\STATE Set \[\mathcal{B}^k\leftarrow\{\{f_h\}_{h=1}^{H+1}:\forall h\in[H],\|f_{h}\|_{\infty}\leq H+1-h,\|f_h\|^2_{\mathcal{D}_h^k(f_{h+1})}\leq\text{inf}_{g\in \mathcal{F}}\|g\|^2_{\mathcal{D}_h^k(f_{h+1})}+\beta,f_{H+1}\equiv 0\}\]
		\STATE Set $$(Q_1^k,Q_2^k,...,Q_{H+1}^k)\leftarrow \text{argmax}_{Q^k\in\mathcal{B}^k}(V_{Q_1^k}(s_1))$$
		\STATE For all $h\in[H]$, set $\mu_h^k\leftarrow \mu_{Q_h^k}$
		\ENDIF
		\STATE Let the min-player choose a policy $\nu^k$
		\STATE Execute policy $(\mu^{\tilde{k}},\nu^k)$ to induce a trajectory $s_1^k,a_1^k,b_1^k,r_1^k,...,s^k_H,a^k_H,b_H^k,r^k_H,s^k_{H+1}$	
	\ENDFOR
	\end{algorithmic}
\end{algorithm}
\subsubsection{Theoretical Guanrantee}

\begin{thm}[formal]
\label{thm:MG}
Assume Assumption~\ref{assum:cover}, Assumption~\ref{assum:express_3_mg}, and Assumption~\ref{assum:real_mg} holds and $T$ is sufficiently large. With probability at least $1-4\delta$, Algorithm~\ref{alg:MG} achieves a regret bound
\[
\operatorname{NashRegret}(K)\leq O\left(\sqrt{\iota_1\cdot H^3\cdot T}\right),
\]
where 
\[
\iota_1=\log(T\cN(\cF,1/K)/\delta)\cdot\dim_{OBE}(\cF,1/\sqrt{K}).
\]
Furthermore, with probability at least $1-\delta$, the max-player switches her policy for at most $\tilde{O}(d^2H)$ times, and the algorithm calls the nested optimization oracle for at most $\tilde{O}(d^2H)$ times.
\end{thm}

\section{Omitted Proofs}
\subsection{Proof of Theorem~\ref{thm:sample}}
\label{pf:sample}
In this section we first state Theorem~\ref{thm:sample} formally, then give  a complete proof of it.

\textbf{Theorem 1} (formal)\textbf{.} With probability $1-\delta$, for all $(k,h)\in[K]\times[H]$, the number of distinct elements in $\widehat{\mathcal{Z}}_h^k$ is bounded by
\[
n_d(\widehat{\mathcal{Z}}_h^k)\leq O(\log(T\mathcal{N}(\mathcal{F},\sqrt{\delta/64T^3})/\delta) \cdot\dim_E(\mathcal{F},1/T)\cdot\log^2 T),
\]
and
\[
\frac{1}{10000}\|f_1-f_2\|^2_{\mathcal{Z}^k_h}\leq \min\{\|f_1-f_2\|^2_{\widehat{\mathcal{Z}}^k_h},T(H+1)^2\} \leq 10000\|f_1-f_2\|^2_{\mathcal{Z}^k_h},\quad \forall \|f_1-f_2\|^2_{\mathcal{Z}^k_h}> 100\beta,
\]
and
\[
\min\{\|f_1-f_2\|^2_{\widehat{\mathcal{Z}}^k_h},T(H+1)^2\} \leq 10000\beta,\quad \forall \|f_1-f_2\|^2_{\mathcal{Z}^k_h}\leq  100\beta.
\]
\subsubsection{Analysis and Propositions}
\label{sec:props}
We explain the intuition behind Theorem~\ref{thm:sample}. As mentioned before, we need to show that the sub-sampled dataset $\widehat{\mathcal{Z}}_h^k$
\begin{itemize}
    \item provides a good approximation to $\mathcal{Z}_h^k$; and 
    \item has a much lower complexity than $\mathcal{Z}_h^k$ (in terms of number of distinct elements).
\end{itemize}
We define the following enlarged and shrunk confidence sets.
For all $(k,h)\in[K]\times[H]$ and $\alpha\in[\beta,+\infty)$, define
\[
\underline{\mathcal{B}}_h^k(\alpha):=\{(f_1,f_2)\in\mathcal{F}\times\mathcal{F}| \|f_1-f_2\|^2_{\mathcal{Z}^k_h}\leq\alpha/100\},
\]
\[
\mathcal{B}_h^k(\alpha):=\{(f_1,f_2)\in\mathcal{F}\times\mathcal{F}| \min\{\|f_1-f_2\|^2_{\widehat{\mathcal{Z}}^k_h},T(H+1)^2\}\leq\alpha\},
\]
\[
\overline{\mathcal{B}}_h^k(\alpha):=\{(f_1,f_2)\in\mathcal{F}\times\mathcal{F}| \|f_1-f_2\|^2_{\mathcal{Z}^k_h}\leq100\alpha\}.
\]
For each $(k,h)\in[K]\times[H]$, we use $\mathcal{E}_h^k(\alpha)$ to denote the event that 
$$
\underline{\mathcal{B}}_h^k(\alpha)\subseteq\mathcal{B}_h^k(\alpha)\subseteq\overline{\mathcal{B}}_h^k(\alpha).
$$
Furthermore, we denote that
$$
\mathcal{E}_h^k:=\bigcap_{n=0}^\infty\mathcal{E}_h^k(100^n\beta).
$$
Event $\mathcal{E}_h^k$ characterizes the meaning of ``good approximation''.
In fact, if $\mathcal{E}_h^k$ happens, we can show that $\|f_1 - f_2\|_{\widehat{\mathcal{Z}}_h^k}$ is close to $\|f_1 - f_2\|_{\mathcal{Z}_h^k}$ up to a constant factor, thus the confidence set induced by $\widehat{\mathcal{Z}}_h^k$ is accurate.
The following proposition verifies that $\mathcal{E}_h^k$ happens with high probability.

\begin{prop}
\label{prop:bounds_of_bonus}
\[
\Pr\left(\bigcap_{h=1}^H\bigcap_{k=1}^K\mathcal{E}_h^k \right)\geq 1-\delta/32.
\]
\end{prop}
As will be shown shortly, Proposition~\ref{prop:bounds_of_bonus} directly implies the approximation part of Theorem~\ref{thm:sample}.\\
Moreover, we define the following bonus functions calculated by $\mathcal{Z}_h^k$ instead of $\widehat{\mathcal{Z}}_h^k$:
\[
\underline{b}_h^k(\cdot,\cdot):= \sup_{\|f_1-f_2\|^2_{\mathcal{Z}^k_h}\leq\beta/100}|(f_1(\cdot,\cdot)-f_2(\cdot,\cdot)|,
\]
\[
\overline{b}_h^k(\cdot,\cdot):= \sup_{\|f_1-f_2\|^2_{\mathcal{Z}^k_h}\leq100\beta}|(f_1(\cdot,\cdot)-f_2(\cdot,\cdot)|.
\]
If $\cE_h^k$ happens, by taking $\alpha=\beta$, we have that
\[
\underline{b}_h^k(\cdot,\cdot)\leq b_h^k(\cdot,\cdot)\leq \overline{b}_h^k(\cdot,\cdot)
\]
which verifies the correctness of our bonus function $b_h^k$ used in the algorithm.

Proposition~\ref{prop:bounded_size} bounds the size of $\widehat{\mathcal{Z}}_h^k$.
\begin{prop}
\label{prop:bounded_size}
With probability at least $1-\delta/8$, the following statements hold:

1. For any fixed $h\in[H]$, the subsampled dataset $\widehat{\mathcal{Z}}_h^k$ ($k=1,2,...,K$) changes for at most 
\[
S_{\max}=C\cdot\log(T\mathcal{N}(\mathcal{F},\sqrt{\delta/64T^3})/\delta) \cdot\dim_E(\mathcal{F},1/T)\cdot\log^2 T
\]
times for some absolute constant $C>0$. 

As a result, for any pair $(k,h)\in[K]\times[H]$, $n_d(\widehat{\mathcal{Z}}_h^k)\leq S_{\max}$. 

2. For any $(h,k)\in[H]\times[K]$, $$
|\widehat{\mathcal{Z}}_h^k|\leq 64T^3/\delta.
$$
\end{prop}
In the following two sections we prove Proposition~\ref{prop:bounds_of_bonus} and Proposition~\ref{prop:bounded_size}. Throughout the proof, We use $\mathscr{F}_k$ to denote the filtration induced by the history up to episode $k$ (include episode $k$) and use $\mathbb{E}_k$ to denote the expectation conditioned on $\mathscr{F}_k$.
\subsubsection{Proof of Proposition~\ref{prop:bounds_of_bonus}}
For completeness, we state a Bernstein-type martingale 
concentration inequality which will be frequently used in our proofs.
\begin{lem}[\citep{freedman1975tail}]
\label{lem:freedman}
Consider a real-valued martingale $\{Y_k:k=0,1,2...\}$ with difference sequence $\{X_k:k=1,2,...\}$. Assume that the difference sequence is uniformly bounded:
\[
|X_k|\leq R \text{ almost surely for } k=1,2,3,...  \text{.}
\]
For a fixed $n\in\mathbb{N}$, assume that
$$
\sum_{k=1}^n\mathbb{E}_{k-1}(X_k^2)\leq\sigma^2 $$
almost surely.
Then for all $t\geq 0$,
\[
P\{|Y_n-Y_0|\geq t\}\leq 2\exp\left\lbrace -\frac{t^2/2}{\sigma^2+Rt/3}\right\rbrace.
\]
\end{lem}
The next lemma upper bounds the size of $\widehat{\mathcal{Z}}_h^k$.
\begin{lem} 
\label{lem:markov}
With probability at least $1-\delta/64T$,
\[
|\widehat{\mathcal{Z}}_h^k|\leq 64T^3/\delta \quad \forall (k,h)\in[K]\times[H].
\]
\end{lem}
\begin{proof}
Consider a fixed pair $(k,h)\in[K]\times[H]$. By Markov's inequality we have that 
$$
|\widehat{\mathcal{Z}}_h^k|\leq 64T^2|\mathcal{Z}_h^k|/\delta
$$
holds with probability at least $1-\delta/(64T^2)$.
With a union bound for all $(k,h)\in[K]\times[H]$ we complete the proof.
\end{proof}
Now we start to analyze the events defined in Section~\ref{sec:props}. Recall that in Section~\ref{sec:props} we mentioned that event $\mathcal{E}_h^k$ characterizes the meaning of ``good approximation''. Our next lemma formalizes this intuition.
\begin{lem}
\label{lem:app}
If $\cE_h^k$ happens, then
\[
\frac{1}{10000}\|f_1-f_2\|^2_{\mathcal{Z}^k_h}\leq \min\{\|f_1-f_2\|^2_{\widehat{\mathcal{Z}}^k_h},T(H+1)^2\} \leq 10000\|f_1-f_2\|^2_{\mathcal{Z}^k_h},\quad \forall \|f_1-f_2\|^2_{\mathcal{Z}^k_h}> 100\beta
\]
and
\[
\min\{\|f_1-f_2\|^2_{\widehat{\mathcal{Z}}^k_h},T(H+1)^2\} \leq 10000\beta,\quad \forall \|f_1-f_2\|^2_{\mathcal{Z}^k_h}\leq  100\beta.
\]
\end{lem}
\begin{proof}
If $\|f_1-f_2\|^2_{\mathcal{Z}^k_h}\leq  100\beta$, we have $(f_1,f_2)\in\underline{\mathcal{B}}_h^k(10000\beta)$. From $\cE_h^k$ we know $(f_1,f_2)\in\mathcal{B}_h^k(10000\beta)$, which implies the desired result.

If $\|f_1-f_2\|^2_{\mathcal{Z}^k_h}>  100\beta$, assume that $100^n\beta<\|f_1-f_2\|^2_{\mathcal{Z}^k_h}\leq 100^{n+1}\beta$, $n\in\mathbb{N}^*$. Then we have $(f_1,f_2)\notin\overline{\mathcal{B}}_h^k(100^{n-1}\beta)$ and also $(f_1,f_2)\notin\mathcal{B}_h^k(100^{n-1}\beta)$. This implies that  $ \min\{\|f_1-f_2\|^2_{\widehat{\mathcal{Z}}^k_h},T(H+1)^2\}\geq 100^{n-1}\beta\geq \frac{1}{10000}\|f_1-f_2\|^2_{\mathcal{Z}^k_h}$. Similarly, we have $(f_1,f_2)\in\underline{\mathcal{B}}_h^k(100^{n+2}\beta)$, then also $(f_1,f_2)\in\mathcal{B}_h^k(100^{n+2}\beta)$. Thus we have $ \min\{\|f_1-f_2\|^2_{\widehat{\mathcal{Z}}^k_h},T(H+1)^2\}\leq 100^{n+2}\beta\leq 10000\|f_1-f_2\|^2_{\mathcal{Z}^k_h}$.
\end{proof}
Recall that Proposition~\ref{prop:bounds_of_bonus} states that $\bigcap_{h=1}^H\bigcap_{k=1}^K\mathcal{E}_h^k$ happens with high probability. As will be shown in the proof of Proposition~\ref{prop:bounds_of_bonus} later, to bound the probability of $\bigcap_{h=1}^H\bigcap_{k=1}^K\mathcal{E}_h^k$ we only need to bound $\Pr\left( \mathcal{E}_h^1\mathcal{E}_h^2...\mathcal{E}_h^{k-1}(\mathcal{E}_h^k(100^n\beta))^c\right)$. Note that $\mathcal{E}_h^k(100^n\beta)$ always holds if $100^n\beta\geq T(H+1)^2$. We establish the following lemma.
\begin{lem} 
\label{lem:event_main}
For
$\alpha\in[\beta,T(H+1)^2]$,
$$
\Pr\left( \mathcal{E}_h^1\mathcal{E}_h^2...\mathcal{E}_h^{k-1}(\mathcal{E}_h^k(\alpha))^c\right)\leq \delta/(32T^2).
$$
\end{lem}
\begin{proof} We use $\overline{\mathcal{Z}}_h^k$ to denote the dataset without rounding, i.e., we replace every element $\hat{z}$ with $z$ in $\widehat{\mathcal{Z}}_h^k$. Denote $C_1:=C\cdot\log(T\mathcal{N}(\mathcal{F},\sqrt{\delta/64T^3})/\delta)$ to be the parameter used in Algorithm~\ref{alg:sample}.

 Consider a fixed pair $(f_1,f_2)\in\mathcal{C}(\mathcal{F}, \sqrt{\delta/(64T^3)})\times\mathcal{C}(\mathcal{F},\sqrt{\delta/(64T^3)})$.
 
For each $i\geq 2$, define
\[
Z_i=\max\{\|f_1-f_2\|^2_{\mathcal{Z}_h^{i}},\min\{\|f_1-f_2\|^2_{\widehat{\mathcal{Z}}_h^{i-1}},T(H+1)^2\}\}
\]
and
\[
Y_i=
\begin{cases}
\frac{1}{p_{z_h^{i-1}}}(f_1(z_h^{i-1})-f_2(z_h^{i-1}))^2 ~~~~~~~ &z_h^{i-1}\text{ is added into } \overline{\mathcal{Z}}_h^i\text{ and }Z_i\leq 2000000\alpha\\
0  &z_h^{i-1}\text{ is not added into } \overline{\mathcal{Z}}_h^i\text{ and }Z_i\leq 2000000\alpha\\
(f_1(z_h^{i-1})-f_2(z_h^{i-1}))^2  &Z_i> 2000000\alpha
\end{cases}.
\]
$Y_i$'s are used to characterize the sampling procedure. Note that  $Y_i$ is adapted to the filtration $\mathscr{F}_i$, and 
$ 
\mathbb{E}_{i-1}[Y_i]=(f_1(z_h^{i-1})-f_2(z_h^{i-1}))^2
$. In order to use Freedman's inequality, we need to bound $Y_i$ and its variance.

If $p_{z_h^{i-1}}=1$ or $\min\{\|f_1-f_2\|^2_{\widehat{\mathcal{Z}}_h^{i-1}},T(H+1)^2\}> 2000000\alpha$, then $Y_i-\mathbb{E}_{i-1}[Y_i]=\text{Var}_{i-1}[Y_i-\mathbb{E}_{i-1}[Y_i]]=0$. Otherwise from the definition of $p_z$ in Algorithm~\ref{alg:sample} we have that:
\begin{align*}
|Y_i-\mathbb{E}_{i-1}[Y_i]|&\leq (\min\{\|f_1-f_2\|^2_{\widehat{\mathcal{Z}}_h^{i-1}},T(H+1)^2\}+\beta)\cdot1/C_1\\
&\leq 3000000\alpha/C_1
\end{align*}
and
\begin{align*}
\text{Var}_{i-1}[Y_i-\mathbb{E}_{i-1}[Y_i]]&\leq \frac{1}{p_{z_h^{i-1}}}(f_1(z_h^{i-1})-f_2(z_h^{i-1}))^4\\
&\leq  (f_1(z_h^{i-1})-f_2(z_h^{i-1}))^2\cdot 3000000\alpha/C_1.
\end{align*}
It is easy to verify that
\begin{align*}
\sum_{i=2}^k\text{Var}_{i-1}[Y_i-\mathbb{E}_{i-1}[Y_i]]
\leq  (3000000\alpha)^2/C_1.
\end{align*}
By Freedman's inequality (Lemma~\ref{lem:freedman}), we have that 
\begin{align*}
&\Pr\left\lbrace \left| \sum_{i=2}^k \left( Y_i-\mathbb{E}_{i-1}[Y_i]\right)\right| \geq \alpha/100 \right\rbrace \\
&\leq 2\exp\left\lbrace -\frac{(\alpha/100)^2/2}{(3000000\alpha)^2/C_1+\alpha\cdot 3000000\alpha/3C_1}\right\rbrace\\
&\leq (\delta/64T^2)/(\mathcal{N}(\mathcal{F},\sqrt{\delta/(64T^3)}))^2 .
\end{align*}
With a union bound, the above inequality implies that with probability at least $1-\delta/(64T^2)$, for any pair $(f_1,f_2)\in\mathcal{C}(\mathcal{F},\sqrt{\delta/(64T^3)})\times\mathcal{C}(\mathcal{F},\sqrt{\delta/(64T^3)})$, the corresponding $Y_i$'s satisfy
\[
\left| \sum_{i=2}^k \left( Y_i-\mathbb{E}_{i-1}[Y_i]\right)\right| \leq  \alpha/100.
\]

Now we condition on the above event and the event defined in Lemma~\ref{lem:markov} for the rest of the proof.

\paragraph{Part 1: ($\underline{\mathcal{B}}_h^k(\alpha)\subseteq\mathcal{B}_h^k(\alpha)$)} Consider any pair $f_1,f_2\in\mathcal{F}$ with $\|f_1-f_2\|^2_{\mathcal{Z}_h^k}\leq \alpha/100$. From the definition we know that there exist $(\hat{f_1},\hat{f_2})\in\mathcal{C}(\mathcal{F},\sqrt{\delta/(64T^3)})\times\mathcal{C}(\mathcal{F},\sqrt{\delta/(64T^3)})$ such that $\|\hat{f_1}-f_1\|_{\infty},\|\hat{f_2}-f_2\|_{\infty}\leq \sqrt{\delta/(64T^3)}$. Then we have that
\begin{align*}
\|\hat{f_1}-\hat{f_2}\|^2_{\mathcal{Z}_h^k}&\leq (\|f_1-f_2\|_{\mathcal{Z}_h^k}+\|f_1-\hat{f_1}\|_{\mathcal{Z}_h^k}+\|\hat{f_2}-f_2\|_{\mathcal{Z}_h^k})^2\\
&\leq (\|f_1-f_2\|_{\mathcal{Z}_h^k}+2\cdot\sqrt{|\mathcal{Z}_h^k|}\cdot\sqrt{\delta/(64T^3)})^2\\
&\leq \alpha/50.
\end{align*}
We consider the $Y_i$'s which corrspond to $\hat{f_1}$ and $\hat{f_2}$. 
Because $\|\hat{f_1}-\hat{f_2}\|^2_{\mathcal{Z}_h^k}\leq \alpha/50$, we also have that $\|\hat{f_1}-\hat{f_2}\|^2_{\mathcal{Z}_h^{k-1}}\leq \alpha/50$. From $\mathcal{E}_h^{k-1}$ we know that $\min\{\|\hat{f_1}-\hat{f_2}\|^2_{\widehat{\mathcal{Z}}_h^{k-1}},T(H+1)^2\}\leq 10000\alpha$. Then from the definition of $Y_i$ we have
\[
\|\hat{f_1}-\hat{f_2}\|_{\overline{\mathcal{Z}}^{k}_h}^2=\sum_{i=2}^kY_i.
\]
Then $\|\hat{f_1}-\hat{f_2}\|_{\overline{\mathcal{Z}}^{k}_h}^2$ can be bounded in the following manner: 
\begin{align*}
\|\hat{f_1}-\hat{f_2}\|_{\overline{\mathcal{Z}}^{k}_h}^2&=\sum_{i=2}^kY_i\\
&\leq  \sum_{i=2}^k\mathbb{E}_{i-1}[Y_i]+\alpha/100\\
&= \|\hat{f_1}-\hat{f_2}\|^2_{\mathcal{Z}_h^k}+\alpha/100\\
&\leq 3\alpha/100.
\end{align*}
As a result, $\|f_1-f_2\|_{\overline{\mathcal{Z}}^{k}_h}^2$ can also be bounded:
\begin{align*}
\|f_1-f_2\|^2_{\overline{\mathcal{Z}}_h^k}&\leq (\|\hat{f_1}-\hat{f_2}\|_{\overline{\mathcal{Z}}_h^k}+\|f_1-\hat{f_1}\|_{\overline{\mathcal{Z}}_h^k}+\|\hat{f_2}-f_2\|_{\overline{\mathcal{Z}}_h^k})^2\\
&\leq (\|\hat{f_1}-\hat{f_2}\|_{\overline{\mathcal{Z}}_h^k}+2\cdot\sqrt{|\overline{\mathcal{Z}}_h^k|}\cdot\sqrt{\delta/(64T^3)})^2\\
&\leq \alpha/25.
\end{align*}
Finally we could bound $\|f_1-f_2\|^2_{\widehat{\mathcal{Z}}_h^k}$:
\begin{align*}
\|f_1-f_2\|^2_{\widehat{\mathcal{Z}}_h^k}&\leq(\|f_1-f_2\|_{\overline{\mathcal{Z}}_h^k}+\sqrt{64T^3/\delta}/(8\sqrt{64T^3/\delta}))^2\\
&\leq \alpha.
\end{align*}
We conclude that for any pair $f_1,f_2 \in \mathcal{F}$ with $\|f_1-f_2\|^2_{\mathcal{Z}_h^k}\leq \alpha/100$, it holds that $\|f_1-f_2\|^2_{\widehat{\mathcal{Z}}_h^k}\leq \alpha$. Thus we must have $\underline{\mathcal{B}}_h^k(\alpha)\subseteq\mathcal{B}_h^k(\alpha)$.

\paragraph{Part 2: ($\mathcal{B}_h^k(\alpha)\subseteq\overline{\mathcal{B}}_h^k(\alpha)$)} Consider any pair $f_1,f_2\in\mathcal{F}$ with $\|f_1-f_2\|^2_{\mathcal{Z}_h^k}> 100 \alpha$. From the definition we know that there exist $(\hat{f_1},\hat{f_2})\in\mathcal{C}(\mathcal{F},\sqrt{\delta/(64T^3)})\times\mathcal{C}(\mathcal{F},\sqrt{\delta/(64T^3)})$ such that $\|\hat{f_1}-f_1\|_{\infty},\|\hat{f_2}-f_2\|_{\infty}\leq \sqrt{\delta/(64T^3)}$. Then we have that
\begin{align*}
\|\hat{f_1}-\hat{f_2}\|^2_{\mathcal{Z}_h^k}&\geq (\|f_1-f_2\|_{\mathcal{Z}_h^k}-\|f_1-\hat{f_1}\|_{\mathcal{Z}_h^k}-\|\hat{f_2}-f_2\|_{\mathcal{Z}_h^k})^2\\
&\geq  (\|f_1-f_2\|_{\mathcal{Z}_h^k}-2\cdot\sqrt{|\mathcal{Z}_h^k|}\cdot\sqrt{\delta/(64T^3)})^2\\
&> 50\alpha.
\end{align*}
 Thus we have  $\|\hat{f_1}-\hat{f_2}\|^2_{\mathcal{Z}_h^k}> 50\alpha$.  We consider the $Y_i$'s which corrspond to $\hat{f_1}$ and $\hat{f_2}$. 
 Here we want  to prove that $\|\hat{f_1}-\hat{f_2}\|^2_{\widehat{\mathcal{Z}}_h^k}> 40\alpha$. For the sake of contradicition we assume that $\|\hat{f_1}-\hat{f_2}\|^2_{\widehat{\mathcal{Z}}_h^k}\leq 40\alpha$. 
 \paragraph{Case 1: $\|\hat{f_1}-\hat{f_2}\|^2_{\mathcal{Z}_h^k}\leq 2000000\alpha$.} From the definition of $Y_i$ we have that
\[
\|\hat{f_1}-\hat{f_2}\|_{\overline{\mathcal{Z}}^{k}_h}^2=\sum_{i=2}^kY_i.
\]
 Combined with the former result, we conclude that 
\[
\|\hat{f_1}-\hat{f_2}\|_{\overline{\mathcal{Z}}^{k}_h}^2=\sum_{i=2}^kY_i\geq \sum_{i=2}^k\mathbb{E}_{i-1}[Y_i]-\alpha/100=\|\hat{f_1}-\hat{f_2}\|^2_{\mathcal{Z}_h^k}-\alpha/100>50\alpha-\alpha/100=49\alpha.
\]
Then we have
\begin{align*}
\|\hat{f}_1-\hat{f}_2\|^2_{\widehat{\mathcal{Z}}_h^k}&\geq(\|\hat{f}_1-\hat{f}_2\|_{\overline{\mathcal{Z}}_h^k}-\sqrt{64T^3/\delta}/(8\sqrt{64T^3/\delta}))^2\\
&> 40\alpha.
\end{align*}
This leads to a contradiction.

\paragraph{Case 2: $\|\hat{f_1}-\hat{f_2}\|^2_{\mathcal{Z}_h^{k-1}}> 1000000\alpha$.} From $\mathcal{E}_h^{k-1}$ we deduce that $\|\hat{f_1}-\hat{f_2}\|^2_{\widehat{\mathcal{Z}}_h^{k-1}}> 100\alpha$ which directly leads to a contradiction.

\paragraph{Case 3: $\|\hat{f_1}-\hat{f_2}\|^2_{\mathcal{Z}_h^k}> 2000000\alpha$ and $\|\hat{f_1}-\hat{f_2}\|^2_{\mathcal{Z}_h^{k-1}}\leq  1000000\alpha$.} It is clear that $(\hat{f}_1(z_h^{k-1})-\hat{f}_2(z_h^{k-1}))^2>1000000\alpha$. From the definition of sensitivity we know that $z_h^{k-1}$ will be added into $\overline{\mathcal{Z}}_h^k$ almost surely. This clearly leads to a contradiction.

We conclude that $\|\hat{f_1}-\hat{f_2}\|^2_{\widehat{\mathcal{Z}}_h^k}> 40\alpha$.

Finally we could bound $\|f_1-f_2\|^2_{\widehat{\mathcal{Z}}_h^k}$:
\begin{align*}
\|f_1-f_2\|^2_{\widehat{\mathcal{Z}}_h^k}&\geq(\|\hat{f_1}-\hat{f_2}\|_{\widehat{\mathcal{Z}}_h^k}-\|f_1-\hat{f_1}\|_{\widehat{\mathcal{Z}}_h^k}-\|\hat{f_2}-f_2\|_{\widehat{\mathcal{Z}}_h^k})^2\\
&\geq (\|\hat{f_1}-\hat{f_2}\|_{\widehat{\mathcal{Z}}_h^k}-2\cdot\sqrt{|\widehat{\mathcal{Z}}_h^k|}\cdot\sqrt{\delta/(64T^3)})^2\\
&> \alpha.
\end{align*}

We conclude that for any pair $f_1,f_2 \in \mathcal{F}$ with $\|f_1-f_2\|^2_{\mathcal{Z}_h^k}>10000 \beta$, it holds that $\|f_1-f_2\|^2_{\widehat{\mathcal{Z}}_h^k}> 100\beta$. This implies that $\mathcal{B}_h^k(\alpha)\subseteq\overline{\mathcal{B}}_h^k(\alpha)$.
\end{proof}
\begin{proof}[Proof of Proposition~\ref{prop:bounds_of_bonus}] 
Note that for all $(k,h)\in[K]\times[H]$, we have
\begin{align}
\label{equ:event_decompose}
\notag&\Pr(\mathcal{E}_h^1\mathcal{E}_h^2...\mathcal{E}_h^{k-1})-\Pr(\mathcal{E}_h^1\mathcal{E}_h^2...\mathcal{E}_h^k)\\
\notag=&\Pr(\mathcal{E}_h^1\mathcal{E}_h^2...\mathcal{E}_h^{k-1}(\mathcal{E}_h^k)^c)\\
\notag= &\Pr\left( \mathcal{E}_h^1\mathcal{E}_h^2...\mathcal{E}_h^{k-1}\left(\bigcap_{n=0}^\infty\mathcal{E}_h^k(100^n\beta)\right)^c\right) \\
\notag= & \Pr\left( \mathcal{E}_h^1\mathcal{E}_h^2...\mathcal{E}_h^{k-1}\bigcup_{n=0}^\infty(\mathcal{E}_h^k(100^n\beta))^c\right) \\
\notag\leq & \sum_{n=0}^\infty \Pr\left( \mathcal{E}_h^1\mathcal{E}_h^2...\mathcal{E}_h^{k-1}(\mathcal{E}_h^k(100^n\beta))^c\right)\\
= & \sum_{n\geq0,100^n\beta\leq T(H+1)^2} \Pr\left( \mathcal{E}_h^1\mathcal{E}_h^2...\mathcal{E}_h^{k-1}(\mathcal{E}_h^k(100^n\beta))^c\right).
\end{align}
Combining Equation~\eqref{equ:event_decompose} and Lemma~\ref{lem:event_main}, we have that for all $(k,h)\in[K]\times[H]$,
\[
    \Pr(\mathcal{E}_h^1\mathcal{E}_h^2...\mathcal{E}_h^{k-1})-\Pr(\mathcal{E}_h^1\mathcal{E}_h^2...\mathcal{E}_h^k)\leq \delta/(32T^2)\cdot(\log(T(H+1)^2/\beta)+2)\leq \delta/32T.
\]
Thus for all $h\in [H]$ we have
\begin{align*}
&\Pr\left(\bigcap_{k=1}^K\mathcal{E}_h^k \right)\\
&=1-\sum_{k=1}^K(\Pr(\mathcal{E}_h^1\mathcal{E}_h^2...\mathcal{E}_h^{k-1})-\Pr(\mathcal{E}_h^1\mathcal{E}_h^2...\mathcal{E}_h^k))\\
&\geq 1-K\cdot(\delta/32T)\\
&=1-\delta/32H.
\end{align*}
By applying a union bound for all $h\in[H]$ we complete the proof.
\end{proof}
\subsubsection{Proof of Proposition~\ref{prop:bounded_size}}
We start our proof by showing that the summation of online sensitivity scores can be upper bounded if $\mathcal{Z}_h^{k}$, i.e, the dataset without sub-sampling, is used. 
\begin{lem}
\label{lem:sum_sen}
For all $h\in[H]$, we have
\[
\sum_{k=1}^{K-1}\operatorname{sensitivity}_{\mathcal{Z}_h^{k},\mathcal{F}}(z_h^{k}) \leq C\cdot\dim_E(\mathcal{F},1/T)\log((H+1)^2T)\log T
\]
for some absolute constant $C>0$.
\end{lem}
\begin{proof} Note that $|\mathcal{Z}_h^{k}|\leq T$, thus we have that
\begin{align*}
\sen_{\mathcal{Z}_h^{k},\mathcal{F}}(z_h^{k})&=\min\left\lbrace \sup_{f_1,f_2\in\mathcal{F}}\frac{(f_1(z_h^{k})-f_2(z_h^{k}))^2}{\min\{\|f_1-f_2\|^2_{\mathcal{Z}_h^{k}},T(H+1)^2\}+\beta},1\right\rbrace\\
&\leq\min\left\lbrace \sup_{f_1,f_2\in\mathcal{F}}\frac{(f_1(z_h^{k})-f_2(z_h^{k}))^2}{\|f_1-f_2\|^2_{\mathcal{Z}_h^{k}}+1},1\right\rbrace.
\end{align*}
For each $k\in [K-1]$, let $f_1, f_2 \in \cF$ be an arbitrary pair of functions such that  \[\frac{(f_1(z_h^{k})-f_2(z_h^{k}))^2}{\|f_1-f_2\|^2_{\mathcal{Z}_h^{k}}+1}\] is maximized, and we define $L(z_h^{k}) = (f_1(z_h^{k}) - f_2(z_h^{k}))^2$ for such $f_1$ and $f_2$.
	Note that $0 \le L(z_h^{k}) \le (H + 1)^2$.
	Let $\cZ_h^K = \bigcup_{\alpha = 0}^{\log((H + 1)^2T) - 1} \pairs^{\alpha} \cup \pairs^{\infty}$ be a dyadic decomposition with respect to $L(\cdot)$ (we assume $\log((H + 1)^2T)$ is an integer for simplicity),  where for each $0 \le \alpha < \log((H + 1)^2T)$, define
	\[	\pairs^{\alpha} = \{z_h^{k}\in\cZ_h^K \mid L(z_h^{k}) \in ((H + 1)^2 \cdot 2^{-\alpha - 1}, (H + 1)^2 \cdot 2^{-\alpha}]\}
	\]
	and
	\[
	\pairs^{\infty} = \{z_h^{k}\in\cZ_h^K \mid L(z_h^{k}) \le 1/T\}.
	\]
	Clearly, for any $z_h^{k} \in \pairs^{\infty}$, $\sen_{\pairs_h^{k}, \funclass}(z_h^{k}) \le 1 / T$ and thus 
	\[
	\sum_{z_h^{k} \in \pairs^{\infty}} \sen_{\pairs_h^{k}, \funclass}(z_h^{k}) \le 1.
	\]
	
	Now we bound $\sum_{z_h^{k} \in \pairs^{\alpha}} \sen_{\pairs_h^{k}, \funclass}(z_h^{k}) $ for each $0\le \alpha < \log((H + 1)^2T)$ separately. 
	For each $\alpha$, let \[N_{\alpha} = |\pairs^{\alpha}| / \ds_E(\funclass, (H + 1)^2 \cdot 2^{-\alpha - 1})\] and we decompose $\pairs^{\alpha}$ into $N_{\alpha} + 1$ disjoint subsets, i.e., $\pairs^{\alpha} = \bigcup_{j = 1}^{N_{\alpha} + 1}\pairs^{\alpha}_j $, by using the following procedure.
	We initialize  $\pairs^{\alpha}_j = \{\}$ for all $j$ and consider each $z_h^k\in\cZ^\alpha$ sequentially.
	For each $z_h^{k}\in\cZ^\alpha$, we find the smallest $1 \le j \le N_{\alpha}$ such that $z_h^{k}$ is $(H + 1)^2 \cdot 2^{-\alpha - 1}$-independent of $\pairs^{\alpha}_j$ with respect to $\funclass$.
	We set $j = N_{\alpha} + 1$ if such $j$ does not exist, and use $j(z_h^{k}) \in [N_{\alpha} + 1]$ to denote the choice of $j$ for $z_h^{k}$. We then add $z_h^k$ into $\cZ_j^\alpha$.
	For each $z_h^{k}\in\cZ^\alpha$, it is clear that $z_h^{k}$ is dependent on each of $\pairs^{\alpha}_1, \pairs^{\alpha}_2, \ldots, \pairs^{\alpha}_{j(z_h^{k}) - 1}$.
	
	Now we show that for each $z_h^{k} \in \pairs^{\alpha}$, 
	\[
	\sen_{\pairs_h^{k}, \funclass}(z_h^{k}) \le \frac{4}{j(z_h^{k})} .
	\]
For any $z_h^{k} \in \pairs^{\alpha}$, we use $f_1, f_2 \in \cF$ to denote the pair of functions in $\funclass$ such that \[\frac{(f_1(z_h^{k}) - f_2(z_h^{k}))^2}{\|f_1 - f_2\|_{\pairs_h^{k}}^2+1}\] is maximized.
	Since $z_h^{k} \in \pairs^{\alpha}$, we must have $(f_1(z_h^{k}) - f_2(z_h^{k}))^2 > (H + 1)^2 \cdot 2^{-\alpha - 1}$.
	Since $z_h^{k}$ is dependent on each of $\pairs^{\alpha}_1, \pairs^{\alpha}_2, \ldots, \pairs^{\alpha}_{j(z_h^{k}) - 1}$,
	and for each $1 \le t < j(z_h^{k})$, we have
	\[
	\|f_1-f_2\|_{\pairs^{\alpha}_t} \ge (H + 1)^2 \cdot 2^{-\alpha - 1}.
	\]
	It is important to note that $\pairs^{\alpha}_1, \pairs^{\alpha}_2, \ldots, \pairs^{\alpha}_{j(z_h^{k}) - 1}\subseteq \pairs_h^{k}$ due to the design of the partition procedure. Thus the online sensitivity score can be bounded by:
	\begin{align*}
	\sen_{\pairs_h^{k}, \funclass}(z_h^{k}) \leq&   \frac{(f_1(z_h^{k}) - f_2(z_h^{k}))^2}{\|f_1 - f_2\|_{\pairs_h^{k}}^2+1} \le \frac{(H + 1)^2 \cdot 2^{-\alpha}}{ \|f_1 - f_2\|_{\pairs_h^{k}}^2}\\
	\le &\frac{(H + 1)^2 \cdot 2^{-\alpha}}{\sum_{t = 1}^{j(z_h^{k}) - 1} \|f_1-f_2\|_{\pairs^{\alpha}_t} } \le 2 / (j(z_h^{k})-1).
	\end{align*}
	By the definition of online sensitivity score we have $\sen_{\pairs_h^{k}, \funclass}(z_h^{k})\le 1$, thus we conclude that:
	$$
	\sen_{\pairs_h^{k}, \funclass}(z_h^{k})\le \min\{\frac{2}{j(z_h^{k})-1},1\}\le \frac{4}{j(z_h^{k})}.
	$$

	Moreover, by the definition of $(H + 1)^2 \cdot 2^{-\alpha - 1}$-independent, we have $|\pairs^{\alpha}_j| \le \ds_E(\funclass, (H + 1)^2 \cdot 2^{-\alpha - 1})$ for all $1 \le j \le N_{\alpha}$.
	Therefore, 
	\begin{align*}
	&\sum_{z_h^{k} \in \pairs^{\alpha}} \sen_{\pairs_h^{k}, \funclass}(z_h^{k})  \le \sum_{1 \le j \le N_{\alpha}} |\pairs^{\alpha}_j| \cdot 4/j + \sum_{z \in \pairs^{\alpha}_{N_{\alpha} + 1}} 4/N_{\alpha} \\
	\le &4 \ds_E(\funclass, (H + 1)^2 \cdot 2^{-\alpha - 1}) \ln(N_{\alpha}) + |\pairs^{\alpha}| \cdot \frac{4\ds_E(\funclass, (H + 1)^2 \cdot 2^{-\alpha - 1})}{|\pairs^{\alpha}|}\\
	\le & 8 \ds_E(\funclass, (H + 1)^2 \cdot 2^{-\alpha - 1}) \log T.
	\end{align*}
	By the monotonicity of eluder dimension, it follows that
	\begin{align*}
	&\sum_{k=1}^{K-1}\operatorname{sensitivity}_{\mathcal{Z}_h^{k},\mathcal{F}}(z_h^{k}) \\
	\le &\sum_{\alpha = 0}^{\log((H + 1)^2T) - 1} \sum_{z_h^{k} \in \pairs^{\alpha}} \sen_{\pairs_h^{k}, \funclass}(z_h^{k})
	+  \sum_{z_h^{k} \in \pairs^{\infty}} \sen_{\pairs_h^{k}, \funclass}(z_h^{k})\\
	\le & 8 \log((H + 1)^2T)\ds_E(\funclass, 1/T) \log T + 1 \\
	\le & 9 \log((H + 1)^2T)\ds_E(\funclass, 1/T) \log T
	\end{align*}
    as desired.\end{proof}
With Lemma~\ref{lem:sum_sen}, now we are ready to prove Proposition~\ref{prop:bounded_size}.

\begin{proof}[Proof of Proposition~\ref{prop:bounded_size}] 
Firstly, note that conditioned on $\mathcal{E}_h^{k}$, which means $\widehat{\mathcal{Z}}_h^{k}$ is a good approximation to $\mathcal{Z}_h^{k}$, the online sensitivity score computed with $\widehat{\mathcal{Z}}_h^{k}$ will be relatively accurate. Formally, this argument can be stated as
\[
\mathbb{I}\{\mathcal{E}_h^{k}\}\cdot \operatorname{sensitivity}_{\widehat{\mathcal{Z}}_h^{k},\mathcal{F}}(z_h^{k})\leq C\cdot \operatorname{sensitivity}_{\mathcal{Z}_h^{k},\mathcal{F}}(z_h^{k})
\]
for some absolute constant $C>0$. This property can be easily derived from Lemma~\ref{lem:app}.

Note that in Algorithm~\ref{alg:sample},\footnote{We use $f\lesssim g$ to donote that $f\leq Cg$ for some absolute constant $C>0$.}
$$
p_z\lesssim\operatorname{sensitivity}_{\mathcal{Z},\mathcal{F}}(z)\cdot\log(T\mathcal{N}(\mathcal{F},\sqrt{\delta/64T^3})/\delta) .
$$
Thus from Lemma~\ref{lem:sum_sen} we know that
\begin{align*}
\sum_{k=1}^{K-1}\mathbb{I}\{\mathcal{E}_h^{k}\}\cdot p_{z_h^{k}}
&\lesssim  \sum_{k=1}^{K-1}\mathbb{I}\{\mathcal{E}_h^{k}\}\cdot \operatorname{sensitivity}_{\widehat{\mathcal{Z}}_h^{k},\mathcal{F}}(z_h^{k})\cdot\log(T\mathcal{N}(\mathcal{F},\sqrt{\delta/64T^3})/\delta)\\
&\lesssim \sum_{k=1}^{K-1} \operatorname{sensitivity}_{\mathcal{Z}_h^{k},\mathcal{F}}(z_h^{k})\cdot \log(T\mathcal{N}(\mathcal{F},\sqrt{\delta/64T^3})/\delta) \\
&\lesssim\log(T\mathcal{N}(\mathcal{F},\sqrt{\delta/64T^3})/\delta) \dim_E(\mathcal{F},1/T)\log^2 T.
\end{align*}
As we can adjust the constant $C$ in Proposition~\ref{prop:bounded_size}, we assume that
\[
\sum_{k=1}^{K-1}\mathbb{I}\{\mathcal{E}_h^{k}\}\cdot p_{z_h^{k}}\leq S_{\max}/3.
\]
For $2\leq k\leq K$, let $X^k_h$ be a random  variable defined as:
\[
X^k_h=
\begin{cases}
\mathbb{I}\{\mathcal{E}_h^{k-1}\} ~~~ &\hat{z}_h^{k-1}\text{ is added into } \widehat{\mathcal{Z}}_h^{k}\\
0 ~~~ &\text{otherwise}
\end{cases}.
\]
Then $X_h^k$ is adapted to the filtration $\mathscr{F}_k$. We have that $\mathbb{E}_{k-1}[X^k_h]=p_{z_h^{k-1}}\cdot\mathbb{I}\{\mathcal{E}_h^{k-1}\}$ and $\mathbb{E}_{k-1}[(X^k_h-\mathbb{E}_{i-1}[X^k_h])^2]=\mathbb{I}\{\mathcal{E}_h^{k-1}\}\cdot p_{z_h^{k-1}}(1-p_{z_h^{k-1}})$.
Note that $X^k_h-\mathbb{E}_{k-1}[X^k_h]$ is a martingale difference sequence with respect to $\mathscr{F}_k$ and
\[
\sum_{k=2}^{K}\mathbb{E}_{k-1}[(X^k_h-\mathbb{E}_{k-1}[X^k_h])^2]=\sum_{k=2}^K\mathbb{I}\{\mathcal{E}_h^{k-1}\}p_{z_h^{k-1}}(1-p_{z_h^{k-1}})
\leq\sum_{k=2}^K\mathbb{I}\{\mathcal{E}_h^{k-1}\}\cdot p_{z_h^{k-1}}\leq
S_{\max}/3,
\]
\[
\sum_{k=2}^K\mathbb{E}_{k-1}[X^k_h]=\sum_{k=2}^Kp_{z_h^{k-1}} \mathbb{I}\{\mathcal{E}_h^{k-1}\}\leq S_{\max}/3.
\]
By Freedman's inequality (Lemma~\ref{lem:freedman}), we have that
\begin{align*}
&\Pr\left\lbrace \sum_{k=2}^KX^k_h\geq S_{\max}\right\rbrace \\
\leq& \Pr\left\lbrace \left|  \sum_{k=2}^K(X^k_h-\mathbb{E}_{k-1}[X^k_h])
\right| \geq2S_{\max}/3\right\rbrace \\
\leq& 2\exp\left\lbrace -\frac{(2S_{\max}/3)^2/2}{S_{\max}/3+2S_{\max}/9}\right\rbrace \\
\leq& \delta/(32T).
\end{align*}
With a union bound we know that with probability at least $1-\delta/32$,
$$
\sum_{k=2}^KX^k_h<S_{\max}\quad  \forall h\in[H].
$$
We condition on the above event and $\bigcap_{h=1}^H\bigcap_{k=1}^K\mathcal{E}_h^k $. In this case, it is clear that for all $h\in[H]$, we add elements into $\widehat{\mathcal{Z}}_h^k$ for at most $S_{\max}$ times. Combining the above result with Lemma~\ref{lem:markov} completes the proof.
\end{proof}
\subsection{Proof of Theorem~\ref{thm:main_regret}}
\label{pf:thm2}
Here we provide the complete proof of Theorem~\ref{thm:main_regret}.
\subsubsection{Analysis of the Optimistic Planning}

Our next lemma bounds the complexity of the bonus function. This step is essential for showing optimism, as we need to establish a uniform convergence argument to deal with the dependency in the data sequence.
\begin{lem}
\label{lem:bonus_class}
With probability at least $1-\delta/8$, for all $(h,k)\in[H]\times[K]$, $b_h^k(\cdot,\cdot)\in\mathcal{M}$.

Here $\mathcal{M}$ is a prespecified function class with bounded size:
\begin{align*}
&\log|\mathcal{M}|\\
\leq &C'\cdot\log(T\mathcal{N}(\mathcal{F},\sqrt{\delta/64T^3})/\delta) \cdot\dim_E(\mathcal{F},1/T)\cdot\log^2 T\cdot\log\left(\mathcal{C}(\mathcal{S}\times\mathcal{A},1/(16\sqrt{64T^3/\delta})) \cdot64T^3/\delta\right)\\
\leq &C\cdot\log(T\mathcal{N}(\mathcal{F},\delta/T^2)/\delta) \cdot\dim_E(\mathcal{F},1/T)\cdot\log^2 T\cdot\log\left(\mathcal{C}(\mathcal{S}\times\mathcal{A},\delta/T^2) \cdot T/\delta\right)
\end{align*}
for some absolute constant $C',C>0$ if $T$ is sufficiently large.
\end{lem}
\begin{proof}
Define $\mathcal{M}$ to be the set of all functions with the form:
\[
\left\lbrace |f_1(\cdot,\cdot)-f_2(\cdot,\cdot)|\mid \|f_1-f_2\|^2_{\mathcal{Z}}\leq\beta\right\rbrace,\quad \mathcal{Z}\in\Omega
\]
where $\Omega$ contains all set with bounded size:
\[
\Omega:=\{\mathcal{Z}\subseteq\mathcal{C}(\mathcal{S}\times\mathcal{A},1/(16\sqrt{64T^3/\delta}))\mid|\mathcal{Z}|\leq 64T^3/\delta,n_d(\cZ)\leq S_{\max}\}
\]
where $S_{\max}$ is defined in Proposition~\ref{prop:bounded_size}.

Conditioned on the event defined in Proposition~\ref{prop:bounded_size}, $\widehat{\mathcal{Z}}_h^k\in \Omega$, and $b_h^k(\cdot,\cdot)\in\mathcal{M}$ for all $(h,k)\in[H]\times[K]$. 
\end{proof}
The next lemma estimates the error of the one-step bellman backup.
\begin{lem}
\label{lem:one_step_error}
Consider a fixed pair $(k,h)\in[K]\times[H]$.
For any $V:\cS\rightarrow[0,H]$, define
	\[
	\mathcal{D}_h^k(V):=\{(s_{h}^{\tau},a_{h}^{\tau},r_h^\tau+ V(s_{h+1}^{\tau}))\}_{\tau \in [k-1]}
	\]
	and also
	\[
	\widehat{f_V}:=\operatorname{argmin}_{f\in\mathcal{F}}\|f\|_{\mathcal{D}_h^k(V)}^2.
	\]
	For any $V:\cS\rightarrow[0,H]$ and $\delta \in (0,1)$, there is an event $\mathcal{E}_{V,\delta}$ which holds with probability at least $1-\delta$, such that for any $V':\mathcal{S}\rightarrow[0,H]$ with $\|V'-V\|_{\infty}\leq1/T$, we have
	\[
	\left\|\widehat{f_{V'}}(\cdot,\cdot)-r_h(\cdot,\cdot)-\sum_{s'\in\mathcal{S}}P_h(s'|\cdot,\cdot)V'(s')\right\|_{\mathcal{Z}_h^k}
	\leq C\cdot(H\sqrt{\log(1/\delta)+\log\mathcal{N}(\mathcal{F},1/T)}).
	\]
	for some absolute constant $C>0$.
\end{lem}
\begin{proof} 
The proof is almost identical to that of Lemma~5 in \cite{wang2020reinforcement}. We provide a proof here for completeness.

In our proof, we consider a fixed $V:\cS\to[0,H]$, and define
\[f_V(\cdot,\cdot)  := r_h(\cdot, \cdot)+\sum_{s'\in\cS}P_h(s'\mid\cdot, \cdot)V(s').\]
By Assumption~\ref{assum:express_1}, we have that $f_V(\cdot,\cdot)\in\cF$.

For any $f \in \funclass$, we consider $\sum_{\tau=1}^{k-1}\xi_{h}^{\tau}(f)$ where 
\[\xi_{h}^\tau(f) := 2(f(s_{h}^\tau, a^\tau_{h}) - f_V(s_{h}^\tau, a^\tau_{h}))\cdot (f_V(s_{h}^\tau, a^\tau_{h}) -r_{h}^\tau- V(s_{h+1}^\tau)).\]
Then $\xi^\tau_h(f)$ is adapted to the filtration $\mathscr{F}_\tau$ and $\EE_{\tau-1}\left[\xi^\tau_h(f)\right]=0$.
Moreover,
\[
|\xi^\tau_{h}(f)|\le 2(H + 1)\left|f(s^\tau_{h},a^\tau_{h}) - f_V(s^\tau_{h}, a^\tau_{h})\right|.
\]
By Azuma-Hoeffding inequality, we have
\[
\Pr\left[\left|\sum_{\tau=1}^{k-1} \xi^\tau_{h}(f)\right|\ge \varepsilon\right]
\le 2\exp\left(-\frac{\varepsilon^2}{8(H + 1)^2\|f-f_V\|_{\pairs^k_h}^2}\right).
\]
Let 
\begin{align*}
\varepsilon &= \left(8(H + 1)^2\log\left(\frac{2\coversize(\funclass, 1 / T)}{ \delta}\right)\cdot \|f-f_V\|_{\pairs^k_h}^2\right)^{1/2}\\
&\le 4(H + 1)\|f-f_V\|_{\pairs^k_h}\cdot \sqrt{\log(2/\delta) +  \log\coversize(\cF, 1/T)}.
\end{align*}
We have, with probability at least $1-\delta$,
for all $f\in \cover(\funclass, 1 / T)$,
\[
\left|\sum_{\tau=1}^{k-1} \xi^\tau_{h}(f)\right|
\le 4(H + 1)\|f-f_V\|_{\pairs^k_h}\cdot  \sqrt{\log(2/\delta) +  \log\coversize(\cF, 1/T)}.
\]
We define the above event to be $\cE_{V, \delta}$, and we condition on this event for the rest of the proof.

For all $f\in \cF$, 
there exists $g\in \cover(\cF, 1/T)$,
such that $\|f-g\|_{\infty}\le 1/T$, and
we have
\begin{align*}
\left|\sum_{\tau=1}^{k-1} \xi^\tau_{h}(f)\right|
&\le 
\left|\sum_{\tau=1}^{k-1} \xi^\tau_{h}(g)\right|
+ 2(H + 1)\\
&\le 4(H + 1)\|g-f_V\|_{\pairs^k_h}\cdot  \sqrt{\log(2/\delta) +  \log\coversize(\cF, 1/T)} + 2( H + 1)\\
&\le 4(H + 1)(\|f-f_V\|_{\pairs^k_h} + 1)\cdot  \sqrt{\log(2/\delta) +  \log\coversize(\cF, 1/T)} + 2(H + 1)
.
\end{align*}

Consider $V':\cS\rightarrow [0,H]$ with $\|V'-V\|_{\infty }\le 1/T$. 
We have
\[
\|{f_{V'}} - f_V\|_{\infty}
\le 
 \|V'-V\|_{\infty} \le 1/T.
\]
For any $f \in \funclass$, 
\begin{align*}
\|f\|_{\cD^k_{h}(V')}^2 -\|f_{V'}\|_{\cD^k_{h}(V')}^2
=&\|f-f_{V'}\|_{\pairs^k_h}^2 + 2\sum_{\tau=1}^{k-1}(f(s^{\tau}_{h}, a^{\tau}_{h}) - {f_{V'}}(s^{\tau}_{h}, a^{\tau}_{h}))\cdot ({f_{V'}}(s^{\tau}_{h}, a^{\tau}_{h}) -  r_{h}^{\tau}- V'(s_{h+1}^\tau)).
\end{align*}
For the second term, we have,
\begin{align*}
&2\sum_{\tau=1}^{k-1}(f(s^{\tau}_{h}, a^{\tau}_{h}) - {f_{V'}}(s^{\tau}_{h}, a^{\tau}_{h}))\cdot ({f_{V'}}(s^{\tau}_{h}, a^{\tau}_{h}) - r_{h}^{\tau} - V'(s_{h+1}^\tau))\\
\ge& 
2\sum_{\tau=1}^{k-1}(f(s^{\tau}_{h}, a^{\tau}_{h}) - f_V(s^{\tau}_{h}, a^{\tau}_{h}))\cdot (f_V(s^{\tau}_{h}, a^{\tau}_{h}) -  r_{h}^{\tau}- V(s_{h+1}^\tau)) - 4(H + 1)\cdot\|V'-V\|_{\infty}\cdot k\\
=& 
\sum_{\tau=1}^{k-1} \xi^\tau_{h}(f) - 4(H + 1)\cdot\|V'-V\|_{\infty}\cdot k\\
\ge&
 -4(H + 1)(\|f-f_V\|_{\pairs^k_h} + 1)\cdot  \sqrt{\log(2/\delta) +  \log\coversize(\cF, 1/T)} - 2(H + 1) - 4(H + 1)\cdot\|V'-V\|_{\infty}\cdot k\\
\ge&
-4(H + 1)(\|f-f_{V'}\|_{\pairs^k_h}+2)\cdot  \sqrt{\log(2/\delta) +  \log\coversize(\cF, 1/T)} - 6(H + 1).
\end{align*}
Recall that $\wh{f}_{V'} =\arg\min_{f\in \cF}\|f\|_{\cD^k_h(V')}^2$.
We have $\|\wh{f}_{V'}\|_{\cD^k_h({V'})}^2 - \|f_{V'}\|_{\cD^k_h({V'})}^2\le 0$, which implies, 
\begin{align*}
0&\ge \|\wh{f}_{V'}\|_{\cD^k_h({V'})}^2 - \|f_{V'}\|_{\cD^k_h({V'})}^2\\
&=\|\wh{f}_{V'}-f_{V'}\|_{\pairs^k_h}^2 + 2\sum_{\tau=1}^{k-1}(\wh{f}_{V'}(s^{\tau}_{h}, a^{\tau}_{h}) - {f_{V'}}(s^{\tau}_{h}, a^{\tau}_{h}))\cdot ({f_{V'}}(s^{\tau}_{h}, a^{\tau}_{h}) -  r_{h}^{\tau}- V'(s_{h+1}^\tau))\\
&\ge 
\|\wh{f}_{V'}-f_{V'}\|_{\pairs^k_h}^2 -4 ( H + 1)(\|\wh{f}_{V'}-f_{V'}\|_{\pairs^k_h}+2)\cdot  \sqrt{\log(2/\delta) +  \log\coversize(\cF, 1/T)} - 6(H + 1).
\end{align*}
Solving the above inequality, we have,
\[
\|\wh{f}_{V'}-f_{V'}\|_{\pairs^k_h}
\le C\cdot\big(H\cdot\sqrt{\log(1/\delta) + \log\coversize(\cF, 1/T)}\big)
\]
for some absolute constant $C>0$.
\end{proof}
Our next lemma shows that $f_h^k$ belongs to the desired confidence region.
\begin{lem}
\label{lem:confidence_region}
Let $\mathcal{E}_2$ denote the event that for all $(k,h)\in[K]\times[H]$,
	\[
	\left\|f_h^k-\bar{f}_h^k \right\|_{\mathcal{Z}_h^k}\leq \beta/100
	\]
	where $\bar{f}_h^k(\cdot,\cdot)=\sum_{s\in\mathcal{S}'}P_h(s'|\cdot,\cdot)V_{h+1}^k(s')+r_h(\cdot,\cdot).$
	
	Then $\Pr[\mathcal{E}_2]\geq 1-\delta/4$ provided
	\[
	\beta\geq C\cdot H^2\cdot(\log(T/\delta)+\log\mathcal{N}(\mathcal{F},1/T)+\log|\mathcal{M}|)
	\]
	for some absolute constant $C>0$.
\end{lem}
\begin{proof} Note that for all $(k,h)\in[K]\times[H]$,
\[
Q_h^k(\cdot,\cdot)=\min\{f_h^k(\cdot,\cdot)+b_h^k(\cdot,\cdot),H\},\text{and}
\]
\[
V_h^k(\cdot)=\max_{a\in \mathcal{A}}Q_h^k(\cdot,a).
\]
We define
\[
\mathcal{Q}:=\{\min\{f(\cdot,\cdot)+m(\cdot,\cdot) ,H\}|f\in\mathcal{C}(\mathcal{F},1/T),m\in \mathcal{M}\}\cup\{0\},\text{and}
\]
\[
\mathcal{V}:=\{\max_{a\in\mathcal{A}}q(\cdot,a)|q\in\mathcal{Q}\}.
\]
Then $\log|\mathcal{V}|\leq \log|\mathcal{M}|+\log\mathcal{N}(\mathcal{F},1/T)+1.$

Conditioned on the event defined in Lemma~\ref{lem:bonus_class}, we have that $b_h^k(\cdot,\cdot)\in\mathcal{M}$ for all $(h,k)\in[H]\times[K]$. Thus for all $(k,h)\in[K]\times[H]$, $\mathcal{V}$ is a $(1/T)$-cover of $V_h^k(\cdot)$.

For each $V\in \cV$, let $\cE_{V,\delta/(8|\cV|T)}$ be the event defined in Lemma~\ref{lem:one_step_error}. Note that $\cE_{V,\delta/(8|\cV|T)}$ also relates to a fixed pair $(k,h)$. 
By applying Lemma~\ref{lem:one_step_error} and a union bound, we have
$\Pr\left[\bigcap_{(k,h)\in[K]\times[H]}\bigcap_{V\in \cV} \cE_{V,\delta/(8|\cV|T)}
\right]\ge 1-\delta/8$.
We also condition on this event for the rest of the proof.

For $(k,h)\in[K]\times[H]$, recall that $f_{h}^{k}$ is the solution to the regression problem in Algorithm~\ref{alg:main}, i.e., 
$f^k_{h}= \argmin_{f\in \cF}\|f\|_{\cD^{k}_h}^2$.
Let $V\in \cV$ such that $\|V - V^{k}_{h+1}\|_{\infty} \le 1/T$.
By the definition of $\cE_{V,\delta/(8|\cV|T)}$ (the one relats to this $(k,h)$ pair), we have that
\begin{align*}
&\left\|f_h^k(\cdot,\cdot)-r_h(\cdot,\cdot)-\sum_{s'\in\mathcal{S}}P_h(s'|\cdot,\cdot)V_{h+1}^k(s')\right\|_{\mathcal{Z}_h^k}\\
\lesssim &H\sqrt{\log(8|\cV|T/\delta)+\log\mathcal{N}(\mathcal{F},1/T)}\\
\lesssim &H\sqrt{\log(T/\delta)+\log\mathcal{N}(\mathcal{F},1/T)+\log|\mathcal{M}|}.
\end{align*}
as desired.
\end{proof}
Finally, we use the above result to show optimism.
\begin{lem}
\label{lem:optimism}
Let $\mathcal{E}_3$ denote the event that for all $(k,h)\in [K]\times[H]$, and all $(s,a)\in\mathcal{S}\times\mathcal{A}$,
\[
Q^*_h(s,a)\leq Q_h^k(s,a) \leq \bar{f}_h^k(s,a)+2b_h^k(s,a)
\]
where
$\bar{f}_h^k(\cdot,\cdot)=\sum_{s\in\mathcal{S}'}P_h(s'|\cdot,\cdot)V_{h+1}^k(s')+r_h(\cdot,\cdot)$.

Then $\Pr[\mathcal{E}_3]\geq1-\delta/2$.
\end{lem}
\begin{proof} 
We condition on the event defined in Proposition~\ref{prop:bounds_of_bonus} and $\mathcal{E}_2$ defined in Lemma~\ref{lem:confidence_region}. Because $\left\|f_h^k-\bar{f}_h^k \right\|_{\mathcal{Z}_h^k}\leq \beta/100$, from the definition of $\underline{b}_h^k$ we have that $\left| \bar{f}_h^k(\cdot,\cdot)-f_h^k(\cdot,\cdot)\right|\leq \underline{b}_h^k(\cdot,\cdot)$. Moreover, by Proposition~\ref{prop:bounds_of_bonus} we have $\underline{b}_h^k(\cdot,\cdot) \leq b_h^k(\cdot,\cdot)$.

Thus for all
$ (k,h)\in[K]\times[H]$, $(s,a)\in \mathcal{S}\times\mathcal{A}$, we have
\begin{align*}
&Q_h^k(s,a)\\
=&\min\{f_h^k(s,a)+b_h^k(s,a),H\}\\
\leq& \min\left\lbrace \bar{f}_h^k(s,a)+b_h^k(s,a)+\left| \bar{f}_h^k(s,a)-f_h^k(s,a)\right|,H\right\rbrace  \\
\leq&  \min\left\lbrace \bar{f}_h^k(s,a)+b_h^k(s,a)+\underline{b}_h^k(s,a),H\right\rbrace \\
\leq&  \min\left\lbrace \bar{f}_h^k(s,a)+2b_h^k(s,a),H\right\rbrace \\
\leq&   \bar{f}_h^k(s,a)+2b_h^k(s,a).
\end{align*}
Next we use induction on $h$ to prove $Q^*_h(\cdot,\cdot)\leq Q_h^k(\cdot,\cdot)$. The inequality clearly holds when $h=H+1$. Now we assume $Q^*_{h+1}(\cdot,\cdot)\leq Q_{h+1}^k(\cdot,\cdot)$ for some $h\in[H]$. Then obviously we have $V^*_{h+1}(\cdot)\leq V_{h+1}^k(\cdot)$. Therefore for all $(s,a)\in\mathcal{S}\times\mathcal{A}$,
\begin{align*}
&Q^*_h(s,a)\\
=&  r_h(s,a)+\sum_{s'\in \mathcal{S}}P_h(s'|s,a)V_{h+1}^*(s')\\
\leq& \min\left\lbrace  r_h(s,a)+\sum_{s'\in \mathcal{S}}P_h(s'|s,a)V_{h+1}^k(s'),H\right\rbrace \\
=& \min\left\lbrace  \bar{f}_h^k(s,a),H\right\rbrace\\
\leq&\min\left\lbrace  f_h^k(s,a)+b_h^k(s,a),H\right\rbrace\\
=&Q^k_h(s,a).
\end{align*}
\end{proof}
\subsubsection{Regret decomposition}
Now we are ready to bound the regret. For any $k\in[K]$, we let $\tilde{k}$ represents the episode index we update the policy to the one used in the $k$-th episode.
\begin{lem}
\label{lem:regret_decompose}
With probability at least $1-5\delta/8$, we have
\[
\text{Regret}(K)\leq 4H\sqrt{KH\cdot\log(16/\delta)}+2\sum_{k=1}^K\sum_{h=1}^H b_h^{k}(s_h^k,a_h^k).
\]
\end{lem}
\begin{proof}
$\forall(k,h)\in[K]\times[H-1]$, define
$$
\xi_h^k=\sum_{s'\in \mathcal{S}}P_h(s'|s_h^k,a_h^k)(V_{h+1}^{\tilde{k}}(s')-V_{h+1}^{\pi_{\tilde{k}}}(s'))
-(V_{h+1}^{\tilde{k}}(s_{h+1}^k)-V_{h+1}^{\pi_{\tilde{k}}}(s_{h+1}^k)).
$$
Note that $\{\xi_h^k\}$ is a martingale difference sequence and $|\xi_h^k|\leq 2H$. By Azuma-Hoeffding inequality, with probability at least $1-\delta/8$, 
$$
\sum_{k=1}^K\sum_{h=1}^{H-1}\xi_h^k\leq 4H\sqrt{KH\cdot\log(16/\delta)}.
$$
Conditioned on the above event and $\mathcal{E}_3$ defined in Lemma~\ref{lem:optimism}, we have
\begin{align*}
&\text{Regret}(K)\\
=&\sum_{k=1}^K\left( V^*_1(s_1^k)-V_1^{\pi_{\tilde{k}}}(s_1^k)\right) \\
\leq& \sum_{k=1}^K\left(V^{\tilde{k}}_1(s_1^k)-V_1^{\pi_{\tilde{k}}}(s_1^k) \right) \\
=& \sum_{k=1}^K\left(Q^{\tilde{k}}_1(s_1^k,a_1^k)-Q_1^{\pi_{\tilde{k}}}(s_1^k,a_1^k) \right) \quad(\text{note that we played }\pi_{\tilde{k}}\text{ in epsiode }k)\\
=&\sum_{k=1}^K\left(\sum_{s'\in \mathcal{S}}P_1(s'|s_1^k,a_1^k)(V_{2}^{\tilde{k}}(s')-V_{2}^{\pi_{\tilde{k}}}(s'))+2b_1^{\tilde{k}}(s_1^k,a_1^k)\right) \\
=&\sum_{k=1}^K\left((V_{2}^{\tilde{k}}(s_2^k)-V_{2}^{\pi_{\tilde{k}}}(s_2^k))+\xi_1^k+2b_1^{\tilde{k}}(s_1^k,a_1^k)\right) \\
\leq& ...\\
\leq& \sum_{k=1}^K\sum_{h=1}^{H-1}\xi_h^k+2\sum_{k=1}^K\sum_{h=1}^H b_h^{\tilde{k}}(s_h^k,a_h^k)\\
=& \sum_{k=1}^K\sum_{h=1}^{H-1}\xi_h^k+2\sum_{k=1}^K\sum_{h=1}^H b_h^{k}(s_h^k,a_h^k)\quad(\text{note that } b_h^{\tilde{k}}(\cdot,\cdot)=b_h^k(\cdot,\cdot))\\
\leq& 4H\sqrt{KH\cdot\log(16/\delta)}+2\sum_{k=1}^K\sum_{h=1}^H b_h^{k}(s_h^k,a_h^k)
\end{align*}
as desired.
\end{proof}
The next lemma bounds the summation of the exploration bonus in terms of the eluder dimension.
\begin{lem}
\label{lem:bonus_sum}
With probability at least $1-\delta/32$,
\[
\sum_{k=1}^K\sum_{h=1}^H b_h^k(s_h^k,a_h^k)\leq H+H(H+1)\dim_E(\mathcal{F},1/T)+C\cdot\sqrt{\dim_E(\mathcal{F},1/T)\cdot TH\cdot\beta}
\]
for some absolute constant $C>0$.
\end{lem}
\begin{proof} We condition on the event defined in Proposition~\ref{prop:bounds_of_bonus} in the proof. Then we have
\begin{align*}
b_h^k(s_h^k,a_h^k)
&\leq \overline{b}_h^k(s_h^k,a_h^k)\\
&=\sup_{\|f_1-f_2\|^2_{\mathcal{Z}^k_h}\leq100\beta}|(f_1(s_h^k,a_h^k)-f_2(s_h^k,a_h^k)|.
\end{align*}
In the rest of the proof, we bound $\sum_{k=1}^K b_h^k(s_h^k,a_h^k)$ for each $h\in[H]$ separately.

For any given $\epsilon>0$ and $h\in[H]$, let $\mathcal{L}_h=\{(s_h^k,a_h^k)|k\in[K],b_h^k(s_h^k,a_h^k)>\epsilon\}$ with $|\mathcal{L}_h|=L_h$. We will show that there exists $z_h^k:=(s_h^k,a_h^k)\in\mathcal{L}_h$ such that $(s_h^k,a_h^k)$ is $\epsilon$-dependent on at least $L_h/\dim_{E}(\mathcal{F},\epsilon)-1$ disjoint subsequences in $\mathcal{Z}_h^k\cap \mathcal{L}_h$.
Denote $N=L_h/\dim_{E}(\mathcal{F},\epsilon)-1$. 

We decompose $\mathcal{L}_h$ into $N+1$ disjoint subsets, $\mathcal{L}_h=\cup_{j=1}^{N+1}\mathcal{L}_h^j$ by the following procedure. We initialize $\mathcal{L}_h^j=\{\}$ for all $j$ and consider each $z_h^k\in \mathcal{L}_h$ sequentially. For each $z_h^k\in \mathcal{L}_h$, we find the smallest $1\leq j\leq N$ such that $z_h^k$ is $\epsilon$-independent on $\mathcal{L}_h^j$ with respect to $\mathcal{F}$. We set $j=N+1$ if such $j$ does not exist. We add $z_h^k$ into $\mathcal{L}_h^j$ afterwards. When the decomposition of $\mathcal{L}_h$ is finished, $\mathcal{L}_h^{N+1}$ must be nonempty as $\mathcal{L}_h^{j}$ contains at most $\dim_{E}(\mathcal{F},\epsilon)$ elements for $j\in[N]$. For any $z_h^k\in\mathcal{L}_h^{N+1}$, $z_h^k$ is $\epsilon$-dependent on at least $L_h/\dim_{E}(\mathcal{F},\epsilon)-1$ disjoint subsequences in $\mathcal{Z}_h^k\cap \mathcal{L}_h$.

On the other hand, there exist $f_1,f_2\in\mathcal{F}$ such that $|f_1(s_h^k,a_h^k)-f_2(s_h^k,a_h^k)|>\epsilon$ and $  \|f_1-f_2\|^2_{\mathcal{Z}^k_h}\leq100\beta$. By the definition of $\epsilon$-dependent we have
\[
(L_h/\dim_{E}(\mathcal{F},\epsilon)-1)\epsilon^2\leq\|f_1-f_2\|^2_{\mathcal{Z}_h^k}\leq100\beta\]
which implies
\[
L_h\leq\left( \frac{100\beta}{\epsilon^2}+1\right)\dim_{E}(\mathcal{F},\epsilon).
\]
Let $b_1\geq b_2\geq...\geq b_K$ be a permutation of $\{b_h^k(s_h^k,a_h^k)\}_{k\in[K]}$. For any $b_k\geq1/K$, we have 
\[
k\leq\left( \frac{100\beta}{b_k^2}+1\right)\dim_{E}(\mathcal{F},b_k)\leq\left( \frac{100\beta}{b_k^2}+1\right)\dim_{E}(\mathcal{F},1/K)
\]
which implies
\[
b_k\leq\left(\frac{t}{\dim_E(\mathcal{F},1/K)}-1 \right)^{-1/2}\cdot\sqrt{100\beta} .
\]
Moreover, we have $b_k\leq H+1$. Therefore,
\begin{align*}
\sum_{k=1}^Kb_k\leq& 1+(H+1)\dim_E(\mathcal{F},1/K)+\sum_{\dim_E(\mathcal{F},1/K)<k\leq K}\left(\frac{k}{\dim_E(\mathcal{F},1/K)}-1 \right)^{-1/2}\cdot\sqrt{100\beta}\\
\leq& 1+(H+1)\dim_E(\mathcal{F},1/K)+C\cdot\sqrt{\dim_E(\mathcal{F},1/K)\cdot K\cdot\beta}.
\end{align*}
Summing up for all $h\in[H]$, we conclude that with probability at least $1-\delta/8$,
\begin{align*}
\sum_{k=1}^K\sum_{h=1}^H b_h^k(s_h^k,a_h^k)&\leq H+H(H+1)\dim_E(\mathcal{F},1/K)+CH\cdot\sqrt{\dim_E(\mathcal{F},1/K)\cdot K\cdot\beta}\\
&=H+H(H+1)\dim_E(\mathcal{F},1/K)+C\cdot\sqrt{\dim_E(\mathcal{F},1/K)\cdot TH\cdot\beta}\\
&=H+H(H+1)\dim_E(\mathcal{F},1/T)+C\cdot\sqrt{\dim_E(\mathcal{F},1/T)\cdot TH\cdot\beta}
\end{align*}
as desired.
\end{proof}
\begin{proof}[Proof of Theorem~\ref{thm:main_regret}] By Lemma~\ref{lem:regret_decompose} and Lemma~\ref{lem:bonus_sum}, with probability at least $1-3\delta/4$, we have
\begin{align*}
&\text{Regret}(K)\\
\leq& 4H\sqrt{KH\cdot\log(16/\delta)}+2\sum_{k=1}^K\sum_{h=1}^H b_h^{k}(s_h^k,a_h^k)\\
=&4H\sqrt{KH\cdot\log(16/\delta)}+2\left(  H+H(H+1)\dim_E(\mathcal{F},1/T)+C\cdot\sqrt{\dim_E(\mathcal{F},1/T)\cdot TH\cdot\beta}\right) \\
\lesssim& \sqrt{\dim_E(\mathcal{F},1/T)\cdot T\cdot H^3\cdot\log(T\mathcal{N}(\mathcal{F},\delta/T^2)/\delta)\cdot \dim_E(\mathcal{F},1/T)\cdot\log^2 T\cdot\log\left(\mathcal{C}(\mathcal{S}\times\mathcal{A},\delta/(T^2)) \cdot T/\delta\right)}\\
\lesssim& \sqrt{T\cdot H^3\cdot\log(T\mathcal{N}(\mathcal{F},\delta/T^2)/\delta)\cdot \dim^2_E(\mathcal{F},1/T)\cdot\log^2 T\cdot\log\left(\mathcal{C}(\mathcal{S}\times\mathcal{A},\delta/(T^2)) \cdot T/\delta\right)}.
\end{align*}

We also condition on the event defined in Proposition~\ref{prop:bounded_size}, in which case the global switching cost is bounded by $O( H\cdot \iota_2)$. At the same time, under the event defined in Proposition~\ref{prop:bounded_size}, the size of the sub-sampled dataset $\widehat{\cZ}_h^k$ is at most $\widetilde{O}(\poly(dH))$. Thus by the analysis in Section~\ref{sec:bisearch} the algorithm calls a $\Omega(K)$-sized regression oracle for at most $\widetilde{O}(d^2H^2)$ times.
\end{proof}
\subsection{Proof of Theorem~\ref{thm:main_regret2}}
\label{proof_thm3}
Here we first provide a proof sketch of Theorem~\ref{thm:main_regret2}. The full proof then follows.
\subsubsection{Proof Sketch}
For simplicity, we only discusses the case of choosing $\cD_{\Delta}$ as the distribution family $\Pi$ in the definition of Bellman eluder dimension. The proof for the case of choosing $\cD_{\cF}$ as the distribution family is similar. Similar to \citet{jin2021bellman} we can show optimism, i.e., $Q^*\in \cB^k$ for all $k\in[K]$. Thus optimistic planning step guarantees that $V_1^*(s_1)\leq \max_{a\in\cA} Q_1^k(s_1,a)$. Thus the regret can be decomposed as follows:
{\small
\begin{align*}
\text{Regret}(K)&\leq \sum_{k=1}^K\left(\max_{a\in A}Q_1^{\tilde{k}}(s_1,a)-V_1^{\pi^{\tilde{k}}}(s_1)\right)= \sum_{h=1}^H\sum_{k=1}^K \mathbb{E}_{\pi^{\tilde{k}}}[(Q_h^{\tilde{k}}-\cT_h Q_{h+1}^{\tilde{k}})(s_h,a_h)]\\
&\leq\sum_{h=1}^H\sum_{k=1}^K (Q_h^{\tilde{k}}-\cT_h Q_{h+1}^{\tilde{k}})(s_h^k,a_h^k)+O(H\sqrt{HK\log (HK/\delta)})
\end{align*}
}
where the last inequality follows from standard martingale concentration. Here $(Q_h^{\tilde{k}}-\cT_h Q_{h+1}^{\tilde{k}})(s_h^k,a_h^k)$ plays similar role of ``bonus''. The next core step is to bound $\|Q_h^{\tilde{k}}-\cT_h Q_{h+1}^{\tilde{k}}\|^2_{\cZ_h^k}$, i.e., we want to show that the estimated Q-value in the $\tilde{k}$-th episode still has low empirical bellman error in the k-th episode. The constraints on the confidence set can only guarantee $\|Q_h^{\tilde{k}}-\cT_h Q_{h+1}^{\tilde{k}}\|^2_{\cZ_h^{\tilde{k}}}\leq O(\beta)$, where $\beta$ is the parameter in Algorithm~\ref{alg:plansample}. The sub-sampled dataset builds a bridge:
\[
\|Q_h^{\tilde{k}}-\cT_h Q_{h+1}^{\tilde{k}}\|^2_{\cZ_h^k}\lesssim \|Q_h^{\tilde{k}}-\cT_h Q_{h+1}^{\tilde{k}}\|_{\widehat{\cZ}_h^k}^2 = \|Q_h^{\tilde{k}}-\cT_h Q_{h+1}^{\tilde{k}}\|_{\widehat{\cZ}_h^{\tilde{k}}}^2\lesssim \|Q_h^{\tilde{k}}-\cT_h Q_{h+1}^{\tilde{k}}\|^2_{\cZ_h^{\tilde{k}}}\leq O(\beta)
\]
Here the first and the third inequality follows from Theorem~\ref{thm:sample}. Then we can bound the sum of ``bonuses'' in terms of the Bellman eluder dimension using techniques in \citet{jin2021bellman, russo2013eluder}.

\subsubsection{Formal Proof}
We first provide the concentration bound.
\begin{lem}[Lemma 24 in \citet{jin2021bellman}]
\label{lem:jin1}
With probability at least $1-\delta$, for all $(k,h)\in[K]\times[H]$,
\[
\|Q_h^k-\cT_h Q_{h+1}^k\|^2_{\cZ_h^k}\leq O(\beta)
\]
\end{lem}

We use $\tilde{k}$ to denote the index of the policy used in the $k$-th episode. Our switching criteria naturally indicates that $\widehat{\cZ}_h^k=\widehat{\cZ}_h^{\tilde{k}}$. The next lemma shows that the estimated Q-value in the $\tilde{k}$-th episode still has low empirical bellman error with respect to the data collected in the k-th episode.
\begin{lem}
\label{lem:jin3}
With probability at least $1-\delta$, for all
$(k,h)\in[K]\times[H]$,
\begin{itemize}
\item (a) $\|Q_h^{\tilde{k}}-\cT_h Q_{h+1}^{\tilde{k}}\|_{\cZ_h^k}^2\leq O(\beta)$
\item (b) $\sum_{i=1}^{k-1} \mathbb{E}_{\pi^{\tilde{i}}}[(Q_h^{\tilde{k}}-\cT_h Q_{h+1}^{\tilde{k}})(s_h,a_h)]^2\leq O(\beta)$
\end{itemize}
\end{lem}
\begin{proof}
Condition on the event defined in theorem~\ref{thm:sample} and the event defined in Lemma~\ref{lem:jin1}. Then we have that
\[
\|Q_h^{\tilde{k}}-\cT_h Q_{h+1}^{\tilde{k}}\|^2_{\cZ_h^k}\lesssim \|Q_h^{\tilde{k}}-\cT_h Q_{h+1}^{\tilde{k}}\|_{\widehat{\cZ}_h^k}^2 = \|Q_h^{\tilde{k}}-\cT_h Q_{h+1}^{\tilde{k}}\|_{\widehat{\cZ}_h^{\tilde{k}}}^2\lesssim \|Q_h^{\tilde{k}}-\cT_h Q_{h+1}^{\tilde{k}}\|^2_{\cZ_h^k}\leq O(\beta).
\]
Here the first and the third inequality follows from Theorem~\ref{thm:sample}, and the last inequality follows the result of Lemma~\ref{lem:jin1}. (b) can be derived directly from (a) by standard martingale concentration.
\end{proof}
\begin{lem}[Optimism, lemma 25 in \citet{jin2021bellman}]
\label{lem:jin2}
With probability at least $1-\delta$, for all $k\in[K]$,
\[
(Q_1^*,Q_2^*,...,Q_{H+1}^*)\in \cB^k.
\]
As a result, $V_1^*(s_1)\leq \max_{a\in\cA}Q_1^k(s_1,a)$ holds for all $k\in[K]$.
\end{lem}
\begin{lem}[Regret Decomposition]
\label{lem:decomp}
With probability at least $1-\delta$,
\[
\operatorname{Regret}(K)\leq \sum_{h=1}^H\sum_{k=1}^K \mathbb{E}_{\pi^{\tilde{k}}}[(Q_h^{\tilde{k}}-\cT_h Q_{h+1}^{\tilde{k}})(s_h,a_h)],
\]
and
\[
\operatorname{Regret}(K)\leq \sum_{h=1}^H\sum_{k=1}^K (Q_h^{\tilde{k}}-\cT_h Q_{h+1}^{\tilde{k}})(s_h^k,a_h^k)+O(H\sqrt{HK\log (HK/\delta)}).
\]
\end{lem}
\begin{proof}
Condition on the event in Lemma~\ref{lem:jin2}. Then we have that
\[
\text{Regret}(K)\leq \sum_{k=1}^K\left(\max_{a\in A}Q_1^{\tilde{k}}(s_1,a)-V_1^{\pi^{\tilde{k}}}(s_1)\right)= \sum_{h=1}^H\sum_{k=1}^K \mathbb{E}_{\pi^{\tilde{k}}}[(Q_h^{\tilde{k}}-\cT_h Q_{h+1}^{\tilde{k}})(s_h,a_h)].
\]
The second inequality then follows from standard martingale concentration.
\end{proof}
\begin{lem}[Lemma 26 in \citet{jin2021bellman}]
\label{lem:jin4}
For a fixed $h\in[H]$, suppose that $\{(f_k,f_k')\}_{k=1}^K\subseteq \cF\times\cF$, and $\{\mu_k\}_{k=1}^K\subseteq \Pi_h$ satisfy that for all $k\in[K]$, $\sum_{t=1}^{k-1} \mathbb{E}_{\mu_t}[(f_k-\cT_h f_k')^2]\leq O(\beta)$. Then we have that 
\[
\sum_{k=1}^K |\mathbb{E}_{\mu_k}[f_k-\cT_h f_k']|\leq O(\sqrt{\dim_{DE}((I-\cT_h)\cF,\Pi_h,1/\sqrt{K})\beta K}+H\cdot \min\{K,\dim_{DE}((I-\cT_h)\cF,\Pi_h,1/\sqrt{K})\})
\]
\end{lem}
\begin{proof}[Proof of Theorem~\ref{thm:main_regret2}]
Condition on the event defined in Lemma~\ref{lem:jin2}.
\paragraph{Case 1: $\dim_{BE}(\cF,\varepsilon)=\max_{h\in[H]}\dim_{DE}((I-\cT_h)\cF,\cD_{\Delta},\varepsilon)$.} We use Lemma~\ref{lem:jin4} with 
\[
f_k=Q_h^{\tilde{k}},\quad f_k'=Q_{h+1}^{\tilde{k}},\quad
\mu_k=\mathbf{1}\{\cdot=(s_h^k,a_h^k)\}.
\]
Then we have that
\begin{align*}
\sum_{k=1}^K (Q_h^{\tilde{k}}-\cT_h Q_{h+1}^{\tilde{k}})(s_h^k,a_h^k)&\leq O\left(\sqrt{\dim_{DE}((I-\cT_h)\cF,\cD_{\Delta},1/\sqrt{K})\beta K}\right)\\
&\leq O\left(\sqrt{KH^2\cdot\log(T\mathcal{N}(\mathcal{F},1/K)/\delta)\cdot \dim_{BE}(\mathcal{F},1/\sqrt{K})}\right)
\end{align*}
Plugging this result into Lemma~\ref{lem:decomp} completes the proof.

\paragraph{Case 2: $\dim_{BE}(\cF,\varepsilon)=\max_{h\in[H]}\dim_{DE}((I-\cT_h)\cF,\cD_{\cF,h},\varepsilon)$.} We use Lemma~\ref{lem:jin4} with 
\[
f_k=Q_h^{\tilde{k}},\quad f_k'=Q_{h+1}^{\tilde{k}},\quad
\mu_k=\mathbb{P}^{\pi^{\tilde{k}}}\{s_h=\cdot,a_h=\cdot\}.
\]
Then we have that
\begin{align*}
\sum_{k=1}^K \mathbb{E}_{\pi^{\tilde{k}}}[(Q_h^{\tilde{k}}-\cT_h Q_{h+1}^{\tilde{k}})(s_h,a_h)]&\leq O\left(\sqrt{\dim_{DE}((I-\cT_h)\cF,\cD_{\cF,h},1/\sqrt{K})\beta K}\right)\\
&\leq O\left(\sqrt{KH^2\cdot\log(T\mathcal{N}(\mathcal{F},1/K)/\delta)\cdot \dim_{BE}(\mathcal{F},1/\sqrt{K})}\right)
\end{align*}
Plugging this result into Lemma~\ref{lem:decomp} completes the proof.
\end{proof}
\subsection{Proof of Theorem~\ref{thm:main_free}}
\label{pf:thm3}

Firstly, note that the online sub-sampling procedure used in Algorithm~\ref{alg:exploration} is exactly the same with Algorithm~\ref{alg:main}. Thus the properties of online sub-sampling, i.e., Proposition~\ref{prop:bounds_of_bonus} and Proposition~\ref{prop:bounded_size} still hold. As a special case, Proposition~\ref{prop:bounds_of_bonus} and Proposition~\ref{prop:bounded_size} also imply the corresponding results for $b_h$ and $\widehat{\cZ}_h$ used in Algorithm~\ref{alg:planning}.
\begin{lem}
\label{lem:bonus_class_2}
With probability at least $1-\delta/8$, for all $(h,k)\in[H]\times[K]$, $b_h^k(\cdot,\cdot)\in\mathcal{M}$.

Here $\mathcal{M}$ is a prespecified function class with bounded size:
\begin{align*}
&\log|\mathcal{M}|\\
\leq &C'\cdot\log(T\mathcal{N}(\mathcal{F},\sqrt{\delta/64T^3})/\delta) \cdot\dim_E(\mathcal{F},1/T)\cdot\log^2 T\cdot\log\left(\mathcal{C}(\mathcal{S}\times\mathcal{A},1/(16\sqrt{64T^3/\delta})) \cdot64T^3/\delta\right)\\
\leq &C\cdot\log(T\mathcal{N}(\mathcal{F},\delta/T^2)/\delta) \cdot\dim_E(\mathcal{F},1/T)\cdot\log^2 T\cdot\log\left(\mathcal{C}(\mathcal{S}\times\mathcal{A},\delta/T^2) \cdot T/\delta\right).
\end{align*}
for some absolute constant $C',C>0$ if $T$ is sufficiently large.

As a special case, $b_h(\cdot,\cdot)\in\mathcal{M}$ as well.
\end{lem}
\begin{proof}
The proof is identical to Lemma~\ref{lem:bonus_class}.
\end{proof}
The next lemma estimate the error of the one-step bellman backup.
\begin{lem}
\label{lem:one_step_error_2}
Consider a fixed pair $(k,h)\in[K]\times[H]$. 
For any $V:\cS\rightarrow [0,H]$, define
	\[
	\mathcal{D}_h^k(V):=\{(s_{h}^{\tau},a_{h}^{\tau}, V(s_{h+1}^{\tau}))\}_{\tau \in [k-1]}
	\]
	and also
	\[
	\widehat{f_V}:=\operatorname{argmin}_{f\in\mathcal{F}}\|f\|_{\mathcal{D}_h^k(V)}^2.
	\]
	For any $V:\cS\rightarrow[0,H]$ and $\delta \in (0,1)$, there is an event $\mathcal{E}_{V,\delta}$ which holds with probability at least $1-\delta$, such that for any $V':\mathcal{S}\rightarrow[0,H]$ with $\|V'-V\|_{\infty}\leq2/T$, we have
	\[
	\left\|\widehat{f_{V'}}(\cdot,\cdot)-\sum_{s'\in\mathcal{S}}P_h(s'|\cdot,\cdot)V'(s')\right\|_{\mathcal{Z}_h^k}\leq C\cdot(H\sqrt{\log(1/\delta)+\log\mathcal{N}(\mathcal{F},1/T)}).
	\]
	for some absolute constant $C>0$
\end{lem}
\begin{proof} A key observation is that $\sum_{s'\in\mathcal{S}}P_h(s'|\cdot,\cdot)V'(s') \in \mathcal{F}$ due to Assumption~\ref{assum:express_2}. The rest of the proof is identical to Lemma~\ref{lem:one_step_error}.
\end{proof}
The next lemma verifies the confidence region.
\begin{lem}
\label{lem:confidence_region_2}
Let $\mathcal{E}_2$ denote the event that for all $(k,h)\in[K]\times[H]$,
	\[
	\left\|f_h^k-\bar{f}_h^k \right\|_{\mathcal{Z}_h^k}\leq \beta/100
	\]
	where $\bar{f}_h^k(\cdot,\cdot)=\sum_{s\in\mathcal{S}'}P_h(s'|\cdot,\cdot)V_{h+1}^k(s')$.
	
	For the planning phase, for all $h\in[H]$, all reward function $r$ in the function class $\cR$
	\[
	\left\|f_h-\bar{f}_h \right\|_{\mathcal{Z}_h}\leq \beta/100
	\]
	where 
	$
	\bar{f}_h=\sum_{s\in\mathcal{S}'}P_h(s'|\cdot,\cdot)V_{h+1}(s')
	$.
	
	Then we have $\Pr[\mathcal{E}_2]\geq 1-\delta/4$ provided
	\[
	\beta\geq C\cdot H^2\cdot(\log(T/\delta)+\log\mathcal{N}(\mathcal{F},1/T)+\log\mathcal{N}(\mathcal{R},1/T)+\log|\mathcal{M}|).
	\]
	for some absolute constant $C>0$.
\end{lem}
\begin{proof} We condition on the event defined in Lemma~\ref{lem:bonus_class_2} in the whole proof.

Note that for all $(k,h)\in[K]\times[H]$,
\[
Q_h^k(\cdot,\cdot)=\min\{f_h^k(\cdot,\cdot)+b_h^k(\cdot,\cdot)+\Pi_{[0,1]}[b_h^k(\cdot,\cdot)/H],H\},\text{and}
\]
\[
V_h^k(\cdot)=\max_{a\in \mathcal{A}}Q_h^k(\cdot,a).
\]
We define
\[
\mathcal{Q}:=\{\min\{f(\cdot,\cdot)+m(\cdot,\cdot)+\Pi_{[0,1]}[m(\cdot,\cdot)/H] ,H\}|f\in\mathcal{C}(\mathcal{F},1/T),m\in \mathcal{M}\}\cup\{0\}, \text{and}
\]
\[
\mathcal{V}:=\{\max_{a\in\mathcal{A}}q(\cdot,a)|q\in\mathcal{Q}\}.
\]
Then $\log|\mathcal{V}|\leq \log|\mathcal{M}|+\log\mathcal{N}(\mathcal{F},1/T)+1$.

Because $b_h^k(\cdot,\cdot)\in\cM$, $\mathcal{V}$ is a $(1/T)$-cover of $V_h^k(\cdot)$ for all $(k,h)\in[K]\times[H+1]$. 
For each $V\in \cV$, let $\cE_{V,\delta/(16|\cV|T)}$ be the event defined in Lemma~\ref{lem:one_step_error_2}. Note that $\cE_{V,\delta/(16|\cV|T)}$ relates to a fixed pair $(k,h)$. 
By Lemma~\ref{lem:one_step_error_2} and a union bound, we have
$\Pr\left[\bigcap_{(k,h)\in[K]\times[H]}\bigcap_{V\in \cV} \cE_{V,\delta/(16|\cV|T)}
\right]\ge 1-\delta/16$.
We also condition on this event.

For $(k,h)\in[K]\times[H]$, recall that $f_{h}^{k}$ is the solution to the regression problem in Algorithm~\ref{alg:exploration}, i.e., 
$f^k_{h}= \argmin_{f\in \cF}\|f\|_{\cD^{k}_h}^2$.
Let $V\in \cV$ such that $\|V - V^{k}_{h+1}\|_{\infty} \le 1/T$.
By the definition of $\cE_{V,\delta/(16|\cV|T)}$ (the one relates to this $(k,h)$ pair), we have that
\begin{align*}
&\left\|f_h^k(\cdot,\cdot)-\sum_{s'\in\mathcal{S}}P_h(s'|\cdot,\cdot)V_{h+1}^k(s')\right\|_{\mathcal{Z}_h^k}\\
\lesssim &H\sqrt{\log(16|\cV|T/\delta)+\log\mathcal{N}(\mathcal{F},1/T)}\\
\lesssim &H\sqrt{\log(T/\delta)+\log\mathcal{N}(\mathcal{F},1/T)+\log|\mathcal{M}|}.
\end{align*}

For the planning phase, we define a new function class:
\[
\mathcal{F}^*:=\{f(\cdot,\cdot)+r(\cdot,\cdot)|f\in\mathcal{F},r\in\mathcal{R}\}.
\]
Then from the definition of covering number we have that:
\[
\mathcal{C}(\mathcal{F}^*,2/T)\leq \mathcal{C}(\mathcal{F},1/T)\mathcal{C}(\mathcal{R},1/T).
\]
Thus for all $h\in[H]$, \begin{align*}
Q_h(\cdot,\cdot)&=\min\{f_h(\cdot,\cdot)+b_h(\cdot,\cdot)+r_h(\cdot,\cdot),H\}\\
&=\min\{f^*(\cdot,\cdot)+b_h(\cdot,\cdot),H\}\quad (f^*\in\mathcal{F}^*),\text{and}
\end{align*}
\[
V_h(\cdot)=\max_{a\in \mathcal{A}}Q_h(\cdot,a).
\]
We define
\[
\mathcal{Q}^*:=\{\min\{f^*(\cdot,\cdot)+m(\cdot,\cdot),H\}|f\in\mathcal{C}(\mathcal{F}^*,2/T),m\in \mathcal{M}\}\cup\{0\},\text{and}
\]
\[
\mathcal{V}^*:=\{\max_{a\in\mathcal{A}}q(\cdot,a)|q\in\mathcal{Q}^*\}.
\]
Then we have
\begin{align*}
\log|\mathcal{V}^*|&\leq \log|\mathcal{M}|+\log\mathcal{N}(\mathcal{F}^*,2/T)+1 \\
&\leq \log|\mathcal{M}|+\log\mathcal{N}(\mathcal{F},1/T)+\log\mathcal{N}(\mathcal{R},1/T)+1.
\end{align*}

Because $b_h(\cdot,\cdot)\in\cM$, $\mathcal{V}^*$ is a $(2/T)$-cover of $V_h(\cdot)$ for all $h\in[H+1]$.
Similarly, for each $V\in \cV^*$, let $\cE_{V,\delta/(16|\cV^*|T)}$ be the event defined in Lemma~\ref{lem:one_step_error_2}. Note that $\cE_{V,\delta/(16|\cV^*|T)}$ relates to a fixed pair $(k,h)\in[K]\times[H]$. We only consider those with $k=K$.
By Lemma~\ref{lem:one_step_error_2} and a union bound, we have
$\Pr\left[\bigcap_{h\in[H],k=K}\bigcap_{V\in \cV} \cE_{V,\delta/(16|\cV^*|T)}
\right]\ge 1-\delta/16K\geq 1-\delta/16$.
We also condition on this event.

For $h\in[H]$, recall that $f_{h}$ is the solution to the regression problem in Algorithm~\ref{alg:planning}, i.e., 
$f_{h}= \argmin_{f\in \cF}\|f\|_{\cD_h}^2$.
Let $V\in \cV^*$ such that $\|V - V_{h+1}\|_{\infty} \le 1/T$.
By the definition of $\cE_{V,\delta/(16|\cV^*|T)}$ (the one relates to this $h$ and $k=K$, note that $\cZ_h^K=\cZ_h$), we have that
\begin{align*}
&\left\|f_h(\cdot,\cdot)-\sum_{s'\in\mathcal{S}}P_h(s'|\cdot,\cdot)V_{h+1}(s')\right\|_{\mathcal{Z}_h}\\
\lesssim &H\sqrt{\log(16|\cV^*|T/\delta)+\log\mathcal{N}(\mathcal{F},1/T)}\\
\lesssim &H\sqrt{\log(T/\delta)+\log\mathcal{N}(\mathcal{F},1/T)+\log\mathcal{N}(\mathcal{R},1/T)+\log|\mathcal{M}|}.
\end{align*}
Combining the above two parts we complete the proof.
\end{proof}
Finally, we use the above result to show optimism, in both the exploration and planning phase.
\begin{lem}
\label{lem:optimism_2}
Let $\mathcal{E}_3$ denote the event that for all $(k,h)\in [K]\times[H]$, and all $(s,a)\in\mathcal{S}\times\mathcal{A}$,
\[
Q^*_h(s,a,r^k)\leq Q_h^k(s,a) \leq \bar{f}_h^k(s,a)+r_h^k(\cdot,\cdot)+2b_h^k(s,a)
\]
where
$\bar{f}_h^k(\cdot,\cdot)=\sum_{s\in\mathcal{S}'}P_h(s'|\cdot,\cdot)V_{h+1}^k(s')$.\\
And for the planning phase, for all $h\in[H]$, all $(s,a)\in\mathcal{S}\times\mathcal{A}$, and all reward function $r$ in the function class $\cR$,
\[
Q^*_h(s,a,r)\leq Q_h(s,a) \leq \bar{f}_h(s,a)+r_h(\cdot,\cdot)+2b_h(s,a)
\]
where
$\bar{f}_h(\cdot,\cdot)=\sum_{s\in\mathcal{S}'}P_h(s'|\cdot,\cdot)V_{h+1}(s')$.\\
Then $\Pr[\mathcal{E}_3]\geq1-3\delta/8$.
\end{lem}
\begin{proof} We condition on the event defined in Proposition~\ref{prop:bounds_of_bonus} and $\mathcal{E}_2$ defined in Lemma~\ref{lem:confidence_region_2}. Because $\left\|f_h^k-\bar{f}_h^k \right\|_{\mathcal{Z}_h^k}\leq \beta/100$, from the definition of $\underline{b}_h^k$ we have that $\left| \bar{f}_h^k(\cdot,\cdot)-f_h^k(\cdot,\cdot)\right|\leq \underline{b}_h^k(\cdot,\cdot)$. Moreover, by Proposition~\ref{prop:bounds_of_bonus} we have $\underline{b}_h^k(\cdot,\cdot) \leq b_h^k(\cdot,\cdot)$.
Thus for all
$ (k,h)\in[K]\times[H]$, $(s,a)\in \mathcal{S}\times\mathcal{A}$, we have
\begin{align*}
Q_h^k(s,a)&=\min\{f_h^k(s,a)+r_h^k(s,a)+b_h^k(s,a),H\}\\
&\leq \min\left\lbrace \bar{f}_h^k(s,a)+r_h^k(s,a)+b_h^k(s,a)+\left| \bar{f}_h^k(s,a)-f_h^k(s,a)\right|,H\right\rbrace  \\
&\leq  \min\left\lbrace \bar{f}_h^k(s,a)+r_h^k(s,a)+b_h^k(s,a)+\underline{b}_h^k(s,a),H\right\rbrace \\
&\leq  \min\left\lbrace \bar{f}_h^k(s,a)+r_h^k(s,a)+2b_h^k(s,a),H\right\rbrace \\
&\leq   \bar{f}_h^k(s,a)+r_h^k(s,a)+2b_h^k(s,a).
\end{align*}
Next we use induction on $h$ to prove $Q^*_h(\cdot,\cdot,r^k)\leq Q_h^k(\cdot,\cdot)$. The inequality clearly holds when $h=H+1$. Now we assume $Q^*_{h+1}(\cdot,\cdot,r^k)\leq Q_{h+1}^k(\cdot,\cdot)$ for some $h\in[H]$. Then obviously we have $V^*_{h+1}(\cdot,r^k)\leq V_{h+1}^k(\cdot)$. Therefore for all $(s,a)\in\mathcal{S}\times\mathcal{A}$,
\begin{align*}
Q^*_h(s,a,r^k)&=  r_h^k(s,a)+\sum_{s'\in \mathcal{S}}P_h(s'|s,a)V_{h+1}^*(s',r^k)\\
&\leq \min\left\lbrace  r_h^k(s,a)+\sum_{s'\in \mathcal{S}}P_h(s'|s,a)V_{h+1}^k(s'),H\right\rbrace \\
&= \min\left\lbrace  r_h^k(s,a)+\bar{f}_h^k(s,a),H\right\rbrace\\
&\leq\min\left\lbrace  r_h^k(s,a)+f_h^k(s,a)+b_h^k(s,a),H\right\rbrace\\
&=Q^k_h(s,a).
\end{align*}
The proof of the second inequality is identical. One only need to discard the superscript $k$ in the above argument.
\end{proof}
\begin{lem}
\label{lem:bonus_sum_2}
	With probability at least $1-\delta/32$,
	\[
	\sum_{k=1}^K\sum_{h=1}^H b_h^k(s_h^k,a_h^k)\leq 
	H+H(H+1)\dim_E(\mathcal{F},1/T)+C\cdot\sqrt{\dim_E(\mathcal{F},1/T)\cdot TH\cdot\beta}
	\]
	for some absolute constant $C>0$.
\end{lem}
\begin{proof} The proof is identical to Lemma~\ref{lem:bonus_sum}.
\end{proof}
The next lemma bounds the summation of the optimistic value functions in the exploration phase. The techniques are similar to the standard regret decomposition for optimistic algorithms.
\begin{lem}
\label{lem:sum_value}
With probability at least $1-\delta/2$, 
\[
\sum_{k=1}^KV_1^{\tilde{k}}(s_1)= O(\sqrt{ T\cdot H^3\cdot\iota_1})
\]
where 
\begin{align*}
\iota_1=\log(\mathcal{N}(\mathcal{R},1/T))&\cdot\dim_E(\mathcal{F},1/T)\\
+\log(T\mathcal{N}(\mathcal{F},\delta/T^2)/\delta)\cdot \dim^2_E(\mathcal{F},1/T)&\cdot\log^2 T\cdot\log\left(\mathcal{N}(\mathcal{S}\times\mathcal{A},\delta/T^2) \cdot T/\delta\right).
\end{align*}
\end{lem}
\begin{proof}
For all $(k,h)\in[K]\times[H-1]$, denote
\[
\xi_h^k=\sum_{s'\in\mathcal{S}}P_h(s'|s_h^k,a_h^k)V^{\tilde{k}}_{h+1}(s')-V_{h+1}^{\tilde{k}}(s_{h+1}^k).
\]
If $\mathcal{E}_3$ defined in Lemma~\ref{lem:optimism_2} happens, we have
\begin{align*}
\sum_{k=1}^KV_1^{\tilde{k}}(s_1)&= \sum_{k=1}^KV_1^{\tilde{k}}(s_1^k)\\
&=\sum_{k=1}^KQ_1^{\tilde{k}}(s_1^k,a_1^k)\\
&\leq \sum_{k=1}^K\left( r_1^{\tilde{k}}(s_1^k,a_1^k)+\sum_{s'\in\mathcal{S}}P_1(s'|s_1^k,a_1^k
)V_2^{\tilde{k}}(s')+2b_1^{\tilde{k}}(s_1^k,a_1^k)\right) \\
&\leq  \sum_{k=1}^K\left( \sum_{s'\in\mathcal{S}}P_1(s'|s_1^k,a_1^k
)V_2^{\tilde{k}}(s')+(2+1/H)b_1^k(s_1^k,a_1^k)\right) \\
&=\sum_{k=1}^K\left( V_2^{\tilde{k}}(s_2^k)+\xi_1^k+(2+1/H)b_1^k(s_1^k,a_1^k)\right) \\
&\leq...\\
&\leq \sum_{k=1}^K\sum_{h=1}^{H-1}\xi_h^k+\sum_{k=1}^K\sum_{h=1}^{H}(2+1/H)b_h^k(s_h^k,a_h^k).
\end{align*}
Note that $\{\xi_h^k\}_{(k,h)\in[K]\times[H]}$ (arranged in lexicographical order) is a martingale difference sequence with $|\xi_h^k|\leq H$. By Azuma-Hoeffding inequality, we have
\[
\Pr\left\lbrace\left| \sum_{k=1}^K\sum_{h=1}^{H-1}\xi_h^k \right| \leq C'\cdot\sqrt{KH^3\log(1/\delta)}\right\rbrace \geq 1-\delta/16.
\]
for some absolute constant $C'>0$.
Conditioned on the above event, the event defined in Lemma~\ref{lem:bonus_sum_2}, and $\mathcal{E}_3$ defined in Lemma~\ref{lem:optimism_2} we conclude that with probability at least $1-\delta/2$,
\begin{align*}
&\sum_{k=1}^KV_1^{\tilde{k}}(s_1)\\
\leq& \sum_{k=1}^K\sum_{h=1}^{H-1}\xi_h^k+\sum_{k=1}^K\sum_{h=1}^{H}(2+1/H)b_h^k(s_h^k,a_h^k)\\
=&C'\cdot\sqrt{KH^3\log(1/\delta)}+(2+1/H)\cdot\left(  H+H(H+1)\dim_E(\mathcal{F},1/T)+C\cdot\sqrt{\dim_E(\mathcal{F},1/T)\cdot TH\cdot\beta}\right) \\
\lesssim&\sqrt{T\cdot H^3\cdot \iota_1}
\end{align*}
as desired.
\end{proof}
The next lemma bounds the error of the planning policy in terms of the expectation of the bonus functions.
\begin{lem}
\label{lem:bound_suboptimal}
If $\cE_3$ defined in Lemma~\ref{lem:optimism_2} (optimism) happens, then for all reward function $r$ in the function class $\cR$,
\[ 
V_1^*(s_1,r)-V_1^\pi(s_1,r) \leq  2HV_1^*(s_1,\Pi_{[0,1]}[b/H]).
\]
Here $b$ is the bonus function computed during the planning phase, and $\pi$ is the output policy in the planning phase.

\end{lem}
\begin{proof} 
We generalize the lemma and use induction on $h$ to prove that for any $h\in[H+1]$,
\[
V_h(s)-V_h^\pi(s,r)-  2HV_h^\pi(s,\Pi_{[0,1]}[b/H])\leq0 \quad \forall s\in\mathcal{S}.
\]
The result is obvious for $h=H+1$. Suppose for some $h\in[H]$, it holds that
\[
V_{h+1}(s)-V_{h+1}^\pi(s,r)-  2HV_{h+1}^\pi(s,\Pi_{[0,1]}[b/H])\leq0 \quad \forall s\in\mathcal{S}.
\]
Then for all $s\in\mathcal{S}$, we have
\begin{align*}
&V_h(s)-V_h^\pi(s,r)-  2HV_h^\pi(s,\Pi_{[0,1]}[b/H])\\
=&Q_h(s,\pi(s))-Q_h^\pi(s,\pi(s),r)-  2HQ_h^\pi(s,\pi(s),\Pi_{[0,1]}[b/H])\\
\leq& \left( r_h(s,\pi(s)) + \sum_{s'\in \mathcal{S}}P_h(s'|s,\pi(s))V_{h+1}(s')+2\Pi_{[0,H]}[b_h(s,\pi(s))]\right)\\ 
-&\left(r_h(s,\pi(s)) + \sum_{s'\in \mathcal{S}}P_h(s'|s,\pi(s))V_{h+1}^\pi(s',r) \right)\\ -&2H\left(\Pi_{[0,1]}[b_h(s,\pi(s))/H] + \sum_{s'\in \mathcal{S}}P_h(s'|s,\pi(s))V_{h+1}(s',\Pi_{[0,1]}[b/H]) \right)\\
=& \sum_{s'\in \mathcal{S}}P_h(s'|s,\pi(s))\left(V_{h+1}(s')-V_{h+1}^\pi(s',r)-2HV_{h+1}^\pi(s',\Pi_{[0,1]}[b/H]) \right) \\
\leq&0.
\end{align*}
as desired.

By taking $h=1$ and $s=s_1$ as a special case, we have that
\[
V_1(s_1)-V_1^\pi(s_1,r)-  2HV_1^\pi(s_1,\Pi_{[0,1]}[b/H])\leq0.
\]
Then we conclude
\[
V_1^*(s_1,r)-V_1^\pi(s_1,r)\leq V_1(s_1,r)-V_1^\pi(s_1,r)\leq2HV_1^{\pi}(s_1,\Pi_{[0,1]}[b/H])\leq2HV_1^*(s_1,\Pi_{[0,1]}[b/H]).
\]
\end{proof}
\begin{lem}
\label{lem:bonud_suboptimal_2}
With probability at least $1-7\delta/8$, for all reward function $r$ in the function class $\cR$,
\[
V_1^*(s_1,r)-V_1^\pi(s_1,r)= O( H^3\cdot\sqrt{\iota_1 /K})
\]
where
\begin{align*}
\iota_1=\log(\mathcal{N}(\mathcal{R},1/T))&\cdot\dim_E(\mathcal{F},1/T)\\
+\log(T\mathcal{N}(\mathcal{F},\delta/T^2)/\delta)\cdot \dim^2_E(\mathcal{F},1/T)&\cdot\log^2 T\cdot\log\left(\mathcal{N}(\mathcal{S}\times\mathcal{A},\delta/T^2) \cdot T/\delta\right).
\end{align*}
\end{lem}
\begin{proof}
We condition on $\cE_3$ defined in Lemma~\ref{lem:optimism_2} and the event defined in Lemma~\ref{lem:sum_value}.
By Lemma~\ref{lem:bound_suboptimal}, we have
$$
V_1^*(s_1,r)-V_1^\pi(s_1,r)\leq 
2HV_1^*(s_1,\Pi_{[0,1]}[b/H]).
$$
By the monotonicity of the bonus function, we have that
$$
\Pi_{[0,1]}[b_h(\cdot,\cdot)/H]\leq\Pi_{[0,1]}[b_h^k(\cdot,\cdot)/H] = r^k_h(\cdot,\cdot) \quad \forall (k,h) \in [K]\times[H].
$$
Then the right hand side can be bounded in the following manner:
$$
2HV_1^*(s_1,\Pi_{[0,1]}[b/H])\leq2\frac{H}{K} \sum_{k=1}^K V_1^*(s_1,r^{\tilde{k}})\leq 2\frac{H}{K}\sum_{k=1}^K V_1^{\tilde{k}}(s_1).
$$
Substituting the bound for $\sum_{k=1}^K V_1^{\tilde{k}}(s_1)$ completes the proof.
\end{proof}
\begin{proof}[Proof of Theorem~\ref{thm:main_free}] Simply combing Proposition~\ref{prop:bounded_size} and Lemma~\ref{lem:bonud_suboptimal_2} with a union bound completes the proof.
\end{proof}
\subsection{Proof of Theorem~\ref{thm:MG}}
We first provide the concentration bound. Similar to the single-agent setting, we define the original dataset as
\[
\mathcal{Z}_h^k:=\{(s_h^{\tau},a_h^{\tau},b_h^{\tau})\}_{\tau\in[k-1]}
\]
The proof of the following lemma is identical to that of Lemma~\ref{lem:jin1}
\begin{lem}
\label{lem:jin11}
With probability at least $1-\delta$, for all $(k,h)\in[K]\times[H]$,
\[
\|Q_h^k-\cT_h Q_{h+1}^k\|^2_{\cZ_h^k}\leq O(\beta)
\]
\end{lem}

Similar to the single-agent setting, we use $\tilde{k}$ to denote the index of the policy used in the $k$-th episode. Our switching criteria indicates that $\widehat{\cZ}_h^k=\widehat{\cZ}_h^{\tilde{k}}$. The next lemma shows that the estimated Q-value in the $\tilde{k}$-th episode still has low empirical bellman error with respect to the data collected in the k-th episode.
\begin{lem}
\label{lem:jin31}
With probability at least $1-2\delta$, for all
$(k,h)\in[K]\times[H]$,
\[\|Q_h^{\tilde{k}}-\cT_h Q_{h+1}^{\tilde{k}}\|_{\cZ_h^k}^2\leq O(\beta)\]
\end{lem}
\begin{proof}
Condition on the event defined in Theorem~\ref{thm:sample} and the event defined in Lemma~\ref{lem:jin11}. Then we have that
\[
\|Q_h^{\tilde{k}}-\cT_h Q_{h+1}^{\tilde{k}}\|^2_{\cZ_h^k}\lesssim \|Q_h^{\tilde{k}}-\cT_h Q_{h+1}^{\tilde{k}}\|_{\widehat{\cZ}_h^k}^2 = \|Q_h^{\tilde{k}}-\cT_h Q_{h+1}^{\tilde{k}}\|_{\widehat{\cZ}_h^{\tilde{k}}}^2\lesssim \|Q_h^{\tilde{k}}-\cT_h Q_{h+1}^{\tilde{k}}\|^2_{\cZ_h^k}\leq O(\beta).
\]
Here the first and the third inequality follows from Theorem~\ref{thm:sample}, and the last inequality follows the result of Lemma~\ref{lem:jin1}. 
\end{proof}
The following lemma shows optimism. 
\begin{lem}
\label{lem:jin21}
With probability at least $1-\delta$, for all $k\in[K]$,
\[
(Q_1^*,Q_2^*,...,Q_{H+1}^*)\in \cB^k.
\]
As a result, $V_1^*(s_1)\leq V_{Q_1^k}(s_1)$ holds for all $k\in[K]$.
\end{lem}
\begin{proof}
The proof of the first part is identical to that of Lemma~\ref{lem:jin2}. To see the second part, note that
\[
V_1^*(s_1)=V_{Q_1^*}(s_1)\leq V_{Q_1^k}(s_1)
\]
\end{proof}
The next lemma is a special case of Lemma~\ref{lem:jin4}.
\begin{lem}
\label{lem:jin41}
For a fixed $h\in[H]$, suppose that $\{(f_k,f_k')\}_{k=1}^K\subseteq \cF\times\cF$, and $\{z_k\}_{k=1}^K\subseteq \cS\times\cA\times\cB$ satisfy that for all $k\in[K]$, $\sum_{t=1}^{k-1} (f_k-\cT_h f_k')(z_t)^2\leq O(\beta)$. Then we have that 
\[
\sum_{k=1}^K |(f_k-\cT_h f_k')(z_k)|\leq O\left(\sqrt{\dim_{OBE}(\cF,1/\sqrt{K})\beta K}+H\cdot \min\{K,\dim_{OBE}(\cF,1/\sqrt{K})\}\right)
\]
\end{lem}
The next lemma decomposes the regret in this setting.
\begin{lem}[Regret Decomposition]
\label{lem:decomp_MG}
With probability at least $1-2\delta$,
\[
\operatorname{NashRegret}(K)\leq \sum_{k=1}^K\sum_{h=1}^H({Q_h^{\tilde{k}}}-\cT Q_{h+1}^{\tilde{k}})(s_h^k,a_h^k,b_h^k)+O\left(H\sqrt{KH\log(KH/\delta)}\right).
\]
\end{lem}
\begin{proof}
We condition on the event defined in Lemma~\ref{lem:jin21}. Then we have that
\begin{align*}
\operatorname{NashRegret}(K)&=\sum_{k=1}^K\left[V_1^{*}(s_1)-V_1^{\mu^{\tilde{k}},\nu^k}(s_1)\right]\\
&\overset{(i)}{\leq}\sum_{k=1}^K\left[V_{Q_1^{\tilde{k}}}(s_1)-V_1^{\mu^{\tilde{k}},\nu^k}(s_1)\right]\\
&=\sum_{k=1}^K\sum_{h=1}^H\EE_{\mu^{\tilde{k}},\nu^k}\left[V_{Q_h^{\tilde{k}}}(s_h)-r_h(s_h,a_h,b_h)-V_{Q_{h+1}^{\tilde{k}}}(s_{h+1})\right]\\
&=\sum_{k=1}^K\sum_{h=1}^H\EE_{\mu^{\tilde{k}},\nu^k}\left[{\max_{\mu}\min_{\nu}\mu^{\top}Q_h^{\tilde{k}}}(s_h,\cdot,\cdot)\nu-r_h(s_h,a_h,b_h)-V_{Q_{h+1}^{\tilde{k}}}(s_{h+1})\right]\\
&\overset{(ii)}{=}\sum_{k=1}^K\sum_{h=1}^H\EE_{\mu^{\tilde{k}},\nu^k}\left[{\min_{\nu}(\mu_h^{\tilde{k}})^{\top}Q_h^{\tilde{k}}}(s_h,\cdot,\cdot)\nu-r_h(s_h,a_h,b_h)-V_{Q_{h+1}^{\tilde{k}}}(s_{h+1})\right]\\
&\leq\sum_{k=1}^K\sum_{h=1}^H\EE_{\mu^{\tilde{k}},\nu^k}\left[{(\mu_h^{\tilde{k}})^{\top}Q_h^{\tilde{k}}}(s_h,\cdot,\cdot)\nu_h^k-r_h(s_h,a_h,b_h)-V_{Q_{h+1}^{\tilde{k}}}(s_{h+1})\right]\\
&=\sum_{k=1}^K\sum_{h=1}^H\EE_{\mu^{\tilde{k}},\nu^k}\left[{Q_h^{\tilde{k}}}(s_h,a_h,b_h)-r_h(s_h,a_h,b_h)-V_{Q_{h+1}^{\tilde{k}}}(s_{h+1})\right]\\
&=\sum_{k=1}^K\sum_{h=1}^H\EE_{\mu^{\tilde{k}},\nu^k}\left[({Q_h^{\tilde{k}}}-\cT Q_{h+1}^{\tilde{k}})(s_h,a_h,b_h)\right]\\
&\overset{(iii)}{\leq}\sum_{k=1}^K\sum_{h=1}^H({Q_h^{\tilde{k}}}-\cT Q_{h+1}^{\tilde{k}})(s_h^k,a_h^k,b_h^k)+O\left(H\sqrt{KH\log(KH/\delta)}\right).\\
\end{align*}
Here (i) is due to Lemma~\ref{lem:jin21}, (ii) is due to $\mu_h^k=\mu_{Q_h^k}$, and (iii) is by Azuma-Hoeffding's inequality.
\end{proof}

\begin{proof}[Proof of Theorem~\ref{thm:MG}]
We condition on the event defined in Lemma~\ref{lem:jin31} and Lemma~\ref{lem:decomp_MG}. By Lemma~\ref{lem:jin41} (with $f_k=Q_h^{\tilde{k}}$ and $f_k'=Q_{h+1}^{\tilde{k}}$), we have that
\begin{align*}
\operatorname{NashRegret}(K)&\leq \sum_{k=1}^K\sum_{h=1}^H({Q_h^{\tilde{k}}}-\cT Q_{h+1}^{\tilde{k}})(s_h^k,a_h^k,b_h^k)+O\left(H\sqrt{KH\log(KH/\delta)}\right)\\
&\leq H\cdot O\left(\sqrt{\dim_{OBE}(\cF,1/\sqrt{K})\beta K}+H\cdot \min\{K,\dim_{OBE}(\cF,1/\sqrt{K})\}\right)+O\left(H\sqrt{KH\log(KH/\delta)}\right)
\end{align*}
Plugging the value of $\beta$ completes the proof of the first part. The second part is a direct implication of Theorem~\ref{thm:sample}.
\end{proof}
\section{Model Misspecification}
\label{sec:mis}
In this section we study the case when there is a misspecification error in our model. We show that our algorithms are robust to the violation of the assumptions. We stick to the planner in Section~\ref{sec:close}.

In the standard RL setting, Assumption~\ref{assum:express_1} with a misspecification error is stated as:
\begin{assum}\label{assum:mis1}
	There exists a set of functions $\funclass\subseteq\{f:\cS\times\cA\rightarrow [0, H + 1]\}$ and a real number $\zeta>0$, such that for any $V:\cS\rightarrow [0, H]$ and all $h\in[H]$, there exists $f_V \in \funclass$ which satisfies
	\[
	\max_{(s,a)\in \cS\times \cA}\left|f_V(s, a) - r_h(s, a) - \sum_{s' \in \states}P_h(s' \mid s, a)V(s')\right| \le \zeta.
	\]
	We call $\zeta$ the \emph{misspecification error}.
\end{assum} 

In the reward-free RL setting, Assumption~\ref{assum:express_2} with a misspecification error is stated as:

\begin{assum}\label{assum:mis2}
	There exists a set of functions $\funclass\subseteq\{f:\cS\times\cA\rightarrow [0, H + 1]\}$ and a real number $\zeta>0$, such that for any $V:\cS\rightarrow [0, H]$ and all $h\in[H]$, there exists $f_V \in \funclass$ which satisfies
	\[
	\max_{(s,a)\in \cS\times \cA}\left|f_V(s, a) - \sum_{s' \in \states}P_h(s' \mid s, a)V(s')\right| \le \zeta.
	\]
	We call $\zeta$ the \emph{misspecification error}.
\end{assum} 

Our algorithms for the misspecification case are same with the original algorithms except for the change of the global parameter $\beta$.

In the standard RL setting (Algorithm~\ref{alg:main}), we set $\beta$ to be:
\[
\beta=C\cdot(H^2\cdot\log(T\mathcal{N}(\mathcal{F},\delta/T^2)/\delta) \cdot\dim_E(\mathcal{F},1/T)\cdot\log^2 T\cdot
\log\left(\mathcal{N}(\mathcal{S}\times\mathcal{A},\delta/T^2) \cdot T/\delta\right)+T\zeta).
\]
In the reward-free RL setting (Algorithm~\ref{alg:planning} and Algorithm~\ref{alg:exploration}), we set $\beta$ to be:
\begin{align*}
\beta=C\cdot(H^2\cdot\log(T\mathcal{N}(\mathcal{F},\delta/T^2)/\delta) \cdot\dim_E(\mathcal{F},1/T)&\cdot\log^2 T\cdot
\log\left(\mathcal{N}(\mathcal{S}\times\mathcal{A},\delta/T^2) \cdot T/\delta\right)+T\zeta)\\
+C\cdot H^2\cdot\log(\mathcal{N}(\mathcal{R},1/T))&\cdot\dim_E(\mathcal{F},1/T)
\end{align*}
In Lemma~\ref{lem:single1} and Lemma~\ref{lem:single2}, we bound the single-step optimization error in the misspecified case. The proofs are identical to that of Lemma~11 in \cite{wang2020reinforcement}. 
\begin{lem}
\label{lem:single1}
Suppose $\cF$ satisfies Assumption~\ref{assum:mis1}. Consider a fixed pair $(k,h)\in[K]\times[H]$. 
For any $V:\cS\rightarrow [0,H]$, define
	\[
	\mathcal{D}_h^k(V):=\{(s_{h}^{\tau},a_{h}^{\tau},r_h^\tau+ V(s_{h+1}^{\tau}))\}_{\tau \in [k-1]}
	\]
	and also
	\[
	\widehat{f_V}:=\operatorname{argmin}_{f\in\mathcal{F}}\|f\|_{\mathcal{D}_h^k(V)}^2.
	\]
	For any $V:\cS\rightarrow[0,H]$ and $\delta \in (0,1)$, there is an event $\mathcal{E}_{V,\delta}$ which holds with probability at least $1-\delta$, such that for any $V':\mathcal{S}\rightarrow[0,H]$ with $\|V'-V\|_{\infty}\leq1/T$, we have
	\[
	\left\|\widehat{f_{V'}}(\cdot,\cdot)-r_h(\cdot,\cdot)-\sum_{s'\in\mathcal{S}}P_h(s'|\cdot,\cdot)V'(s')\right\|_{\mathcal{Z}_h^k}
	\leq C\cdot(\sqrt{H^2(\log(1/\delta)+\log\mathcal{N}(\mathcal{F},1/T))+T\zeta}).
	\]
\end{lem}
\begin{lem}
\label{lem:single2}
Suppose $\cF$ satisfies Assumption~\ref{assum:mis2}. Consider a fixed pair $(k,h)\in[K]\times[H]$.
For any $V:\cS\rightarrow [0,H]$, define
	\[
	\mathcal{D}_h^k(V):=\{(s_{h}^{\tau},a_{h}^{\tau}, V(s_{h+1}^{\tau}))\}_{\tau \in [k-1]}
	\]
	and also
	\[
	\widehat{f_V}:=\operatorname{argmin}_{f\in\mathcal{F}}\|f\|_{\mathcal{D}_h^k(V)}^2.
	\]
	For any $V:\cS\rightarrow[0,H]$ and $\delta \in (0,1)$, there is an event $\mathcal{E}_{V,\delta}$ which holds with probability at least $1-\delta$, such that for any $V':\mathcal{S}\rightarrow[0,H]$ with $\|V'-V\|_{\infty}\leq2/T$, we have
	\[
	\left\|\widehat{f_{V'}}(\cdot,\cdot)-\sum_{s'\in\mathcal{S}}P_h(s'|\cdot,\cdot)V'(s')\right\|_{\mathcal{Z}_h^k}\leq C\cdot(\sqrt{H^2(\log(1/\delta)+\log\mathcal{N}(\mathcal{F},1/T))+T\zeta}).
	\]
\end{lem}
Then similar to Lemma~\ref{lem:confidence_region} and Lemma~\ref{lem:confidence_region_2} we can verifies that the new value of $\beta$ can derive the desired confidence region.

With the new value of $\beta$, the new regret bound in the regular RL setting can be easily derived. In the reward-free setting the bound for $V_1^*(s_1,r)-V_1^\pi(s_1,r)$ will have an additional $O(\sqrt{H^5\cdot\dim_{E}(\cF,1/T)\cdot\zeta})$ term. Therefore there will be an irreducible error in the error bound.

Formally, our results in the misspecification case is stated in Theorem~\ref{thm:mis1} and Theorem~\ref{thm:mis2}.

\begin{thm} 
\label{thm:mis1}
Assume Assumption~\ref{assum:mis1} holds, and $T$ is sufficiently large.
With probability at least $1-\delta$, the algorithm achieves a regret bound,
\[
\text{Regret}(K)= O(\sqrt{ \iota_1\cdot H^3\cdot T }+\sqrt{\dim_{E}(\cF,1/T)\cdot H\cdot \zeta}\cdot T)
\]
where 
\[
\iota_1=\log(T\mathcal{N}(\mathcal{F},\delta/T^2)/\delta)\cdot \dim^2_E(\mathcal{F},1/T)\cdot\log^2 T
\cdot\log\left(\mathcal{N}(\mathcal{S}\times\mathcal{A},\delta/T^2) \cdot T/\delta\right)
\]
Furthermore, with probability $1-\delta$ the algorithm calls a $\Omega(K)$-sized regression oracle for at most $\widetilde{O}(d^2H^2)$ times.
\end{thm}

\begin{thm}
\label{thm:mis2}
Suppose Assumption~\ref{assum:mis2} holds and $T$ is sufficiently large.
For any given $\delta\in(0,1)$, after collecting $K$ trajectories during the exploration phase (by Algorithm~\ref{alg:exploration}), with probability at least $1-\delta$, for any reward function $r=\{r_h\}_{h=1}^H$ satisfying Assumption~\ref{assum:reward}, Algorithm~\ref{alg:planning} outputs an $O( H^3\cdot\sqrt{\iota_1/K}+\epsilon_i)$-optimial policy for the MDP $(\cS,\cA, P, r, H, s_1)$. Here,
\begin{align*}
\iota_1=\log(\mathcal{N}(\mathcal{R},1/T))&\cdot\dim_E(\mathcal{F},1/T)\\
+\log(T\mathcal{N}(\mathcal{F},\delta/T^2)/\delta)\cdot\log^2 T
&\cdot \dim^2_E(\mathcal{F},1/T)
\cdot\log\left(\mathcal{N}(\mathcal{S}\times\mathcal{A},\delta/T^2) \cdot T/\delta\right)
\end{align*}
and $\epsilon_i$ is the irreducible error:
\[
\epsilon_i=\sqrt{H^5\cdot\dim_{E}(\cF,1/T)\cdot\zeta}.
\]
Furthermore, with probability $1-\delta$ the algorithm calls a $\Omega(K)$-sized regression oracle for at most $\widetilde{O}(d^2H^2)$ times.
\end{thm}
\end{document}